\documentclass[10pt]{article} 
\usepackage[preprint]{tmlr}

\usepackage{amsmath,amsfonts,bm}

\def\eqref#1{equation~\ref{#1}}

\def\1{\bm{1}}

\DeclareMathAlphabet{\mathsfit}{\encodingdefault}{\sfdefault}{m}{sl}
\SetMathAlphabet{\mathsfit}{bold}{\encodingdefault}{\sfdefault}{bx}{n}

\DeclareMathOperator*{\argmax}{arg\,max}
\DeclareMathOperator*{\argmin}{arg\,min}

\usepackage[utf8]{inputenc} 
\usepackage[T1]{fontenc}    
\usepackage{hyperref}       
\usepackage{url}            
\usepackage{booktabs}       
\usepackage{amsfonts}       
\usepackage{nicefrac}       
\usepackage{microtype}      
\usepackage{xcolor}         
\usepackage{wrapfig}

\usepackage{amsmath}
\usepackage{amssymb}
\usepackage{mathtools}
\usepackage{amsthm}

\usepackage[capitalize,noabbrev]{cleveref}

\theoremstyle{plain}
\usepackage{hyperref}
\usepackage{url}
\usepackage{wrapfig}
\usepackage{makecell}
\usepackage{amsthm}
\usepackage{comment}
\usepackage{refcount}
\usepackage{graphicx}
\usepackage{multirow}
\usepackage{amsmath, amssymb}
\usepackage[utf8]{inputenc} 
\usepackage[T1]{fontenc}    
\usepackage{hyperref}       
\usepackage{url}            
\usepackage{booktabs}       
\usepackage{amsfonts}       
\usepackage{nicefrac}       
\usepackage{microtype}      
\usepackage{xcolor}         
\usepackage{bbm}
\usepackage{subcaption}

\newcolumntype{?}{!{\vrule width 1pt}}
\newtheorem{theorem}{Theorem}[section]
\newtheorem{proposition}[theorem]{Proposition}
\newtheorem{lemma}[theorem]{Lemma}
\newtheorem{corollary}[theorem]{Corollary}

\newtheorem{fact}[theorem]{Fact}
\newtheorem{definition}[theorem]{Definition}
\newtheorem{assumption}[theorem]{Assumption}
\usepackage[textsize=tiny]{todonotes}

\title{Why Does Counterintuitive Likelihood Behavior Rarely Emerge in Tabular Anomaly Detection?}

\author{\name Donghwan Kim \email cain0709@yonsei.ac.kr \\
      \addr Department of Industrial Engineering\\
      University of Yonsei
      \AND
      \name Junghun Phee \email  \\
      \addr H. Milton Stewart School of Industrial and Systems Engineering (ISyE)\\
      Georgia Institute of Technology
      \AND
      \name Hyunsoo Yoon \email hs.yoon@yonsei.ac.kr \\
      \addr Department of Industrial Engineering\\
      University of Yonsei
      }

\def\month{MM}  
\def\year{YYYY} 
\def\openreview{\url{https://openreview.net/forum?id=XXXX}} 

\begin{document}

\maketitle

\begin{abstract}
Deep generative models with tractable likelihoods, exemplified by normalizing flows, offer an effective basis for anomaly detection through likelihood-based scoring. We demonstrate that, unlike in the image domain where deep generative models frequently assign higher likelihoods to anomalous data, such counterintuitive behavior occurs rarely in tabular settings. We first introduce a domain-agnostic formulation that enables consistent detection and evaluation of the counterintuitive phenomenon, addressing the lack of a precise definition in prior work. Through extensive experiments on 47 tabular datasets benchmarked against 13 baseline models and on 10 CV/NLP embedding datasets from ADBench, we show that the phenomenon is consistently rare in general tabular data and that normalizing flow likelihood achieves competitive anomaly detection performance. We further provide theoretical and empirical analyses attributing this disparity to differences in dimensionality and feature correlation between tabular and image data. Therefore, our results indicate that likelihood-based detection with normalizing flows is a practical approach for tabular anomaly detection.
\end{abstract}

\setcounter{footnote}{0} 
\section{Introduction}
\label{Intro}
Generative models, including variational autoencoders (VAEs, \cite{kingma2014vae}), normalizing flows (NFs, \cite{dinh2015nice}), and generative adversarial networks (GANs, \cite{goodfellow2014generative}), are widely used to model complex data distributions across diverse applications such as industrial diagnostics, medical imaging, and financial risk assessment. Among these, normalizing flows are particularly well-suited for anomaly detection due to their ability to compute estimated likelihoods, providing a straightforward mechanism for detecting out-of-distribution (OOD) samples. \interfootnotelinepenalty=10000 \footnote{
Although the two tasks slightly differ, we consider OOD detection and anomaly detection to be the same task, and we will utilize the term anomaly detection. Task definitions are presented in Appendix~\ref{def_ood_ano}.}  The simplest anomaly detection approach using normalizing flows is to assume that normal data $\bold x \in \mathbb{R}^d$ follows the distribution $P$ and anomalous data $\bold x' \in \mathbb{R}^d$ follows a distribution $Q\neq P$, and to determine that a given data $\bold x_{\text{test}} \in \mathbb{R}^d$ is an anomaly if its estimated likelihood is lower than a predefined threshold $ \alpha$ at test time. We refer to this method as NF-SLT (Normalizing Flow with Simple Likelihood Test).

This methodology is based on the intuition that anomalous data are less likely under the distribution of normal data. However, likelihood-based models such as normalizing flows frequently violate this assumption in the image domain, assigning higher likelihoods to OOD data than to in-distribution data. \citet{nalisnick2018do} demonstrate this phenomenon concretely: a model trained on CIFAR-10 \citep{krizhevsky2009learning} assigns higher likelihoods to SVHN \citep{netzer2011reading} than to its own training data, directly contradicting the foundational assumption of likelihood-based anomaly detection. Consequently, naive likelihood thresholding can fail in such cases, a failure mode we refer to as the \textit{counterintuitive phenomenon} (see Section~\ref{related_counter}).

However, it remains unclear whether this phenomenon also occurs in tabular anomaly detection. \citet{kirichenko2020normalizing} observe significant likelihood overlap between in-distribution and OOD data in tabular settings although OOD data receive lower likelihoods on average, but their analysis is limited to only two datasets and lacks comparison with other tabular AD baselines. Moreover, such overlap is difficult to regard as counterintuitive: any result short of 100\% AUROC would qualify as counterintuitive, and it is plausible that likelihood overlap simply reflects the intrinsic difficulty of separating the two distributions rather than a genuine counterintuitive phenomenon.

This calls for a more principled criterion that goes beyond direct likelihood comparisons. Specifically, the criterion should evaluate the generative model's performance relative to domain-appropriate baselines (e.g., DeepSVDD, OCSVM). To address this gap, we first propose a precise, domain-agnostic definition of the counterintuitive phenomenon grounded in likelihood-based model comparisons. Building on this definition, we conduct an extensive set of experiments to examine whether the simple likelihood test, previously criticized for its limitations in image anomaly detection, remains effective in tabular anomaly detection. Consequently, we empirically demonstrate that no tabular dataset in ADBench \citep{han2022adbench} do not exhibit counterintuitive phenomena in the tabular domain, and NF-SLT achieves competitive performance on ADBench. Furthermore, we theoretically and empirically validate this success from the perspectives of dimensionality and feature correlation. To explain this from two complementary perspectives, we introduce the following two facts:


\begin{fact}[Lower Dimension]
\label{fact_lower_dim}
Images are typically organized along three structural axes (height, width, and channel), while tabular data generally have lower dimensionality, consisting of a single feature vector without spatial structure.
\end{fact}
\begin{fact}[Weak Feature Correlation]
\label{fact_hetero}
Images exhibit strong local pixel correlations, which models such as CNNs exploit to capture spatial relationships between neighboring pixels. In contrast, tabular data imposes no a priori structural relationship between features, and inter-feature correlations are generally weaker than those observed in image data.
\end{fact}
Taking ADBench as an example, most datasets contain fewer than 100 features, whereas CIFAR-10, one of the relatively lower-dimensional image datasets, has a dimensionality of 3072. This indicates that the curse of dimensionality is generally more pronounced in the image domain than in the tabular domain, and we analyze how this disparity affects likelihood estimation in normalizing flows. Additionally, \citet{kirichenko2020normalizing} argued that normalizing flows fail to capture semantic information in image OOD detection due to strong local pixel correlations. Building on Fact~\ref{fact_hetero}, we extend this argument to the tabular domain and claim that normalizing flows are substantially less affected by feature correlation in this setting. To support this claim, we quantify the overall degree of feature correlation by measuring the reduction of intrinsic dimensionality (ID) relative to the ambient dimension, interpret this reduction as a proxy for inter-feature correlation strength, and compare the extent of ID reduction across tabular and image datasets. Although certain tabular datasets with high dimensionality or strong inter-feature correlation exist (e.g., genomics datasets), these deviate substantially from the characteristics of typical tabular data; we discuss such cases in Appendix~\ref{real_experiment_dimension_appendix}.

In summary, this study makes three contributions centered on a surprising finding: although likelihood-based detection with normalizing flows is known to fail in the image domain \citep{nalisnick2018do} and shows significant likelihood overlap on tabular data \citep{kirichenko2020normalizing}, NF-SLT rarely exhibits the counterintuitive phenomenon since such overlap measured by AUROC also generally arises in the comparison models while NF-SLT performs competitively against them on ADBench.

\begin{itemize}
  \item We provide a \textbf{domain-agnostic definition of the counterintuitive phenomenon} in simple likelihood tests, and empirically demonstrate that normalizing flows in the tabular domain rarely exhibit this phenomenon.
  \item We validate our findings across all \textbf{47 tabular datasets and 10 CV/NLP embedding datasets} from ADBench without selection bias \citep{shwartz2022tabular}, benchmarking against 13 anomaly detection baselines on the tabular datasets and finding that NF-SLT achieves competitive anomaly detection performance.
  \item We provide \textbf{theoretical and empirical analyses} explaining why the counterintuitive phenomenon does not arise in the tabular domain, attributing this rarity to fundamental differences in dimensionality and feature correlation relative to the image domain.
\end{itemize}

\section{Related Work}
\label{realted_work}
\subsection{Normalizing Flow}
\label{related_nf}
Normalizing flow is a class of generative models that transforms input data $\bold x\in \mathbb{R}^d$, drawn from an unknown distribution with density $p_{\bold x}$, into a latent variable $\bold z\in \mathbb{R}^d$ following a simple distribution $p_{\bold z}$, typically chosen as the standard Gaussian $\mathcal{N}(0,I_d)$ \citep{dinh2017density}. This transformation is carried out by an invertible function $f: \mathbb{R}^d \rightarrow \mathbb{R}^d$ composed of complex mappings such as neural networks \citep{dinh2015nice}, such that $p_{\bold x}$ can be expressed in terms of $p_{\bold z}$ and the Jacobian determinant via the change-of-variables formula, as in Equation \ref{eq:nf}.

\begin{equation}
\label{eq:nf}
    \begin{split}
        \log p_{\bold x}(\bold x) = \log p_{\bold z}(\bold z) + \log |J|, J = \det\frac{\partial \bold z}{\partial \bold x} 
    \end{split}
\end{equation}

In general, the model is trained to maximize the log-likelihood $\log p_{\bold x}(\bold x)$ of the training data, thereby approximating the underlying data distribution \citep{caterini2022entropic}, and new data is sampled by drawing from the pre-defined $p_\bold z$ and passing the result through $f^{-1}$. Normalizing flows can further be categorized by whether the determinant of the Jacobian (i.e., the volume term) is fixed \citep{dinh2015nice} or varies with the input \citep{rezende2015variational,dinh2017density,kingma2018glow, behrmann2019invertible, chen2019residual, durkan2019neural}. A key advantage of normalizing flows over models such as variational autoencoders and generative adversarial networks is their ability to provide exact likelihood estimates without requiring approximate inference techniques \citep{nalisnick2018do}. However, this comes with two constraints: (1) the computational cost of the Jacobian determinant must remain tractable, and (2) the inverse $f^{-1}$ must exist. To satisfy these requirements, commonly used approaches include coupling layers \citep{dinh2015nice, dinh2017density,kingma2018glow}, special-form transformations \citep{rezende2015variational}, and power-series approximations with Lipschitz constraints \citep{behrmann2019invertible, chen2019residual}.

\subsection{Counterintuitive Phenomenon of Likelihood}
\label{related_counter}
\citet{nalisnick2018do} reported that a counterintuitive phenomenon regarding likelihood assignment occurs in models that can obtain estimated likelihoods, such as normalizing flows and autoregressive models in the image domain. This study motivates subsequent work on identifying the cause of this phenomenon and proposing remedies, broadly categorized into architecture modifications, auxiliary score augmentation, and theoretical explanations. \citet{kirichenko2020normalizing} and \citet{schirrmeister2020understanding} improved anomaly detection performance by modifying flow architectures. In particular, the latter introduced an approach that reflects the hierarchical data structure, thereby improving detection performance. \citet{Serrà2020Input} quantified complexity through a general compression algorithm such as PNG, demonstrating that simple images exhibit higher likelihood, and presented an anomaly score combining the likelihood and complexity terms. \citet{kamkari2024a} used Local Intrinsic Dimension (LID) to measure an image's simplicity and proposed a dual thresholding method for LID and likelihood to improve anomaly detection performance. \citet{morningstar2021density}, \citet{osada2024understanding} and \citet{ahmadian2021likelihood} mitigated the drawback of using only a single likelihood score by estimating the density of a vector that combines the likelihood with several auxiliary statistics (e.g., complexity, the log-determinant of the Jacobian). \citet{nalisnick2019detecting} attributed detection failure to the discrepancy between the typicality set \citep{cover1999elements} and the high-density region occupied by OOD data. \citet{zhang2021understanding} presented the view that the counterintuitive phenomenon arises from model misestimation. \citet{le2021perfect} demonstrated that even with a perfect model, simple likelihood-based methods can fail due to variations in the choice of representation. \citet{ren2019likelihood} improved detection performance by using the likelihood ratio between the background and semantic models, and \citet{caterini2022entropic} explained the cause of this phenomenon from an entropic perspective and provided a theoretical account of why likelihood ratio models perform well. Additionally, \citet{kim2026mitigating} showed that likelihood inversion can be mitigated by manipulating entropy with semantic similarity signals from a pretrained encoder.

However, these works primarily focus on image-domain mechanisms or remedies and do not establish a domain-agnostic criterion for counterintuitive failures, nor do they quantify how often such failures occur across a broad suite of tabular benchmarks. Crucially, prior work uses the term \textit{counterintuitive phenomenon} informally, without a precise criterion for when such behavior constitutes a meaningful failure of the generative model. Our work fills this gap with Definition \ref{def:counterintuitive} and large-scale tabular evaluation, complemented by a dimension- and correlation-based analysis of when likelihood-only testing becomes fragile.

\section{Definition of Counterintuitive Phenomenon}
\label{def_counter}

Earlier research on the counterintuitive phenomenon has largely relied on informal observations, noting that anomalous data receive higher likelihood than in-distribution data or that likelihood overlap exists between in-distribution and OOD data \citep{nalisnick2018do, kirichenko2020normalizing}, without establishing a precise criterion for when such behavior should be considered counterintuitive. In particular, \cite{kirichenko2020normalizing} provides no reference to how other detection methods perform on the same data, making it impossible to determine whether low likelihood-based performance reflects a fundamental property of the generative model or merely the intrinsic difficulty of a given dataset, where all methods struggle equally regardless of model class. To address this, we propose a domain-agnostic definition that formalizes the counterintuitive phenomenon through relative comparison against a pool of baseline models.

A natural starting point is the observation that the AUROC of a likelihood-based test admits a direct probabilistic interpretation:
\begin{equation}
\textsc{AUROC} = \Pr(\log P_{\theta}(x) > \log P_{\theta}(x')), \quad x \sim P,\ x' \sim Q.
\end{equation}
This is precisely the probability that a normal sample receives a higher likelihood than an anomalous one; a low AUROC thus directly quantifies the rate of sample-level likelihood inversion. However, a low AUROC alone cannot distinguish genuine likelihood inversion (arising from properties of the generative model) from apparent inversion due to intrinsic dataset difficulty. Building on this interpretation, we formalize the counterintuitive phenomenon through two assumptions about its observable consequences. Note that these assumptions characterize what we would expect to observe if the phenomenon occurs; they describe observable consequences rather than defining the phenomenon itself.

\begin{assumption}[Relatively Low Performance]
\label{assump:counter_1}
If a generative model's likelihood fails to separate normal from anomalous data due to model-intrinsic properties rather than dataset difficulty, most comparison models will outperform the generative model.
\end{assumption}
\begin{assumption}[High Performance Gap]
\label{assump:counter_2}
Even under such failure, the performance gap between the generative model and the outperforming comparison models must be non-trivial. If the gap is small, the failure cannot be attributed to a fundamental property of the generative model.
\end{assumption}

\begin{definition}[Occurrence of Counterintuitive Phenomenon, Informal]
\label{def:counterintuitive}
Let $\textsc{AUROC}_0$ denote the AUROC of the likelihood-only test using the generative model $P_{\theta_0}$ on a normal/abnormal dataset pair $(P,Q)$, and let $\textsc{AUROC}_i$ denote that of the $i$-th comparison model for $i=1,\dots,k$. We say that a counterintuitive phenomenon occurs if both conditions hold:
\begin{equation}
\label{eq:counter_condition_1}
\frac{1}{k}\sum_{i=1}^k \mathbbm{1}\{\textsc{AUROC}_i > \textsc{AUROC}_0\} > \beta,
\end{equation}
\begin{equation}
\label{eq:counter_condition_2}
\min_{i:\textsc{AUROC}_i > \textsc{AUROC}_0}(\textsc{AUROC}_i-\textsc{AUROC}_0) > \gamma.
\end{equation}
\end{definition}

Definition~\ref{def:counterintuitive} states that a counterintuitive phenomenon occurs when the proportion of comparison models whose AUROC exceeds that of $P_{\theta_0}$ is greater than $\beta$, and the minimum AUROC difference between $P_{\theta_0}$ and the outperforming models is greater than $\gamma$. This formalizes counterintuitiveness as a relative failure against domain-appropriate baselines, because a failure is surprising only relative to what comparable detectors achieve on the same task, whereas an absolute AUROC value alone cannot separate genuine likelihood inversion from intrinsic dataset difficulty. Since the reference is the performance of other detectors rather than a domain-specific absolute value, the formulation also remains modality-agnostic. The comparison pool is constructed from models specifically developed and widely adopted for the target domain (e.g., DeepSVDD \citep{ruff2018deep}, MCM \citep{yin2024mcm}), ensuring that baseline performance reflects the difficulty of the task rather than an arbitrary model choice. Note that Definition~\ref{def:counterintuitive} serves as an operational criterion for practically meaningful counterintuitive failure, rather than a characterization of the phenomenon itself. Accordingly, because the criterion is defined relative to the comparison pool, its verdict can depend on the composition of the pool, which we examine with alternative pools in Section~\ref{experiment}. The fully rigorous formulation, including formal treatment of the probability estimates used in place of AUROC, is provided in Appendix~\ref{rigorous_def}.

Consider the CIFAR-10 (in-distribution) vs. SVHN (out-of-distribution) setting. According to \citet{morningstar2021density}, a simple likelihood test using Glow \citep{kingma2018glow} yielded an AUROC of 6.4\%, implying that anomalous SVHN samples receive higher likelihood than in-distribution CIFAR-10 samples in approximately 93.6\% of pairs. In contrast, \citet{sun2022out} achieved AUROC scores exceeding 90\% with their proposed method and comparison models. Based on Definition~\ref{def:counterintuitive}, this case clearly demonstrates a counterintuitive phenomenon, as the generative model performs significantly worse than the comparison models, which also agrees with the literal inversion indicated by an AUROC far below 0.5. However, the criterion does not flag every case of literal inversion. For the tabular dataset yeast examined in Section~\ref{experiment}, NF-SLT attains an AUROC below 0.5 but Definition \ref{def:counterintuitive} is not satisfied, because only 7 of the 13 comparison models outperform it with a minimum gap of merely 0.0002, while a majority of the comparison models also record AUROC below 0.5. Hence the criterion reflects the prior that such a dataset is difficult to separate for general detectors, and it judges counterintuitiveness against this prior rather than against the absolute AUROC alone. To examine whether this phenomenon occurs in tabular data, we apply Definition~\ref{def:counterintuitive} to tabular anomaly detection datasets in the next section.

\section{Experiment}
\label{experiment}
\textbf{Dataset and Preprocessing} The experiment was conducted using the data split protocol in \citet{zong2018deep}. Under this protocol, 50\% of the normal data is used for training, and the remaining normal data together with all anomalous data are used for testing. We used \textbf{all 47 tabular and 10 CV/NLP embedding datasets} presented in ADBench. Using the entire dataset was motivated by \citet{shwartz2022tabular}, who criticized that researchers often introduce selection bias by choosing specific datasets to inflate performance. To address this, we included all proposed benchmark datasets without exclusion. All models utilized the RobustScaler provided by the Python library Scikit-learn \citep{pedregosa2011scikit} to standardize the input data, except for NeuTraLAD and DRL, whose scaling configurations differ for implementation-specific reasons detailed in Appendix~\ref{model_detail_hyp}.

\textbf{Models} We compared the performance of 6 shallow AD models and 7 deep AD models. We implemented the shallow models using PyOD \citep{zhao2019pyod} and Scikit-learn \citep{pedregosa2011scikit}. The shallow models used for comparison are PCA \citep{shyu2003novel}, LOF \citep{breunig2000lof}, IF \citep{liu2008isolation}, OCSVM \citep{scholkopf1999support}, COPOD \citep{li2020copod}, and ECOD \citep{li2022ecod}. The deep baseline models are DAGMM \citep{zong2018deep}, DeepSVDD \citep{ruff2018deep}, GOAD \citep{Bergman2020Classification-Based}, NeuTraLAD \citep{qiu2021neural}, ICL \citep{shenkar2022anomaly}, MCM \citep{yin2024mcm}, DRL \citep{ye2025drl} and NF-SLT implemented with NICE \citep{dinh2015nice}. Overall hyperparameter settings and implementation details are provided in Appendix \ref{model_detail_hyp}.


\begin{table}[!t]
\centering

\caption{(Top): Evaluation performance on 47 tabular datasets. (Bottom): Evaluation performance on 10 CV/NLP embedding datasets. Top2 Cum. Ratio indicates the proportion of datasets where a method ranks within the top-2 in AUROC, and Fail Ratio indicates the proportion where it ranks 9th or lower.}
\label{tab:auroc_fair}

\scriptsize 
\setlength{\tabcolsep}{6pt}

\resizebox{0.8\linewidth}{!}{
\begin{tabular}{lccccc}
\toprule
Method & AUROC $\uparrow$ & AUPRC $\uparrow$ & Avg.\ Rank $\downarrow$ & Top2 Cum. Ratio $\uparrow$ & Fail Ratio $\downarrow$ \\
\midrule
PCA       & 0.7715          & 0.5209          & 7.13                    & 0.15                  & 0.45 \\
LOF       & 0.8169          & 0.5606          & 6.11                    & 0.17                  & 0.26 \\
IF        & 0.8014          & 0.5060          & 6.23                    & 0.19                  & 0.21 \\
OCSVM     & 0.6562          & 0.3833          & 10.34                   & 0.06                  & 0.77 \\
COPOD     & 0.7471          & 0.4419          & 8.40                    & 0.11                  & 0.57 \\
ECOD      & 0.7425          & 0.4530          & 8.74                    & 0.06                  & 0.68 \\
DAGMM     & 0.6467          & 0.3468          & 11.45                   & 0.00                  & 0.89 \\
DeepSVDD  & 0.7687          & 0.5388          & 7.74                    & 0.02                  & 0.49 \\
GOAD      & 0.6086          & 0.4114          & 10.60                   & 0.04                  & 0.62 \\
NeuTraLAD & 0.8081          & 0.5694          & 6.17                    & 0.26                  & 0.30 \\
ICL       & 0.8208          & 0.6170          & 5.70                    & 0.30                  & 0.26 \\
MCM       & 0.7864          & 0.5383          & 7.34                    & 0.11                  & 0.36 \\
DRL       & 0.8363          & 0.6124          & 5.04                    & 0.21                  & 0.09 \\
NF-SLT    & \textbf{0.8575} & \textbf{0.6398} & \textbf{3.74}           & \textbf{0.40}         & \textbf{0.06} \\
\bottomrule
\end{tabular}
}\\[0.2cm]

\resizebox{0.8\linewidth}{!}{
\begin{tabular}{lccccccc}
\toprule
Dataset & DeepSVDD & GOAD & NeuTraLAD & ICL & MCM & DRL & NF-SLT \\
\midrule
CIFAR-10      & 0.9103 & 0.9334          & 0.9405 & 0.9254 & 0.9381 & 0.8416 & \textbf{0.9527} \\
FashionMNIST  & 0.9117 & 0.9060          & 0.9360 & 0.9266 & 0.9380 & 0.8240 & \textbf{0.9455} \\
MNIST-C       & 0.8348 & 0.7741          & 0.8519 & 0.8257 & 0.8836 & 0.7312 & \textbf{0.8950} \\
MVTecAD       & 0.7542 & 0.7960          & 0.8874 & 0.8874 & 0.8408 & 0.7093 & \textbf{0.9100} \\
SVHN          & 0.5466 & 0.5366          & 0.5774 & 0.5626 & 0.5771 & 0.5251 & \textbf{0.5842} \\
20news        & 0.5547 & 0.5438          & 0.6001 & 0.6087 & 0.5995 & 0.5457 & \textbf{0.6547} \\
agnews        & 0.6630 & 0.5857          & 0.6509 & 0.6697 & 0.7252 & 0.5817 & \textbf{0.7591} \\
amazon        & 0.5833 & 0.5613          & 0.6010 & 0.6022 & 0.6049 & 0.5424 & \textbf{0.6194} \\
imdb          & 0.5090 & \textbf{0.5398} & 0.5393 & 0.5098 & 0.5090 & 0.4666 & 0.5013 \\
yelp          & 0.6490 & 0.6138          & 0.6620 & 0.6690 & 0.6750 & 0.5705 & \textbf{0.6971} \\
\bottomrule
\end{tabular}
}

\vspace{-0.5cm}
\end{table}

\textbf{Evaluation}
We evaluate all AD models using AUROC and AUPRC. For tabular datasets, we run 10 repeated trials and report the mean AUROC and the relative rank of each method in Table~\ref{tab:auroc_fair}. For the main results in Table~\ref{tab:auroc_fair}, we adopt a single global hyperparameter configuration per model, selected to maximize the AUROC averaged across all tabular datasets. Since default hyperparameters are not directly comparable across heterogeneous AD models, we tune each method over its predefined search space and fix one global configuration to avoid per-dataset oracle tuning. Specifically, each candidate configuration is evaluated with 10 trials, and we select the one that achieves the best aggregated mean AUROC. We additionally report (i) the standard deviation across the 10 trials and (ii) per-dataset performance under the selected global configurations in Appendix~\ref{add_experiment}. Appendix~\ref{model_detail_hyp} provides the hyperparameter search spaces for all models and includes a sensitivity analysis.

\textbf{Experiment Result}
Under Definition~\ref{def:counterintuitive}, if the counterintuitive phenomenon were prevalent in the tabular domain, we would expect NF-SLT to exhibit a high fail ratio, where on certain datasets it dramatically underperforms comparison models by a non-trivial margin despite performing well overall. This pattern, strong performance on most datasets yet catastrophic failure on specific ones, would be the signature of genuine likelihood inversion rather than consistently poor performance across all datasets. However, as shown in Table~\ref{tab:auroc_fair}, NF-SLT attains the highest AUROC, AUPRC, and Top2 Cum. Ratio while recording the lowest average rank and fail ratio compared with all 13 comparison models. NF-SLT also performs competitively on ADBench's CV/NLP embedding datasets and outperforms all deep baselines on 9 out of 10 datasets. Since the AUROC of the likelihood test equals the probability that a normal sample receives a higher likelihood than an anomalous one, we also report the rate of literal inversion. NF-SLT attains an AUROC below 0.5 on only 1 of the 47 tabular datasets (yeast, 0.4652) and below 0.6 on only 5 datasets (yeast, WPBC, vertebral, ALOI, and speech). Among them, yeast is the weakest case for NF-SLT, but only 7 of the 13 comparison models outperform it on this dataset, and the minimum AUROC difference between MCM and NF-SLT is only 0.0002, which does not satisfy the second condition. Moreover, 8 of the 13 comparison models record AUROC below 0.5 on yeast and the best comparison model (DeepSVDD) attains only 0.6315, which indicates that the difficulty is largely shared across the comparison pool. Hence this case cannot be attributed to a counterintuitive phenomenon.

Similarly, among the CV/NLP embedding datasets, NF-SLT underperforms the best baseline only on `imdb', yet the performance gap is again marginal and fails to satisfy the second condition of Definition~\ref{def:counterintuitive}, so we do not regard it as a counterintuitive case. To assess the sensitivity of Definition~\ref{def:counterintuitive}, we set the competitor pool to the full set of comparison models and swept thresholds over $\beta \in \left\{\frac{8}{13}, \frac{9}{13}, \frac{10}{13}, \frac{11}{13}, \frac{12}{13}\right\}$ and $\gamma \in \{0.3, 0.4, 0.5, 0.6\}$, evaluating all possible $(\beta,\gamma)$ combinations. Across this entire grid, no tabular dataset was classified as counterintuitive, so no decision flips occurred; the flip frequency is therefore 0\% within the tested ranges. We further repeated this analysis with alternative competitor pools: (i) shallow-model competitors only and (ii) deep-model competitors only, adjusting $\beta$ appropriately to reflect the pool size. Under both alternative pools and the same $\gamma$ sweep, we again observed no counterintuitive cases and hence no decision flips. The decision is also unchanged when the minimum in Equation~\ref{eq:counter_condition_2} is replaced by the mean or the maximum gap among the outperforming models, since the largest margin by which any comparison model outperforms NF-SLT on any tabular dataset is 0.1663 (DeepSVDD on yeast), which is smaller than the smallest value of $\gamma$ in the original grid (0.3). In addition, for the full pool, extending the sweep to $\gamma$ as small as 0.002 still yields no counterintuitive case, and the first decision flip requires $\gamma$ below 0.001 (Lymphography with $\beta\in\{8/13,9/13\}$), which corresponds to a practically negligible AUROC difference. 

Beyond this threshold analysis, we report detection performance on datasets dominated by categorical features and diverse anomaly types in Appendix~\ref{detection_performance_ano_types}, where NF-SLT again shows strong performance. For an additional consistency check, we compare NF-SLT against the typicality test \citep{nalisnick2019detecting}, a method specifically proposed to mitigate the counterintuitive phenomenon, and find that simple likelihood testing still outperforms it across all aggregate metrics (Appendix~\ref{typicality_test_perf}). We further verify that more expressive normalizing flow architectures do not necessarily yield better anomaly detection performance in the tabular domain as NICE outperforms RealNVP \citep{dinh2017density} on all aggregate metrics and achieves higher AUROC than MAF \citep{papamakarios2017masked} on five of the seven datasets on which MAF trained stably, despite their higher architectural complexity than NICE (Appendix~\ref{nice_realnvp_comp}).

\section{Why Is the Simple Likelihood Test Successful in Tabular Data?}
\label{why_simple}

\subsection{High Dimension Perspective}
\label{high_dim_perspective}
Based on Fact~\ref{fact_lower_dim}, we characterize the conditions under which likelihood testing degrades, and show that these conditions are rarely met in tabular data due to their lower dimensionality. It has been reported that likelihood inversion in the image domain usually occurs when normal data has higher input complexity than anomalous data \citep{Serrà2020Input}, and high data complexity can be interpreted as high entropy of the underlying distribution. Building on this observation, we extend the likelihood-gap expression in \cite{caterini2022entropic}, which characterizes the expected likelihood difference between normal and abnormal data in terms of entropy, and link it to Fact~\ref{fact_lower_dim}.

Let the distribution of normal data be $P$, the distribution of abnormal data be $Q$, and let $P_\theta$ be a model such as normalizing flow that estimates the density of $P$. Then, the gap of the likelihood of each distribution estimated by $P_\theta$ can be expressed as follows:
\begin{align}
\label{eq:likelihood_gap}
\begin{split}
        &\mathbb{E}_{\textbf{x}\sim P}[\log P_{\theta}(\textbf{x})] - \mathbb{E}_{\textbf{x}\sim Q}[\log P_{\theta}(\textbf{x})]   \\
        &=  D_{KL}(Q||P_{\theta}) -D_{KL}(P||P_{\theta}) +\mathbb{H}(Q) -\mathbb{H}(P)
\end{split}
\end{align}
where $\mathbb{H}(P)$ is the entropy of distribution $P$, and $D_{KL}(Q||P_{\theta})$ is the KL-divergence between distribution $Q$ and density estimation model like normalizing flow $P_\theta$. 
In Equation \ref{eq:likelihood_gap}, if $\mathbb{H}(P) - \mathbb{H}(Q) > D_{KL}(Q\|P)$, the expectation gap of the likelihood becomes negative in the limit of a well-trained model. However, in the previous study by \citet{caterini2022entropic}, the effect of dimensionality on the expected likelihood gap was not analyzed; we therefore extend their analysis to examine how dimensionality affects the likelihood gap in Theorem \ref{theorem:gap_expectation_likelihood_indep}, with proofs for all results provided in Appendix \ref{proof_theorem}.

\setcounter{theorem}{3}
\begin{theorem}[Impact of Dimensionality on Likelihood Gap]
\label{theorem:gap_expectation_likelihood_indep}
Let $P=\prod^d_{i=1}p_i(x_i)$ be an independent $d$-dimensional continuous probability density model in $\mathbb{R}^d$ and $Q$ be a distribution on $\mathbb{R}^d$ with a density whose $i$-th one-dimensional marginal is $q_i$. Assume the conditions of Lemma~\ref{lemma:kldivergence_approx} and $\mathbb{E}_{x_i\sim q_i}[|\log p_i(x_i)|]<\infty$ for all $i\in[d]$. Let $P_{\theta}$ be a well-trained density estimation model that approximates $P$ (i.e., $p_{\theta}(x) \rightarrow p(x)$ pointwise as $\theta \rightarrow \theta_0$). If $\mathbb{H}(P)-\mathbb{H}(Q) > D_{KL}(Q||P)$, then in the limit as $\theta \rightarrow \theta_0$, the lower bound of the gap between the expectations of the likelihood for $P$ and $Q$ decreases linearly with respect to $d$.
\end{theorem}

According to Theorem \ref{theorem:gap_expectation_likelihood_indep}, even when $P_\theta$ is an almost perfect model, if $\mathbb{H}(P)-\mathbb{H}(Q) > D_{KL}(Q||P)$, the lower bound of the likelihood gap decreases linearly with dimension, and under the stronger per-dimension condition in Corollary \ref{cor:special_case_gap_indep}, the upper bound is also guaranteed negative. Furthermore, Corollary \ref{cor:upper_bound_auroc} shows that under additional moment assumptions, the AUROC upper bound also decays polynomially with dimension (proof in Appendix \ref{proof_theorem}).

\setcounter{theorem}{5}
\begin{corollary}[Dimensionality and AUROC Upper Bound]
\label{cor:upper_bound_auroc}
Building on the assumptions of Corollary~\ref{cor:special_case_gap_indep}, suppose the $n$-th absolute central moment of the log-likelihood difference, $\log p_\theta(Y) - \log p_\theta(X)$, scales as $\mathcal{O}(d^k)$ for some $n > 1$ and $k < n$. In this case, if the average log-likelihood gap becomes negative, the maximum achievable AUROC for distinguishing samples from $P$ and $Q$ is inversely related to the dimensionality $d$. This indicates that as the dimension increases, the likelihood test becomes fundamentally less effective at separating normal and abnormal samples.
\end{corollary}

\begin{table}[h]
\caption{AUROC scores for likelihood tests as a function of the number of retained independent components using RealNVP with MLP (images preprocessed by ICA). The left block corresponds to $\mathbb{H}(P) > \mathbb{H}(Q)$, and the right block corresponds to $\mathbb{H}(P) < \mathbb{H}(Q)$.}
\centering
\label{tab:ica_auroc}
\small
\setlength{\tabcolsep}{6pt}
\resizebox{0.98\linewidth}{!}{
\begin{tabular}{lcccc @{\hspace{10pt}} lcccc}
\toprule
In-dist ($P$) / Out-dist ($Q$) & 1024 & 512 & 256 & 30 &
In-dist ($P$) / Out-dist ($Q$) & 1024 & 512 & 256 & 30 \\
\midrule
CIFAR-10 / SVHN   & 0.2311 & 0.2924 & 0.2984 & 0.3143 &
SVHN / CIFAR-10   & 0.9917 & 0.9843 & 0.9486 & 0.8520 \\
CIFAR-100 / SVHN  & 0.0843 & 0.1160 & 0.2036 & 0.3490 &
SVHN / CIFAR-100  & 0.9933 & 0.9536 & 0.9137 & 0.8622 \\
CelebA / SVHN     & 0.1207 & 0.1782 & 0.2745 & 0.4711 &
SVHN / CelebA     & 0.9976 & 0.9811 & 0.9722 & 0.9481 \\
\bottomrule
\end{tabular}
}
\end{table}

Hence a tighter AUROC upper bound makes NF-SLT more likely to underperform other models by a non-trivial margin, which makes the counterintuitive criterion in Definition \ref{def:counterintuitive} more likely to be satisfied. To validate this prediction, we apply ICA \citep{hyvarinen2000independent} to high-dimensional image data and measure AUROC as a function of retained independent components via RealNVP (Table \ref{tab:ica_auroc}, with full results in Table \ref{tab:ica_auroc_appendix}, Appendix~\ref{real_experiment_dimension_appendix}). When $\mathbb{H}(P) > \mathbb{H}(Q)$, AUROC increases as dimensionality decreases (e.g., from 0.0843 with 1024 components to 0.3490 with 30 components on CIFAR-100 / SVHN), whereas the reverse condition produces the opposite trend (from 0.9933 to 0.8622 on SVHN / CIFAR-100). These trends are consistent with Theorem \ref{theorem:gap_expectation_likelihood_indep} and Corollary \ref{cor:upper_bound_auroc}. This supports the view that tabular data, being lower-dimensional, is structurally less susceptible to the counterintuitive phenomenon as indicated by Fact \ref{fact_lower_dim}.

However, ICA yields approximately independent components, which matches the independence of $P$ required by Theorem \ref{theorem:gap_expectation_likelihood_indep}, whereas features in real-world tabular data may still exhibit dependencies. Notably, PCA-based experiments (Table \ref{tab:pca_auroc}, Appendix~\ref{real_experiment_dimension_appendix}), in which independence is violated, show a similar trend. To provide theoretical support for this observation, we relax the independence requirement to a block-dependent structure in the following theorem, which permits arbitrary dependence within each block while the blocks of $P$ remain mutually independent (proof in Appendix \ref{proof_theorem}).

\begin{theorem}[Impact of Dimensionality on Likelihood Gap: Block-Dependent Case]
\label{theorem:gap_block_dependent}
Let $P = \prod_{k=1}^{n} P_k$ and $Q$ be a probability density model whose $k$-th block marginal is $Q_k$, where each $P_k$ and $Q_k$ is a continuous probability density model in $\mathbb{R}^{d_k}$ with $\sum_{k=1}^{n} d_k = d$, and $P$ and $Q$ satisfy the same conditions as Lemma~\ref{lemma:kldivergence_approx}. Let $P_{\theta}$ be a well-trained density estimation model that approximates $P$ (i.e., $p_{\theta}(x) \rightarrow p(x)$ pointwise as $\theta \rightarrow \theta_0$). The components $P_k$ are mutually independent across $k$, with arbitrary dependence structure within each block. If $\mathbb{H}(P) - \mathbb{H}(Q) > D_{KL}(Q\|P)$, then:
\begin{align}
\lim_{\theta \rightarrow \theta_0}\left[\mathbb{E}_{x \sim P}[\log p_\theta(x)] - \mathbb{E}_{x \sim Q}[\log p_\theta(x)]\right] \geq d \min_{k \in [n]} \frac{\delta_k}{d_k},
\end{align}
where $\delta_k \triangleq \mathbb{H}(Q_k) - \mathbb{H}(P_k)$.
\end{theorem}

Theorem~\ref{theorem:gap_block_dependent} reduces to Theorem \ref{theorem:gap_expectation_likelihood_indep} when every block is a single coordinate ($d_k=1$). Furthermore, the lower bound $d\min_{k\in[n]}\delta_k/d_k$ remains negative under arbitrary intra-block dependence and is linear in $d$ for a fixed per-dimension entropy gap. If dimensions are added to any block other than the minimizing block $k^*=\argmin_{k\in[n]}\delta_k/d_k$ or an entirely new block is appended, then $\min_{k}\delta_k/d_k$ cannot increase while $d$ strictly increases, so the lower bound becomes strictly more negative and the potential severity of the likelihood inversion grows with dimensionality. In contrast, adding dimensions to $k^*$ itself may dilute the per-dimension entropy gap and the linear scaling is not guaranteed in that case (details in Appendix \ref{proof_theorem}).

Interestingly, in a controlled setting of distinguishing two Gaussians with different means and the same entropy, we observed that the AUROC approaches 0.5 as dimensionality increases when NICE and RealNVP are used with ReLU-like activations. We report this observation, together with an analysis of the flow's latent vector and its dependence on the activation function, in Appendix~\ref{additional_description} without attributing it to dimensionality alone. Additional results on real tabular data including the effect of dimension reduction on InternetAds are provided in Appendix~\ref{real_experiment_dimension_appendix}.

\subsection{Feature Correlation Perspective}
\label{feature_hetero}

In prior work, \citet{kirichenko2020normalizing} showed that likelihood inversion can arise in normalizing flows on image data due to strong local pixel correlations. Meanwhile, \citet{schirrmeister2020understanding} reported that, in the image domain, OOD detection performance can improve when flow architectures use multi-layer perceptrons (MLP) rather than convolutional neural networks (CNNs) in certain settings. As noted by \citet{battaglia2018relational}, CNNs exhibit an inductive bias known as locality, making them particularly effective for image data where local pixel correlations are strong. In contrast, MLPs have a weaker inductive bias and are therefore more suitable for tabular data, where strong structural correlations between features are not assumed.

Building on prior work attributing counterintuitive likelihood behavior to strong feature correlations and CNN-based inductive bias \citep{kirichenko2020normalizing, schirrmeister2020understanding}, we leverage Fact \ref{fact_hetero} to explain why this phenomenon rarely arises in tabular data. This is because tabular features are inherently heterogeneous, unlike the homogeneous features of image data. Images are homogeneous because all features are of the same type (pixels), take values in the same range (e.g., 0–255), and exhibit strong local correlations. By contrast, tabular features are heterogeneous: they include many types (e.g., continuous, discrete, categorical), their ranges or categories are defined independently, and no specific structural relationships between features are assumed.

\begin{figure*}[t!]
\begin{center}
    \centering
    \begin{subfigure}[b]{0.32\textwidth}
        \centering
        \includegraphics[width=\textwidth]{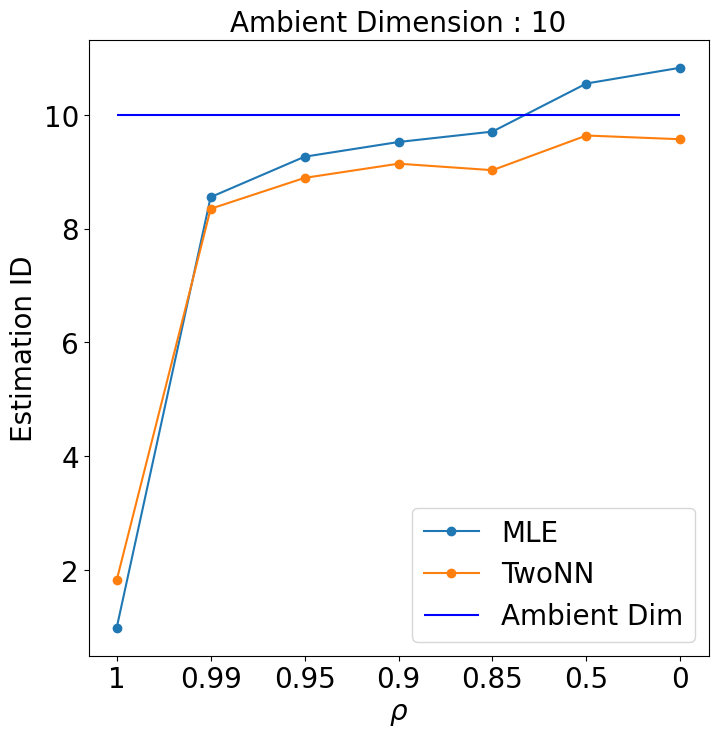}
    \end{subfigure}
    \hfill
    \begin{subfigure}[b]{0.32\textwidth}
        \centering
        \includegraphics[width=\textwidth]{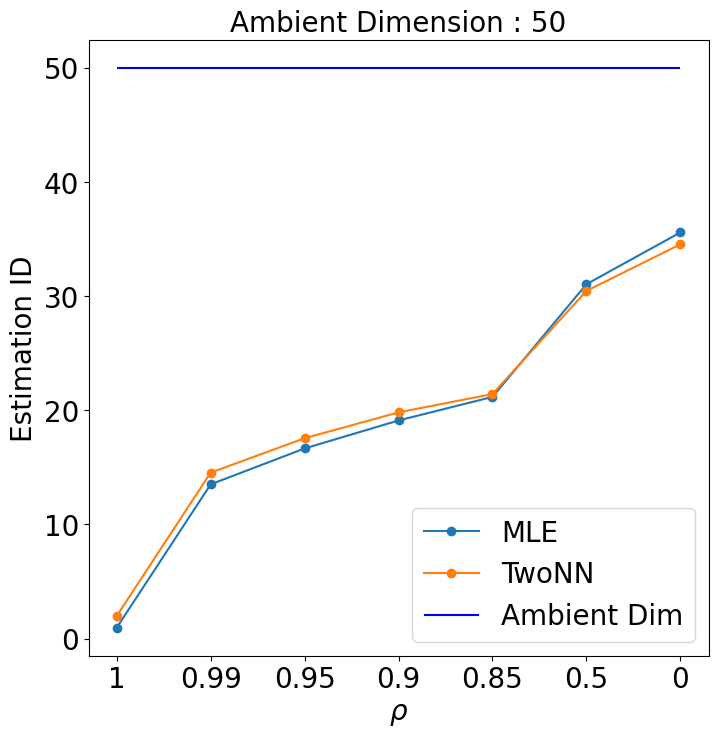}
    \end{subfigure}
    \hfill
    \begin{subfigure}[b]{0.32\textwidth}
        \centering
        \includegraphics[width=\textwidth]{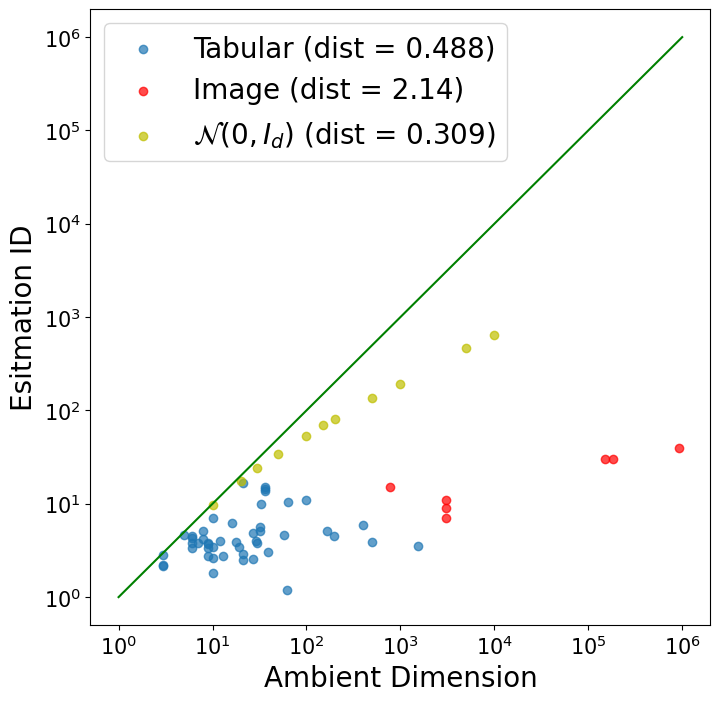}
    \end{subfigure}
    \caption{(Left, Center) : Estimation of ID according to changes in $\rho$. The x-axis represents the value of $\rho$ and the y-axis represents dimension. The horizontal line represents the ambient dimension. (Right) : Log-scale scatter plot for ID estimate and ambient dimension using the TwoNN method of real dataset and synthetic dataset sampled by $\mathcal{N}(0,I_d)$. The estimated ID for the image dataset are from Table 1 in \citet{pope2021the}. The distance recorded in the legend represents the average value of the distance between each point and the green line indicating when the ID estimate and ambient dimension are the same. Each blue/red point corresponds to one ADBench/image dataset.}
    \label{fig:amb_id}
\end{center}
\end{figure*}

To this end, we use overall feature correlation as a proxy for characterizing the differences between the two domains in terms of heterogeneity versus homogeneity, where a higher degree of correlation among features indicates greater homogeneity and a reduced tendency toward independent variation across dimensions. Since correlation indicates the strength of the relationship between features, a high correlation can be interpreted as a strong tendency to follow a specific pattern (e.g., a non-linear relationship), which makes it possible to determine whether the features are heterogeneous or homogeneous. However, common correlation measures (e.g., the Pearson correlation coefficient) are based on pairwise relationships and do not capture the global correlation structure. To address this, we quantify correlation indirectly by considering data $\bold x \in \mathbb{R}^d$ and showing that, as the strength of the correlation increases, the intrinsic dimension (ID) $d'$ becomes smaller relative to the ambient dimension $d$. We refer the reader to \citet{camastra2016intrinsic} for a formal definition of ID. Since the exact ID is unknown, we estimate it using MLE \citep{levina2004maximum} and TwoNN \citep{facco2017estimating}, two popular nearest-neighbor-based estimators.

First, we illustrate the relationship between correlation and ID using a toy example where $X \sim \mathcal{N}(0, \Sigma)$ with a $d$-dimensional AR(1) covariance structure, where the $(i,j)$-th entry of $\Sigma$ is $\rho^{|i-j|}$ for $\rho \in [0,1]$. As $\rho$ approaches 1, correlations among variables become stronger. We then use the TwoNN and MLE ($k=10$) ID estimators to compute the ID of $X$ for ambient dimensions 10 and 50 and plot the resulting values as a function of $\rho$ in the left and center subplots of Figure \ref{fig:amb_id}.  The left and center subplots in Figure \ref{fig:amb_id} show that both ID estimators yield smaller estimates as $\rho$ increases. Therefore, it can be seen that stronger correlation between variables leads the ID to take values considerably smaller than the ambient dimension. The center plot of Figure \ref{fig:amb_id} shows that, even when $\rho=0$, the estimators underestimate the ID. This is consistent with prior observations that these estimators tend to underestimate when the true ID is large \citep{ansuini2019intrinsic,sharma2022scaling}. In such cases, it is reasonable to interpret the estimate as a lower bound on the true ID. In addition, we report in Appendix~\ref{sec:id_robustness} a robustness study of ID estimates on tabular datasets, comparing TwoNN and MLE under various experimental settings. Our analysis places particular emphasis on TwoNN, for which we observe strong robustness to common experimental choices, including sub-sampling and feature scaling.

Based on these results, we define the ratio of the intrinsic dimension to the ambient dimension, which we call the $d$-Ratio, as a measure of overall feature correlation. A higher degree of feature correlation results in a lower intrinsic dimension estimate, and thus a smaller $d$-Ratio, whereas weaker correlation yields a larger $d$-Ratio. While the $d$-Ratio serves as an indirect proxy rather than a direct measure of correlation, it captures the global correlation structure that pairwise measures such as the Pearson coefficient cannot. Subsequently, to validate the findings from synthetic data experiments on real datasets, we estimate the ID of real-world image and tabular datasets and compare these estimates to their corresponding ambient dimensions. To compare real-world datasets, we report ID estimates in Table \ref{tab:id_estimate} for four standard image benchmarks (MNIST, CIFAR-10, CIFAR-100, SVHN) and for the ADBench tabular datasets, using MLE ($k=20$) and TwoNN. According to Table \ref{tab:id_estimate}, all four image datasets have a $d$-Ratio of at most 2\%, whereas the tabular datasets exhibit substantially higher $d$-Ratio values compared to the images.

\begin{wraptable}{r}{0.55\textwidth}
\centering
\caption{(Top): ID estimates for real datasets. (Bottom): Among the 28 datasets on which NF-SLT ranks $\geq 3$, the fraction whose $d$-Ratio falls below each threshold. The $d$-Ratio is the ratio of the ID (TwoNN) to the ambient dimension. Image-dataset results follow \citet{pope2021the}.}
\label{tab:id_estimate}
\scriptsize
\setlength{\tabcolsep}{4pt}
\resizebox{\linewidth}{!}{
\begin{tabular}{lcccc}
\toprule
Dataset & MNIST & CIFAR-10 & CIFAR-100 & SVHN \\
\midrule
MLE   & 13 & 26 & 23 & 19 \\
TwoNN & 15 & 11 &  9 &  7 \\
\midrule
$d$-Ratio & 0.019 & 0.004 & 0.003 & 0.002 \\
\midrule[0.8pt]
Dataset & magicgamma & satellite & landsat & waveform \\
\midrule
MLE   & 5 & 9 & 11 & 16 \\
TwoNN & 7 & 15 & 14 & 17 \\
\midrule
$d$-Ratio & 0.700 & 0.417 & 0.389 & 0.810 \\
\bottomrule
\end{tabular}
}\\[0.15cm]
\resizebox{\linewidth}{!}{
\begin{tabular}{lcccc}
\toprule
$d$-Ratio Threshold & 0.1 & 0.2 & 0.3 & 0.4 \\
\midrule
Rank $\geq 3$ Dataset Ratio & 0.179 & 0.429 & 0.500 & 0.750 \\
\midrule[0.8pt]
$d$-Ratio Threshold & 0.5 & 0.6 & 0.7 & 0.8 \\
\midrule
Rank $\geq 3$ Dataset Ratio & 0.821 & 0.857 & 0.929 & 0.929 \\
\bottomrule
\end{tabular}
}
\end{wraptable}

Additionally, we recorded a log-scale plot of the ID estimates against the ambient dimensions in Figure \ref{fig:amb_id} to check the tendency of the tabular and image datasets. In Figure \ref{fig:amb_id}, the average distance from each point to the green line (where the ID equals the ambient dimension) is larger for image datasets than for tabular datasets. Visually, the blue points (tabular datasets) lie much closer to the green line than the red points (image datasets). The yellow points, corresponding to samples from $\mathcal{N}(0,I_d)$, represent theoretically uncorrelated data and also cluster near the green line. These results indicate that tabular datasets have IDs closer to their ambient dimensions than image datasets, implying weaker overall feature correlation.

Furthermore, for the 28 datasets on which NF-SLT does not achieve high performance (rank $\ge$ 3), Table \ref{tab:id_estimate} reports the fraction whose $d$-Ratio falls below a given threshold. These results indicate that the datasets on which NF-SLT fails to achieve high performance are mostly those with a low $d$-Ratio (below 0.5), even within the tabular domain. Therefore, we conclude that one factor behind NF-SLT's strong performance on tabular data is the heterogeneous nature of its features. We further argue that this effect may act in combination with the improvement in anomaly detection performance obtained by using MLP-based flows, as reported in \citet{schirrmeister2020understanding}. 

Moreover, to account for the absence of counterintuitive phenomena on CV/NLP embeddings in Section \ref{experiment}, we estimate the intrinsic dimension of the ADBench CIFAR-10 and SVHN embedding representations using TwoNN. The estimated intrinsic dimensions are 23 and 18, respectively, whereas the ambient embedding dimension is 512 (smaller than the 3072-dimensional raw pixel space). Despite the reduced ambient dimension, the embeddings exhibit higher intrinsic dimensionality than the original images, implying a larger $d$-Ratio. This suggests that the embedding features are less strongly correlated and, according to the TwoNN ID estimates, behave as if they lay on a higher-dimensional manifold than the raw pixels. Consequently, the effects of high dimensionality and strong feature correlations that degrade likelihood-based ranking are substantially reduced, allowing NF-SLT to perform effectively on these embeddings. This explanation is consistent with \citet{kirichenko2020normalizing}, which reported that the counterintuitive phenomenon is alleviated when using semantic embedding representations instead of raw pixels. Hence, unlike images, tabular data generally exhibit low feature correlation, which contributes to their heterogeneous nature and makes it difficult to satisfy the criterion in Definition \ref{def:counterintuitive} that the density model attains low relative AUROC, being worse than a large fraction of baselines by a non-trivial margin.

\section{Conclusion}
\label{conclusion} 
This paper examined whether the counterintuitive phenomenon in image anomaly detection also appears in tabular data. Although prior work reported significant likelihood overlap on tabular data, we formally define the counterintuitive phenomenon in a domain-agnostic manner and demonstrate through extensive experiments on 47 tabular datasets against 13 baselines that the phenomenon is consistently rare in general tabular data, with NF-SLT achieving competitive performance. We further provide theoretical and empirical analyses attributing this to the lower dimensionality and weaker feature correlation of tabular data. For future work, we encourage the development of flow architectures that better capture semantic information in tabular data, as well as theoretical and empirical studies extending these methods to high-dimensional tabular datasets with correlation structures comparable to those in image data.

\subsubsection*{Broader Impact Statement}
\label{broader_impact}
This work contributes to a clearer understanding of when likelihood-based anomaly detection with deep generative models is reliable. By demonstrating that the counterintuitive likelihood phenomenon rarely arises in general tabular settings, our findings may help practitioners better assess the applicability of simple likelihood tests in real-world anomaly detection tasks, such as monitoring and risk detection in structured data. We do not anticipate direct negative societal impacts from this work; however, as with other anomaly detection methods, care should be taken when deploying likelihood-based models in high-stakes settings, particularly when data distributions exhibit high dimensionality or strong feature correlations.



\bibliography{main}

@inproceedings{yin2024mcm,
  title={MCM: Masked Cell Modeling for Anomaly Detection in Tabular Data},
  author={Yin, Jiaxin and Qiao, Yuanyuan and Zhou, Zitang and Wang, Xiangchao and Yang, Jie},
  booktitle={The Twelfth International Conference on Learning Representations},
  year={2024}
}

@article{fleury2007stability,
  title={A stability result for mean width of Lp-centroid bodies},
  author={Fleury, Bruno and Gu{\'e}don, Olivier and Paouris, Grigoris},
  journal={Advances in Mathematics},
  volume={214},
  number={2},
  pages={865--877},
  year={2007},
  publisher={Elsevier}
}

@inproceedings{guedon2014concentration,
  title={Concentration phenomena in high dimensional geometry},
  author={Gu{\'e}don, Olivier},
  booktitle={ESAIM: Proceedings},
  volume={44},
  pages={47--60},
  year={2014},
  organization={EDP Sciences}
}

@article{klartag2007central,
  title={A central limit theorem for convex sets},
  author={Klartag, Bo'az},
  journal={Inventiones mathematicae},
  volume={168},
  number={1},
  pages={91--131},
  year={2007},
  publisher={Springer}
}

@article{guedon2011interpolating,
  title={Interpolating thin-shell and sharp large-deviation estimates for lsotropic log-concave measures},
  author={Gu{\'e}don, Olivier and Milman, Emanuel},
  journal={Geometric and Functional Analysis},
  volume={21},
  number={5},
  pages={1043--1068},
  year={2011},
  publisher={Springer}
}

@article{zhao2019pyod,
  title={Pyod: A python toolbox for scalable outlier detection},
  author={Zhao, Yue and Nasrullah, Zain and Li, Zheng},
  journal={Journal of machine learning research},
  volume={20},
  number={96},
  pages={1--7},
  year={2019}
}

@inproceedings{shyu2003novel,
  title={A novel anomaly detection scheme based on principal component classifier},
  author={Shyu, Mei-Ling and Chen, Shu-Ching and Sarinnapakorn, Kanoksri and Chang, LiWu},
  booktitle={Proceedings of the IEEE foundations and new directions of data mining workshop},
  pages={172--179},
  year={2003},
  organization={IEEE Press Piscataway, NJ, USA}
}

@inproceedings{breunig2000lof,
  title={LOF: identifying density-based local outliers},
  author={Breunig, Markus M and Kriegel, Hans-Peter and Ng, Raymond T and Sander, J{\"o}rg},
  booktitle={Proceedings of the 2000 ACM SIGMOD international conference on Management of data},
  pages={93--104},
  year={2000}
}

@inproceedings{liu2008isolation,
  title={Isolation forest},
  author={Liu, Fei Tony and Ting, Kai Ming and Zhou, Zhi-Hua},
  booktitle={2008 eighth ieee international conference on data mining},
  pages={413--422},
  year={2008},
  organization={IEEE}
}

@article{scholkopf1999support,
  title={Support vector method for novelty detection},
  author={Sch{\"o}lkopf, Bernhard and Williamson, Robert C and Smola, Alex and Shawe-Taylor, John and Platt, John},
  journal={Advances in neural information processing systems},
  volume={12},
  year={1999}
}

@inproceedings{li2020copod,
  title={COPOD: copula-based outlier detection},
  author={Li, Zheng and Zhao, Yue and Botta, Nicola and Ionescu, Cezar and Hu, Xiyang},
  booktitle={2020 IEEE international conference on data mining (ICDM)},
  pages={1118--1123},
  year={2020},
  organization={IEEE}
}

@article{li2022ecod,
  title={Ecod: Unsupervised outlier detection using empirical cumulative distribution functions},
  author={Li, Zheng and Zhao, Yue and Hu, Xiyang and Botta, Nicola and Ionescu, Cezar and Chen, George H},
  journal={IEEE Transactions on Knowledge and Data Engineering},
  volume={35},
  number={12},
  pages={12181--12193},
  year={2022},
  publisher={IEEE}
}

@inproceedings{zong2018deep,
  title={Deep autoencoding gaussian mixture model for unsupervised anomaly detection},
  author={Zong, Bo and Song, Qi and Min, Martin Renqiang and Cheng, Wei and Lumezanu, Cristian and Cho, Daeki and Chen, Haifeng},
  booktitle={International conference on learning representations},
  year={2018}
}

@inproceedings{ruff2018deep,
  title={Deep one-class classification},
  author={Ruff, Lukas and Vandermeulen, Robert and Goernitz, Nico and Deecke, Lucas and Siddiqui, Shoaib Ahmed and Binder, Alexander and M{\"u}ller, Emmanuel and Kloft, Marius},
  booktitle={International conference on machine learning},
  pages={4393--4402},
  year={2018},
  organization={PMLR}
}

@inproceedings{
Bergman2020Classification-Based,
title={Classification-Based Anomaly Detection for General Data},
author={Liron Bergman and Yedid Hoshen},
booktitle={International Conference on Learning Representations},
year={2020},
url={https://openreview.net/forum?id=H1lK_lBtvS}
}

@inproceedings{qiu2021neural,
  title={Neural transformation learning for deep anomaly detection beyond images},
  author={Qiu, Chen and Pfrommer, Timo and Kloft, Marius and Mandt, Stephan and Rudolph, Maja},
  booktitle={International conference on machine learning},
  pages={8703--8714},
  year={2021},
  organization={PMLR}
}

@inproceedings{shenkar2022anomaly,
  title={Anomaly detection for tabular data with internal contrastive learning},
  author={Shenkar, Tom and Wolf, Lior},
  booktitle={International conference on learning representations},
  year={2022}
}

@inproceedings{dinh2015nice,
  title     = {NICE: Non-linear Independent Components Estimation},
  author    = {Dinh, Laurent and Krueger, David and Bengio, Yoshua},
  booktitle = {International Conference on Learning Representations (ICLR) Workshop},
  year      = {2015},
  url       = {https://arxiv.org/abs/1410.8516}
}

@inproceedings{dinh2017density,
  title={Density estimation using Real NVP},
  author={Dinh, Laurent and Sohl-Dickstein, Jascha and Bengio, Samy},
  booktitle={International Conference on Learning Representations},
  year={2017}
}

@article{han2022adbench,
  title={Adbench: Anomaly detection benchmark},
  author={Han, Songqiao and Hu, Xiyang and Huang, Hailiang and Jiang, Minqi and Zhao, Yue},
  journal={Advances in Neural Information Processing Systems},
  volume={35},
  pages={32142--32159},
  year={2022}
}

@article{shwartz2022tabular,
  title={Tabular data: Deep learning is not all you need},
  author={Shwartz-Ziv, Ravid and Armon, Amitai},
  journal={Information Fusion},
  volume={81},
  pages={84--90},
  year={2022},
  publisher={Elsevier}
}

@inproceedings{kingma2014vae,
  title     = {Auto-Encoding Variational Bayes},
  author    = {Diederik P. Kingma and Max Welling},
  booktitle = {International Conference on Learning Representations (ICLR), Conference Track Proceedings},
  year      = {2014},
  url       = {https://arxiv.org/abs/1312.6114}
}

@article{goodfellow2014generative,
  title={Generative adversarial nets},
  author={Goodfellow, Ian J and Pouget-Abadie, Jean and Mirza, Mehdi and Xu, Bing and Warde-Farley, David and Ozair, Sherjil and Courville, Aaron and Bengio, Yoshua},
  journal={Advances in neural information processing systems},
  volume={27},
  year={2014}
}

@inproceedings{
nalisnick2018do,
title={Do Deep Generative Models Know What They Don't Know? },
author={Eric Nalisnick and Akihiro Matsukawa and Yee Whye Teh and Dilan Gorur and Balaji Lakshminarayanan},
booktitle={International Conference on Learning Representations},
year={2019},
url={https://openreview.net/forum?id=H1xwNhCcYm},
}

@article{kirichenko2020normalizing,
  title={Why normalizing flows fail to detect out-of-distribution data},
  author={Kirichenko, Polina and Izmailov, Pavel and Wilson, Andrew G},
  journal={Advances in neural information processing systems},
  volume={33},
  pages={20578--20589},
  year={2020}
}

@inproceedings{
Serrà2020Input,
title={Input Complexity and Out-of-distribution Detection with Likelihood-based Generative Models},
author={Joan Serrà and David Álvarez and Vicenç Gómez and Olga Slizovskaia and José F. Núñez and Jordi Luque},
booktitle={International Conference on Learning Representations},
year={2020},
url={https://openreview.net/forum?id=SyxIWpVYvr}
}

@inproceedings{osada2024understanding,
  title={Understanding Likelihood of Normalizing Flow and Image Complexity through the Lens of Out-of-Distribution Detection},
  author={Osada, Genki and Takahashi, Tsubasa and Nishide, Takashi},
  booktitle={Proceedings of the AAAI Conference on Artificial Intelligence},
  volume={38},
  number={19},
  pages={21492--21500},
  year={2024}
}

@article{ren2019likelihood,
  title={Likelihood ratios for out-of-distribution detection},
  author={Ren, Jie and Liu, Peter J and Fertig, Emily and Snoek, Jasper and Poplin, Ryan and Depristo, Mark and Dillon, Joshua and Lakshminarayanan, Balaji},
  journal={Advances in neural information processing systems},
  volume={32},
  year={2019}
}

@article{nalisnick2019detecting,
  title={Detecting out-of-distribution inputs to deep generative models using typicality},
  author={Nalisnick, Eric and Matsukawa, Akihiro and Teh, Yee Whye and Lakshminarayanan, Balaji},
  journal={arXiv preprint arXiv:1906.02994},
  year={2019}
}

@inproceedings{
kamkari2024a,
title={A Geometric Explanation of the Likelihood {OOD} Detection Paradox},
author={Hamidreza Kamkari and Brendan Leigh Ross and Jesse C. Cresswell and Anthony L. Caterini and Rahul Krishnan and Gabriel Loaiza-Ganem},
booktitle={Forty-first International Conference on Machine Learning},
year={2024},
url={https://openreview.net/forum?id=EVMzCKLpdD}
}

@inproceedings{caterini2022entropic,
  title={Entropic issues in likelihood-based ood detection},
  author={Caterini, Anthony L and Loaiza-Ganem, Gabriel},
  booktitle={I (Still) Can't Believe It's Not Better! Workshop at NeurIPS 2021},
  pages={21--26},
  year={2022},
  organization={PMLR}
}

@article{kingma2018glow,
  title={Glow: Generative flow with invertible 1x1 convolutions},
  author={Kingma, Durk P and Dhariwal, Prafulla},
  journal={Advances in neural information processing systems},
  volume={31},
  year={2018}
}

@article{chen2019residual,
  title={Residual flows for invertible generative modeling},
  author={Chen, Ricky TQ and Behrmann, Jens and Duvenaud, David K and Jacobsen, J{\"o}rn-Henrik},
  journal={Advances in Neural Information Processing Systems},
  volume={32},
  year={2019}
}

@article{durkan2019neural,
  title={Neural spline flows},
  author={Durkan, Conor and Bekasov, Artur and Murray, Iain and Papamakarios, George},
  journal={Advances in neural information processing systems},
  volume={32},
  year={2019}
}

@inproceedings{rezende2015variational,
  title={Variational inference with normalizing flows},
  author={Rezende, Danilo and Mohamed, Shakir},
  booktitle={International conference on machine learning},
  pages={1530--1538},
  year={2015},
  organization={PMLR}
}

@article{krizhevsky2009learning,
  title={Learning multiple layers of features from tiny images},
  author={Krizhevsky, Alex and Hinton, Geoffrey and others},
  year={2009},
  publisher={Toronto, ON, Canada}
}

@inproceedings{netzer2011reading,
  title={Reading digits in natural images with unsupervised feature learning},
  author={Netzer, Yuval and Wang, Tao and Coates, Adam and Bissacco, Alessandro and Wu, Baolin and Ng, Andrew Y and others},
  booktitle={NIPS workshop on deep learning and unsupervised feature learning},
  volume={2011},
  number={2},
  pages={4},
  year={2011},
  organization={Granada}
}

@inproceedings{behrmann2019invertible,
  title={Invertible residual networks},
  author={Behrmann, Jens and Grathwohl, Will and Chen, Ricky TQ and Duvenaud, David and Jacobsen, J{\"o}rn-Henrik},
  booktitle={International conference on machine learning},
  pages={573--582},
  year={2019},
  organization={PMLR}
}

@article{wold1987principal,
  title={Principal component analysis},
  author={Wold, Svante and Esbensen, Kim and Geladi, Paul},
  journal={Chemometrics and intelligent laboratory systems},
  volume={2},
  number={1-3},
  pages={37--52},
  year={1987},
  publisher={Elsevier}
}

@inproceedings{zhang2024deep,
  title={Deep orthogonal hypersphere compression for anomaly detection},
  author={Zhang, Yunhe and Sun, Yan and Cai, Jinyu and Fan, Jicong},
  booktitle={The Twelfth International Conference on Learning Representations},
  year={2024}
}

@article{pedregosa2011scikit,
  title={Scikit-learn: Machine learning in Python},
  author={Pedregosa, Fabian and Varoquaux, Ga{\"e}l and Gramfort, Alexandre and Michel, Vincent and Thirion, Bertrand and Grisel, Olivier and Blondel, Mathieu and Prettenhofer, Peter and Weiss, Ron and Dubourg, Vincent and others},
  journal={the Journal of machine Learning research},
  volume={12},
  pages={2825--2830},
  year={2011},
  publisher={JMLR. org}
}

@article{schirrmeister2020understanding,
  title={Understanding anomaly detection with deep invertible networks through hierarchies of distributions and features},
  author={Schirrmeister, Robin and Zhou, Yuxuan and Ball, Tonio and Zhang, Dan},
  journal={Advances in Neural Information Processing Systems},
  volume={33},
  pages={21038--21049},
  year={2020}
}

@inproceedings{morningstar2021density,
  title={Density of states estimation for out of distribution detection},
  author={Morningstar, Warren and Ham, Cusuh and Gallagher, Andrew and Lakshminarayanan, Balaji and Alemi, Alex and Dillon, Joshua},
  booktitle={International Conference on Artificial Intelligence and Statistics},
  pages={3232--3240},
  year={2021},
  organization={PMLR}
}

@inproceedings{ahmadian2021likelihood,
  title={Likelihood-free Out-of-Distribution Detection with Invertible Generative Models.},
  author={Ahmadian, Amirhossein and Lindsten, Fredrik and Zhou, Zhi-Hua},
  booktitle={IJCAI},
  pages={2119--2125},
  year={2021}
}

@inproceedings{zhang2021understanding,
  title={Understanding failures in out-of-distribution detection with deep generative models},
  author={Zhang, Lily and Goldstein, Mark and Ranganath, Rajesh},
  booktitle={International Conference on Machine Learning},
  pages={12427--12436},
  year={2021},
  organization={PMLR}
}

@article{le2021perfect,
  title={Perfect density models cannot guarantee anomaly detection},
  author={Le Lan, Charline and Dinh, Laurent},
  journal={Entropy},
  volume={23},
  number={12},
  pages={1690},
  year={2021},
  publisher={MDPI}
}

@inproceedings{sun2022out,
  title={Out-of-distribution detection with deep nearest neighbors},
  author={Sun, Yiyou and Ming, Yifei and Zhu, Xiaojin and Li, Yixuan},
  booktitle={International Conference on Machine Learning},
  pages={20827--20840},
  year={2022},
  organization={PMLR}
}

@article{golan2018deep,
  title={Deep anomaly detection using geometric transformations},
  author={Golan, Izhak and El-Yaniv, Ran},
  journal={Advances in neural information processing systems},
  volume={31},
  year={2018}
}

@book{cover1999elements,
  title={Elements of information theory},
  author={Cover, Thomas M},
  year={1999},
  publisher={John Wiley \& Sons}
}

@inproceedings{
loshchilov2017sgdr,
title={{SGDR}: Stochastic Gradient Descent with Warm Restarts},
author={Ilya Loshchilov and Frank Hutter},
booktitle={International Conference on Learning Representations},
year={2017},
url={https://openreview.net/forum?id=Skq89Scxx}
}

@inproceedings{
loshchilov2018decoupled,
title={Decoupled Weight Decay Regularization},
author={Ilya Loshchilov and Frank Hutter},
booktitle={International Conference on Learning Representations},
year={2019},
url={https://openreview.net/forum?id=Bkg6RiCqY7},
}

@article{facco2017estimating,
  title={Estimating the intrinsic dimension of datasets by a minimal neighborhood information},
  author={Facco, Elena and d’Errico, Maria and Rodriguez, Alex and Laio, Alessandro},
  journal={Scientific reports},
  volume={7},
  number={1},
  pages={12140},
  year={2017},
  publisher={Nature Publishing Group UK London}
}

@article{levina2004maximum,
  title={Maximum likelihood estimation of intrinsic dimension},
  author={Levina, Elizaveta and Bickel, Peter},
  journal={Advances in neural information processing systems},
  volume={17},
  year={2004}
}

@article{battaglia2018relational,
  title={Relational inductive biases, deep learning, and graph networks. arXiv 2018},
  author={Battaglia, Peter W and Hamrick, Jessica B and Bapst, Victor and Sanchez-Gonzalez, Alvaro and Zambaldi, Vinicius and Malinowski, Mateusz and Tacchetti, Andrea and Raposo, David and Santoro, Adam and Faulkner, Ryan and others},
  journal={arXiv preprint arXiv:1806.01261},
  year={2018}
}

@article{camastra2016intrinsic,
  title={Intrinsic dimension estimation: Advances and open problems},
  author={Camastra, Francesco and Staiano, Antonino},
  journal={Information Sciences},
  volume={328},
  pages={26--41},
  year={2016},
  publisher={Elsevier}
}

@inproceedings{
pope2021the,
title={The Intrinsic Dimension of Images and Its Impact on Learning},
author={Phil Pope and Chen Zhu and Ahmed Abdelkader and Micah Goldblum and Tom Goldstein},
booktitle={International Conference on Learning Representations},
year={2021},
url={https://openreview.net/forum?id=XJk19XzGq2J}
}

@article{ansuini2019intrinsic,
  title={Intrinsic dimension of data representations in deep neural networks},
  author={Ansuini, Alessio and Laio, Alessandro and Macke, Jakob H and Zoccolan, Davide},
  journal={Advances in Neural Information Processing Systems},
  volume={32},
  year={2019}
}

@article{sharma2022scaling,
  title={Scaling laws from the data manifold dimension},
  author={Sharma, Utkarsh and Kaplan, Jared},
  journal={Journal of Machine Learning Research},
  volume={23},
  number={9},
  pages={1--34},
  year={2022}
}

@inproceedings{
thimonier2024beyond,
title={Beyond Individual Input for Deep Anomaly Detection on Tabular Data},
author={Hugo Thimonier and Fabrice Popineau and Arpad Rimmel and Bich-Li{\^e}n DOAN},
booktitle={Forty-first International Conference on Machine Learning},
year={2024},
url={https://openreview.net/forum?id=chDpBp2P6b}
}

@article{milligan1985algorithm,
  title={An algorithm for generating artificial test clusters},
  author={Milligan, Glenn W},
  journal={Psychometrika},
  volume={50},
  pages={123--127},
  year={1985},
  publisher={Springer}
}

@article{steinbuss2021benchmarking,
  title={Benchmarking unsupervised outlier detection with realistic synthetic data},
  author={Steinbuss, Georg and B{\"o}hm, Klemens},
  journal={ACM Transactions on Knowledge Discovery from Data (TKDD)},
  volume={15},
  number={4},
  pages={1--20},
  year={2021},
  publisher={ACM New York, NY, USA}
}

@article{aas2009pair,
  title={Pair-copula constructions of multiple dependence},
  author={Aas, Kjersti and Czado, Claudia and Frigessi, Arnoldo and Bakken, Henrik},
  journal={Insurance: Mathematics and economics},
  volume={44},
  number={2},
  pages={182--198},
  year={2009},
  publisher={Elsevier}
}

@misc{hastie2009elements,
  title={The elements of statistical learning: data mining, inference, and prediction},
  author={Hastie, Trevor},
  year={2009},
  publisher={Springer}
}

@article{rayana2016odds,
  title={ODDS library},
  author={Rayana, Shebuti},
  journal={Stony Brook University, Department of Computer Sciences},
  year={2016}
}

@inproceedings{draxler2024universality,
  title={On the Universality of Volume-Preserving and Coupling-Based Normalizing Flows},
  author={Draxler, Felix and Wahl, Stefan and Schnoerr, Christoph and Koethe, Ullrich},
  booktitle={International Conference on Machine Learning},
  pages={11613--11641},
  year={2024},
  organization={PMLR}
}

@inproceedings{marcel2010torchvision,
  title={Torchvision the machine-vision package of torch},
  author={Marcel, S{\'e}bastien and Rodriguez, Yann},
  booktitle={Proceedings of the 18th ACM international conference on Multimedia},
  pages={1485--1488},
  year={2010}
}

@article{hyvarinen2000independent,
  title={Independent component analysis: algorithms and applications},
  author={Hyv{\"a}rinen, Aapo and Oja, Erkki},
  journal={Neural networks},
  volume={13},
  number={4-5},
  pages={411--430},
  year={2000},
  publisher={Elsevier}
}

@article{witten2009penalized,
  title={A penalized matrix decomposition, with applications to sparse principal components and canonical correlation analysis},
  author={Witten, Daniela M and Tibshirani, Robert and Hastie, Trevor},
  journal={Biostatistics},
  volume={10},
  number={3},
  pages={515--534},
  year={2009},
  publisher={Oxford University Press}
}

@article{hanzelmann2013gsva,
  title={GSVA: gene set variation analysis for microarray and RNA-seq data},
  author={H{\"a}nzelmann, Sonja and Castelo, Robert and Guinney, Justin},
  journal={BMC bioinformatics},
  volume={14},
  pages={1--15},
  year={2013},
  publisher={Springer}
}

@article{wang2024biologically,
  title={Biologically informed deep neural networks provide quantitative assessment of intratumoral heterogeneity in post treatment glioblastoma},
  author={Wang, Hairong and Argenziano, Michael G and Yoon, Hyunsoo and Boyett, Deborah and Save, Akshay and Petridis, Petros and Savage, William and Jackson, Pamela and Hawkins-Daarud, Andrea and Tran, Nhan and others},
  journal={npj Digital Medicine},
  volume={7},
  number={1},
  pages={292},
  year={2024},
  publisher={Nature Publishing Group UK London}
}

@article{papamakarios2017masked,
  title={Masked autoregressive flow for density estimation},
  author={Papamakarios, George and Pavlakou, Theo and Murray, Iain},
  journal={Advances in neural information processing systems},
  volume={30},
  year={2017}
}

@article{stimper2023normflows,
  title={normflows: A PyTorch Package for Normalizing Flows},
  author={Stimper, Vincent and Liu, David and Campbell, Andrew and Berenz, Vincent and Ryll, Lukas and Sch{\"o}lkopf, Bernhard and Hern{\'a}ndez-Lobato, Jos{\'e} Miguel},
  journal={Journal of Open Source Software},
  volume={8},
  number={86},
  pages={5361},
  year={2023}
}

@inproceedings{ye2025drl,
  title={DRL: Decomposed Representation Learning for Tabular Anomaly Detection},
  author={Ye, Hangting and Zhao, He and Fan, Wei and Zhou, Mingyuan and dan Guo, Dan and Chang, Yi},
  booktitle={The Thirteenth International Conference on Learning Representations},
  year={2025}
}

@article{fan2010intrinsic,
  title={Intrinsic dimension estimation of data by principal component analysis},
  author={Fan, Mingyu and Gu, Nannan and Qiao, Hong and Zhang, Bo},
  journal={arXiv preprint arXiv:1002.2050},
  year={2010}
}

@article{yeh2009comparisons,
  title={The comparisons of data mining techniques for the predictive accuracy of probability of default of credit card clients},
  author={Yeh, I-Cheng and Lien, Che-hui},
  journal={Expert systems with applications},
  volume={36},
  number={2},
  pages={2473--2480},
  year={2009},
  publisher={Elsevier}
}

@article{kim2026mitigating,
  title={Mitigating the Likelihood Paradox in Flow-based OOD Detection via Entropy Manipulation},
  author={Kim, Donghwan and Yoon, Hyunsoo},
  journal={arXiv preprint arXiv:2602.09581},
  year={2026}
}
\bibliographystyle{tmlr}

\clearpage
\appendix

\section{Definition of OOD Detection and Anomaly Detection}
\label{def_ood_ano}
The goal of anomaly detection is to build a classifier that detects abnormal, out-of-distribution instances \citep{golan2018deep}. In this sense, we follow the common view that anomaly detection and OOD detection share the same objective. While OOD detection assumes that anomalies follow a distribution different from the in-distribution, anomaly detection typically does not impose a specific distributional form. Since this distinction does not affect the fundamental objective we study, we treat the two tasks as equivalent in this work.

\section{Rigorous Formulation of Definition~\ref{def:counterintuitive}}
\label{rigorous_def}

\begin{definition}[Occurrence of Counterintuitive Phenomenon]
\label{def:counterintuitive_rig}
Let $\mathbf{x} \sim P$ and $\mathbf{y} \sim Q$, and let $P_{\theta_0}$ denote a generative model that provides an approximately accurate likelihood estimate for $\mathbf{x}$. Let $P_{\theta_i}$ denote the comparison models, which do not necessarily provide likelihood estimates. We assume all models are well trained.

Let $\varphi_{P_{\theta_0}}(\mathbf{x})$ denote the likelihood estimate from $P_{\theta_0}$, and let $\varphi_{P_{\theta_i}}(x)$ denote the test statistic of the $i$-th comparison model. Define the index set of comparison models that outperform $P_{\theta_0}$ as
\begin{equation}
\label{eq:def1_set}
R = \Bigl\{ i \in [k] \mid \Pr_{\mathbf{x} \sim P,\, \mathbf{y} \sim Q}\bigl(\varphi_{P_{\theta_0}}(\mathbf{x}) > \varphi_{P_{\theta_0}}(\mathbf{y})\bigr) < \Pr_{\mathbf{x} \sim P,\, \mathbf{y} \sim Q}\bigl(\varphi_{P_{\theta_i}}(\mathbf{x}) > \varphi_{P_{\theta_i}}(\mathbf{y})\bigr) \Bigr\}.
\end{equation}
We say that a counterintuitive phenomenon occurs if the following two conditions hold:
\begin{enumerate}
    \item A sufficiently large fraction of comparison models outperform $P_{\theta_0}$:
    \begin{equation}
    \label{eq:def1_condition1}
    \frac{1}{k} \sum_{i=1}^{k} \mathbbm{1}\Bigl\{ \Pr_{\mathbf{x} \sim P,\, \mathbf{y} \sim Q}\bigl(\varphi_{P_{\theta_0}}(\mathbf{x}) > \varphi_{P_{\theta_0}}(\mathbf{y})\bigr) < \Pr_{\mathbf{x} \sim P,\, \mathbf{y} \sim Q}\bigl(\varphi_{P_{\theta_i}}(\mathbf{x}) > \varphi_{P_{\theta_i}}(\mathbf{y})\bigr) \Bigr\} > \beta.
    \end{equation}
    \item The minimum performance gap between $P_{\theta_0}$ and the outperforming comparison models exceeds $\gamma$:
    \begin{equation}
    \label{eq:def1_condition2}
    \min_{i \in R} \Bigl( \Pr_{\mathbf{x} \sim P,\, \mathbf{y} \sim Q}\bigl(\varphi_{P_{\theta_i}}(\mathbf{x}) > \varphi_{P_{\theta_i}}(\mathbf{y})\bigr) - \Pr_{\mathbf{x} \sim P,\, \mathbf{y} \sim Q}\bigl(\varphi_{P_{\theta_0}}(\mathbf{x}) > \varphi_{P_{\theta_0}}(\mathbf{y})\bigr) \Bigr) > \gamma.
    \end{equation}
\end{enumerate}
\end{definition}

\section{Additional Description from Section \ref{high_dim_perspective}}
\label{additional_description}
In the following, we examine from a theoretical perspective why anomaly detection with normalizing flows can fail as dimensionality increases. To this end, we review concentration behavior of the Euclidean norm in high dimensions, which plays a key role in understanding how normalizing flows distinguish between normal and anomalous data. This is because the normalizing flow model $f$ is trained on normal data $X$ such that $f(X)$ follows a standard Gaussian distribution, i.e., $f(X) \sim N(0, I)$ and the training objective is to maximize the log-likelihood, $\log P_{N(0,I)}(f(X) \mid \Theta ) \propto -\|f(X)\|_2^2$, where $\Theta$ represents the parameters of the normalizing flow and $\|\cdot\|_2$ denotes the Euclidean norm. Our analysis in Sections \ref{smaller_dim} and \ref{identical_norm} cannot directly analyze the volume term in models that do not preserve volume, such as RealNVP and Glow, but since \citet{osada2024understanding} shows that the volume term and the latent likelihood are positively correlated, we indirectly study the relationship with the volume term as a behavior of the latent vector. In Section \ref{detection_performance_dimension} we report a synthetic experiment in which the likelihood test degrades with the dimension size regardless of the volume term, and we discuss the scope of this observation.

\subsection{Euclidean Norm Gets More Concentrated as Dimension Increases}
\label{smaller_dim}

If $Z \sim \mathcal{N}(0,I_d)$, then $||Z||^2_{2}$ concentrates near $d$. This means that the test statistic $||f(X)||^2_{2}$ is close to $d$ if X is normal data and the normalizing flow model $f$ is well trained so that $f(X)$ follows $\mathcal{N}(0,I_d)$. Thus, if the normalizing flow is well-trained, the transformed normal data $f(X)$ concentrate on a sphere of radius $\sqrt{d}$ (where $d$ is the dimensionality of $X$) due to the tail bound  properties of sum of independent normal distributions. First, we prove Proposition \ref{prop1}, which states that as the dimension increases, the Euclidean norm tends to concentrate near $\sqrt{d}$.

\begin{proposition}
\label{prop1}

If $Z \sim \mathcal{N}(0,I_d)$, then for all $ 0<t<d$ :
\begin{align}
    \Pr\left( \left| ||Z||^{2}_{2}- d\right| \geq t \right) \leq 2e^{-\frac{t^{2}}{8d}}
\end{align}

\end{proposition}

\begin{proof}
Take a random variable $Z \sim \mathcal{N}(0,I_d)$ in $\mathbb{R}^{d}$. Then for each $Z_i \sim N(0,1)$, $\mathbb{E}\left[ e^{\lambda (Z_{i}^{2}-1)} \right] = \frac{e^{-\lambda}}{\sqrt{1-2\lambda}} \leq e^{4\lambda^{2}/2}$ for all $|\lambda|<1/4$. 
Thus, $||Z||^{2}_{2}$ is sub-exponential with parameters $(2\sqrt{d}, 4)$ and by the properties of sub-exponential random variables, we obtain the concentration bound :
$\Pr\left(\left| ||Z||^{2}_{2} - d \right| \geq t \right) \leq 2e^{-\frac{t^{2}}{8d}}$ for $0<t<d$.
\end{proof}

Since we can show by Proposition \ref{prop1} that $\|Z\|_2^2$ is a sub-exponential random variable with parameter containing $d$, we can verify that the probability that the Euclidean norm is observed far from $\sqrt{d}$ has a sub-exponential tail bound. Then, we prove that the variance of the Euclidean norm, when normalized by $d$, vanishes as $d \rightarrow\infty$.

\begin{theorem}[\citet{klartag2007central, fleury2007stability,guedon2014concentration}]
\label{theorem1}

    $X$ is a log-concave isotropic random variable in $\mathbb{R}^{d}$.
    If $\exists \epsilon_{d}(\epsilon_{d}\rightarrow0)$ such that $\Pr \left( \left| \frac{||X||_{2}}{\sqrt{d}} - 1 \right| \geq \epsilon_{d} \right) \leq \epsilon_{d}$, then $\lim_{d\rightarrow \infty}\frac{\textnormal{Var}||X||_2}{d} = 0$
    
\end{theorem}
\begin{proposition}\
\label{prop2}
If $Z \sim \mathcal{N}(0,I_d)$, 
\begin{align}
   \lim_{d \rightarrow \infty} \frac{\textnormal{Var}||Z||_{2}}{d} = 0
\end{align}
\end{proposition}

\begin{proof}
    From Proposition \ref{prop1}, take $s = dt$. Then $\Pr\left(\left|\|Z\|^{2}_{2} - d\right| \geq dt\right) \leq 2e^{-\frac{dt^{2}}{8}}$ for $0 < t < 1$. Since $\left|\frac{\|Z\|_{2}}{\sqrt{d}} - 1\right| \geq t$ implies $\left|\|Z\|^{2}_{2} - d\right| \geq dt$, we have $\Pr\left(\left|\frac{\|Z\|_{2}}{\sqrt{d}} - 1\right| \geq t\right) \leq \Pr\left(\left|\|Z\|^{2}_{2} - d\right| \geq dt\right) \leq 2e^{-\frac{dt^{2}}{8}}$. Since there exists $\epsilon_{d} \to 0$ such that $\max\left\{t,\, 2e^{-\frac{dt^{2}}{8}}\right\} < \epsilon_{d}$, it follows that $\Pr\left(\left|\frac{\|Z\|_{2}}{\sqrt{d}} - 1\right| \geq \epsilon_{d}\right) \leq \epsilon_{d}$. By Theorem \ref{theorem1}, $\lim_{d\to\infty} \frac{\mathrm{Var}\|Z\|_{2}}{d} = 0$.
\end{proof}

Based on Proposition \ref{prop2}, as dimensionality increases, the Euclidean norm of a normal random variable tends to concentrate more quickly relative to the increase in dimensionality. The propositions show that the latent variable of normalizing flow following $\mathcal{N}(0,I_d)$ has a variance that grows more slowly than the increase in dimension, and that its norm concentrates near $\sqrt{d}$. This implies that as dimension increases, the latent variable of a normalizing flow corresponding to the normal data is concentrated in a sphere with radius $\sqrt{d}$. 

Therefore, the model becomes more vulnerable to slight misestimation or error as dimension increases. If the latent representations of anomalous data concentrate at a norm smaller than $\sqrt{d}$, those anomalies receive higher log-likelihood than normal samples under the model, since $\log p_\theta(x) \propto -\|f(x)\|_2^2$, thereby degrading detection performance. As a result, outcomes may emerge that align with assumptions about counterintuitive phenomena.

\subsection{Euclidean Norm Becomes Increasingly Hard to Distinguish in High Dimensions}
\label{identical_norm}
In high-dimensional spaces, the Euclidean norm becomes a less effective statistical measure, as data points from distinct distributions exhibit nearly identical norms. We demonstrate that for isotropic and log-concave random variables, the Euclidean norm concentrates near the same value as the normal distribution, making norm-based separation increasingly difficult as dimensionality increases. Before starting the analysis, we assume that the distribution of latent vectors obtained when passing the anomaly data through the flow is not $\mathcal{N}(0, I_d)$ but is isotropic and log-concave. This is a working assumption, and the latent norm histograms in Appendix \ref{detection_performance_dimension} (Figures \ref{fig:realnvp_normhist_1} and \ref{fig:realnvp_normhist_2}) are consistent with the norm concentration that it implies.

\begin{theorem}[\citet{guedon2011interpolating}]
\label{theorem:guedon}
    For a log-concave and isotropic random variable $X$ in $\mathbb{R}^{d}$, 
    there exist constants $C,c>0$ such that for any $t>0$,
    \[
    \Pr\!\left(\left|\|X\|_{2} - \sqrt{d}\right| \geq t\sqrt{d}\right) 
    \leq Ce^{-c\sqrt{d}\min\{t^{3},t\}}
    \]
\end{theorem}

Theorem~\ref{theorem:guedon} shows that the Euclidean norm of any 
isotropic log-concave random vector concentrates near $\sqrt{d}$ with 
exponentially high probability. Since the normalizing flow is trained so 
that $f(X) \sim \mathcal{N}(0, I_d)$, the latent representation of normal 
data is isotropic and log-concave, and thus $\|f(X)\|_2$ concentrates 
near $\sqrt{d}$. Under our assumption that anomalous latent vectors are 
also isotropic and log-concave, $\|f(Y)\|_2$ concentrates near $\sqrt{d}$ 
as well. Since both norms concentrate at the same value, we can bound 
their difference with high probability. Specifically, suppose both 
$|\|f(X)\|_2 - \sqrt{d}| < t\sqrt{d}$ and $|\|f(Y)\|_2 - \sqrt{d}| 
< t\sqrt{d}$ hold simultaneously. Then,
\begin{align}
    \left|\|f(X)\|_2 - \|f(Y)\|_2\right| 
    &= \left|\|f(X)\|_2 - \sqrt{d} + \sqrt{d} - \|f(Y)\|_2\right| \notag \\
    &\leq \left|\|f(X)\|_2 - \sqrt{d}\right| + \left|\|f(Y)\|_2 - \sqrt{d}\right| \notag \\
    &< 2t\sqrt{d},
\end{align}
which gives the event inclusion
\begin{align}
    \left\{\left|\|f(X)\|_2 - \|f(Y)\|_2\right| \geq 2t\sqrt{d}\right\} 
    \subseteq 
    \left\{\left|\|f(X)\|_2 - \sqrt{d}\right| \geq t\sqrt{d}\right\} 
    \cup 
    \left\{\left|\|f(Y)\|_2 - \sqrt{d}\right| \geq t\sqrt{d}\right\}.
\end{align}
Applying the union bound and Theorem~\ref{theorem:guedon},
\begin{align}
    &\Pr\!\left(\left|\|f(X)\|_2 - \|f(Y)\|_2\right| \geq 2t\sqrt{d}\right) \notag \\
    &\leq \Pr\!\left(\left|\|f(X)\|_2 - \sqrt{d}\right| \geq t\sqrt{d}\right) 
        + \Pr\!\left(\left|\|f(Y)\|_2 - \sqrt{d}\right| \geq t\sqrt{d}\right) \notag \\
    &\leq 2Ce^{-c\sqrt{d}\min(t^3, t)} \;\to\; 0 \quad \text{as } d \to \infty.
\end{align}
Thus, for any fixed $t > 0$, the relative difference $|\|f(X)\|_2 - \|f(Y)\|_2| / \sqrt{d}$ exceeds $2t$ with probability at most $2Ce^{-c\sqrt{d}\min(t^3,t)} \to 0$ as $d \to \infty$. In other words, while both norms grow as $O(\sqrt{d})$, their difference remains of strictly smaller order with high probability, providing a theoretical basis for the degraded performance of normalizing flow-based anomaly detection in high-dimensional settings.

\subsection{Experiments for Sections \ref{smaller_dim} \& \ref{identical_norm} Using Synthetic Data}
\label{detection_performance_dimension}

\begin{figure*}[h!]
\begin{center}
    \begin{subfigure}{0.24\textwidth}
        \centering
        \includegraphics[width=3.9cm]{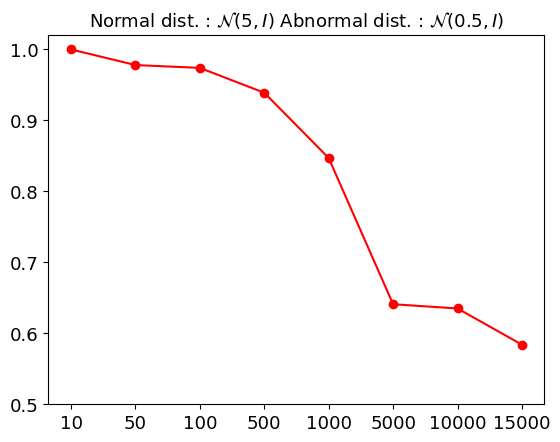}
    \end{subfigure}%
    \begin{subfigure}{0.24\textwidth}
        \centering
        \includegraphics[width=3.9cm]{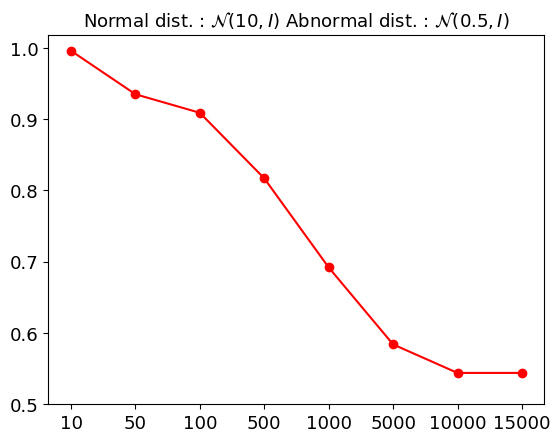}
    \end{subfigure}
    \begin{subfigure}{0.24\textwidth}
        \centering
        \includegraphics[width=3.9cm]{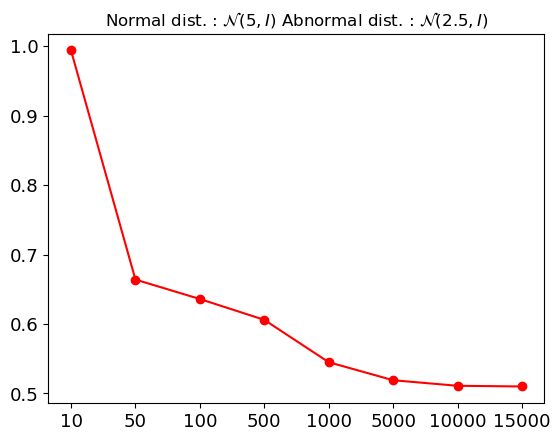}
    \end{subfigure}%
    \begin{subfigure}{0.24\textwidth}
        \centering
        \includegraphics[width=3.9cm]{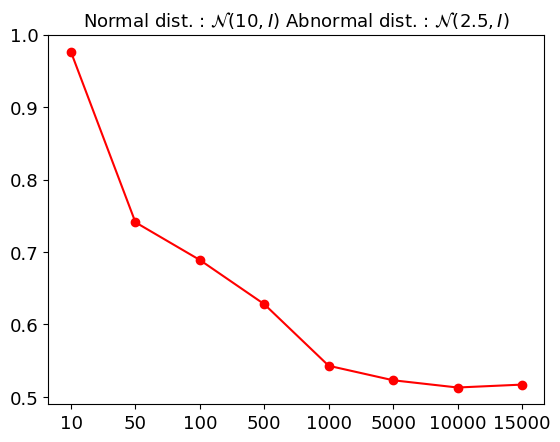}
    \end{subfigure}
    \caption{Performance of NF-SLT implemented by NICE across different dimensions. The y-axis represents AUROC, and the x-axis indicates the dimensionality of the data. The titles of each subfigure specify distributions' parameters. Subplots 1 and 2 have a larger In/Out of Distribution difference than Subplots 3 and 4, and the speed of performance decay as the dimension increases becomes faster as the distribution difference becomes smaller.}
    \label{fig:dimsionality_test}
    \end{center}
\end{figure*}

In this section, we present experiments on synthetic data that exhibit behavior consistent with the results in Sections \ref{smaller_dim} \& \ref{identical_norm}. Although those results hold trivially for zero-centered distributions, our analysis goes beyond this case by confirming that the Euclidean norms remain nearly identical even when the distributions are not centered at zero.

The normal and anomaly distributions are both Gaussian but have different means $\mu$. We first sample $10^4$ data points from the normal Gaussian distribution and use them as the training set, training NICE and RealNVP networks with LeakyReLU activations. We then construct a test set by sampling $10^4$ points from the normal distribution and $10^4$ points from an abnormal Gaussian distribution whose mean differs from that of the normal distribution. A simple likelihood test is performed, and the resulting AUROC as a function of dimensionality is reported for NICE in Figure~\ref{fig:dimsionality_test} and for RealNVP in Figure~\ref{fig:dimsionality_test_realnvp}. The dimensionalities considered in the experiments are $[10, 50, 100, 500, 1000, 5000, 10000, 15000]$. Due to numerical stability issues, the RealNVP experiments were run with the in-distribution mean set to $\mu = 7.5$ instead of $\mu = 10$, and for $\mu = 7.5$ we did not run the experiment at dimension $15000$.

As shown in Figure~\ref{fig:dimsionality_test}, which is the experimental result using NICE, the performance in the 1st and 2nd subfigures degrades significantly between dimensions 1000 and 5000. The 3rd and 4th subfigures begin to degrade in performance at a dimension of 50, smaller than the first two figures, and continue to degrade as the dimension increases. We also conducted an experiment by swapping the means of the in-distribution and out-of-distribution, and confirmed that the AUROC was measured as 1 in most cases. Hence, we found that this phenomenon does not occur symmetrically, similar to the OOD detection experiments on CIFAR-10 and SVHN. In addition, Figure \ref{fig:dimsionality_test_realnvp} (RealNVP results) shows a similar pattern to Figure \ref{fig:dimsionality_test} in that AUROC decreases rapidly as the dimension increases, except for the second subplot. In Figure \ref{fig:dimsionality_test_realnvp}, the second subplot exhibits an inconsistent trend as dimensionality increases. This instability is likely due to the significant mean difference between in-distribution and out-of-distribution, which introduces numerical stability challenges.

\begin{figure*}[t!]
\begin{center}
    \begin{subfigure}{0.24\textwidth}
        \centering
        \includegraphics[width=3.9cm]{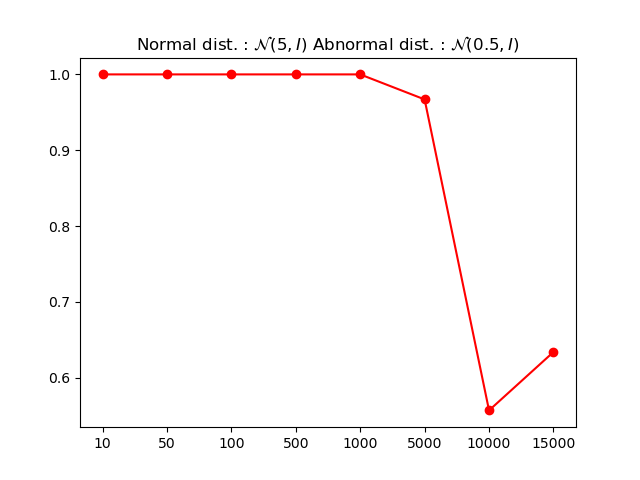}
    \end{subfigure}%
    \begin{subfigure}{0.24\textwidth}
        \centering
        \includegraphics[width=3.9cm]{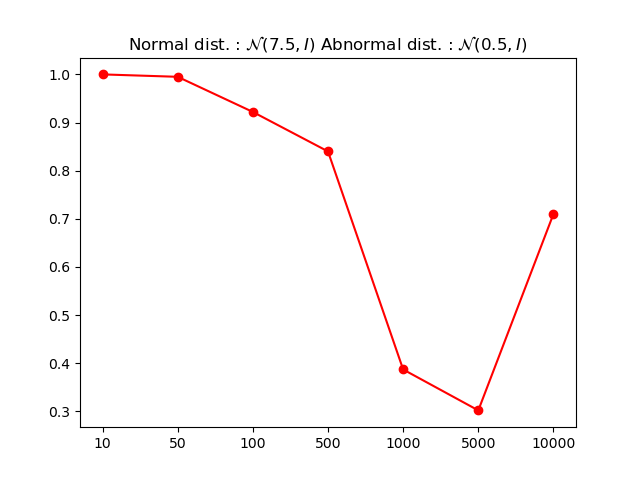}
    \end{subfigure}
    \begin{subfigure}{0.24\textwidth}
        \centering
        \includegraphics[width=3.9cm]{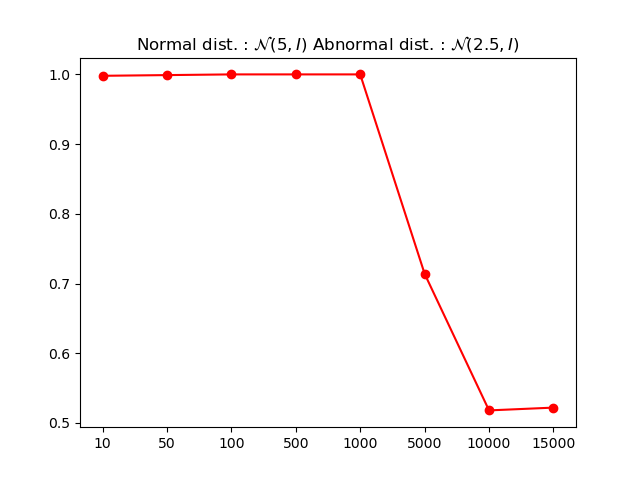}
    \end{subfigure}%
    \begin{subfigure}{0.24\textwidth}
        \centering
        \includegraphics[width=3.9cm]{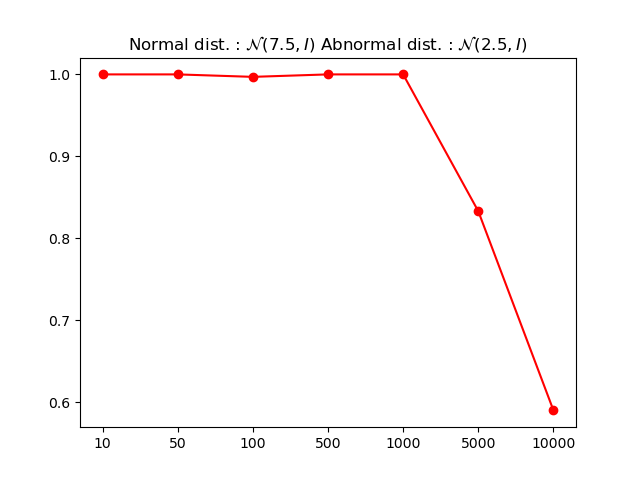}
    \end{subfigure}
    \caption{Performance of NF-SLT implemented by RealNVP across different dimensions. The y-axis represents AUROC, and the x-axis indicates the dimensionality of the data. The titles of each subfigure specify distributions' parameters.}
    \label{fig:dimsionality_test_realnvp}
    \end{center}
\end{figure*}

Additionally, we present a histogram of the log-likelihood of normal and abnormal data along the dimensions from each experiment corresponding to Figure~\ref{fig:dimsionality_test} in Figures \ref{fig:dim_histogram_1} to \ref{fig:dim_histogram_4}. The title of each figure indicates the dimension, the x-axis indicates the log-likelihood, and the y-axis indicates the number of data corresponding to the bin. In addition, the orange and blue histograms represent abnormal data and normal data, respectively. In experiments on RealNVP, to see how the volume term affects the likelihood, we visualized the latent norm and volume as histograms in the experiments for the first and third subplots of Figure 3 in Figures \ref{fig:realnvp_normhist_1} to \ref{fig:realnvp_volumehist_2}. 

According to Equation \ref{eq:nf} for input data $\textbf{x}$, the log-likelihood $\log p(\textbf{x})$ can be calculated with the log-likelihood of the latent $\textbf{z}$ corresponding to $\textbf{x}$ and the volume term for the distribution set as the prior of the flow. Through the histogram results, we can see that when the AUROC approaches 0.5 (i.e., dimension approaches 10000), the latent norm becomes identical and the size of the volume term is reversed. This means that when calculating the log-likelihood with Equation \ref{eq:nf}, the volume term instead has a negative effect on the likelihood test, which is consistent with the results of Figure 4(a) and 4(b) of \citet{nalisnick2018do}. However, in Figures \ref{fig:realnvp_normhist_1} to \ref{fig:realnvp_volumehist_2} , we can see that the scale of the volume term is much smaller than the latent norm, which can be considered as the log-likelihood scale of the latent because $\mathcal{N}(0,I_d)$ is set as the prior of the flow in general and the log-likelihood is proportional to $-||\textbf{z}||^2_2$. Therefore, the phenomenon that AUROC approaches 0.5 when the dimension increases can be interpreted as the fact that in the case of a volume-preserving model such as NICE, the latent norms become identical, and in the case of a non-volume-preserving model such as RealNVP, the latent norms become identical and the volume term shows a behavior that actually hinders the performance, but the effect of the volume term is so small.

This behavior can also be characterized through Equation \ref{eq:likelihood_gap}:
\begin{align}
    \begin{split}
        &\mathbb{E}_{\textbf{x}\sim P}[\log P_{\theta}(\textbf{x})] - \mathbb{E}_{\textbf{x}\sim Q}[\log P_{\theta}(\textbf{x})]\\
        &=  D_{KL}(Q||P_{\theta}) -D_{KL}(P||P_{\theta}) +\mathbb{H}(Q) -\mathbb{H}(P) \\
        &=  D_{KL}(Q||P_{\theta}) -D_{KL}(P||P_{\theta}) \: (\because \mathbb{H}(P)=\mathbb{H}(Q)) \\
    \end{split}
\end{align}
In this case, if $P_\theta$ is a perfect model, the likelihood gap is guaranteed to be positive, so it can be argued that the reason the likelihood gap shrinks is that $D_{KL}(P||P_\theta)$ eventually increases for high-dimensional data (i.e., $P_\theta$ is not a perfect model). Therefore, one might attribute this degradation to the difficulty of optimizing $P_\theta$ toward a perfect model in high-dimensional spaces. Although this phenomenon does not occur symmetrically across dataset pairs, which argues against a purely optimization-based explanation, we cannot exclude that the parametrization or the optimization of the flow contributes to this degradation.

Consequently, these results show that in this setting the Euclidean norms of normal and anomalous latent representations become increasingly difficult to distinguish as the dimension increases, even when the distributions are not zero-centered. In addition, it can be confirmed that the phenomenon occurs regardless of whether the model has a constant volume term or not (i.e., whether the model is volume-preserving) so it can be thought that the presence or absence of a volume term does not help to alleviate the phenomenon. However, we did not observe this phenomenon in real data, and we confirmed that it did not occur when the flow activation function was configured as a hyperbolic tangent rather than a ReLU-like function. Since the entropies of the two distributions coincide in this setting, Theorem~\ref{theorem:gap_expectation_likelihood_indep} does not predict inversion here, and thus we regard this experiment as an observation that motivates further study rather than as evidence that high dimensionality alone causes the degradation. We leave to future research the question of why norms become identical in certain cases when using ReLU-like functions in flow (or why anomalous data are mapped to isotropic and log-concave distributions based on Section \ref{identical_norm}) and whether this phenomenon is reproduced in real data.

\begin{figure}[h!]
    \begin{subfigure}{0.48\textwidth}
        \centering
        \includegraphics[width=\linewidth]{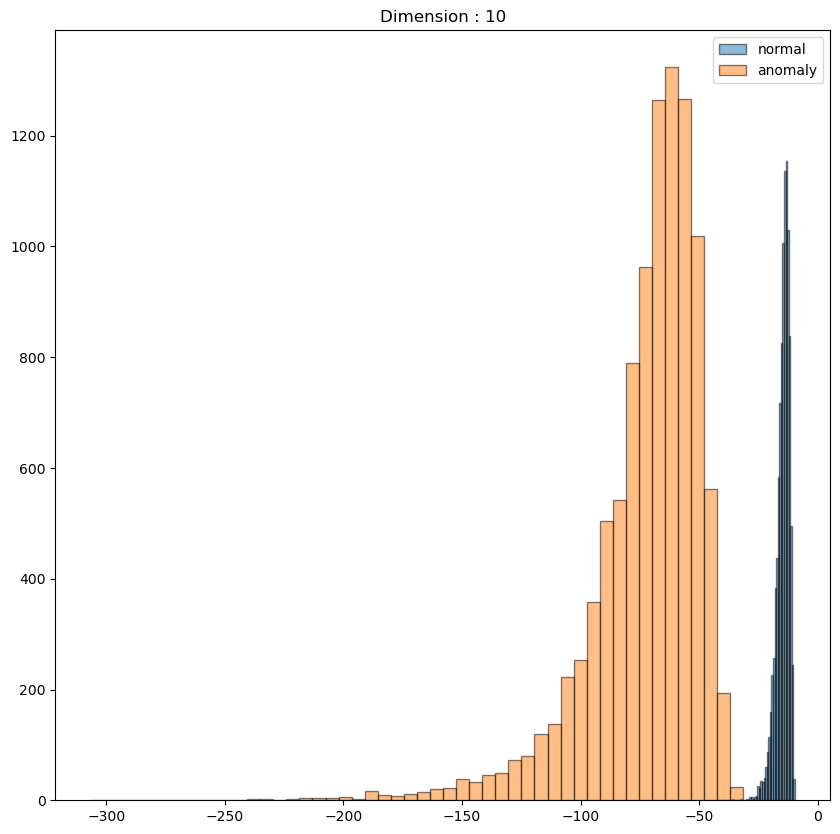}
    \end{subfigure}
    \begin{subfigure}{0.48\textwidth}
        \centering
        \includegraphics[width=\linewidth]{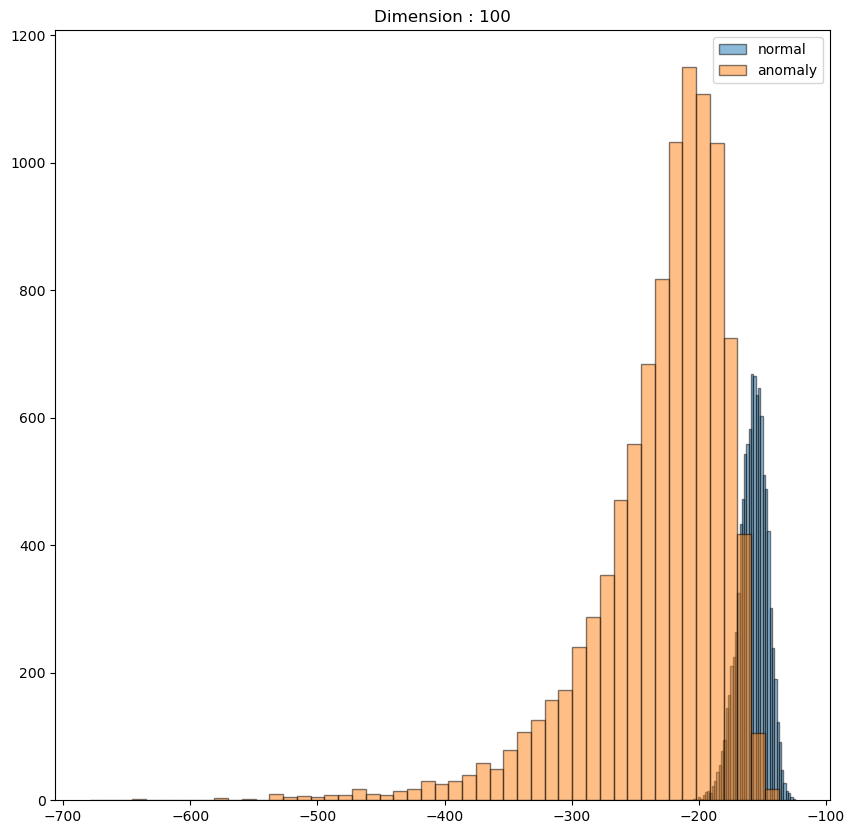}
    \end{subfigure}
    \begin{subfigure}{0.48\textwidth}
        \centering
        \includegraphics[width=\linewidth]{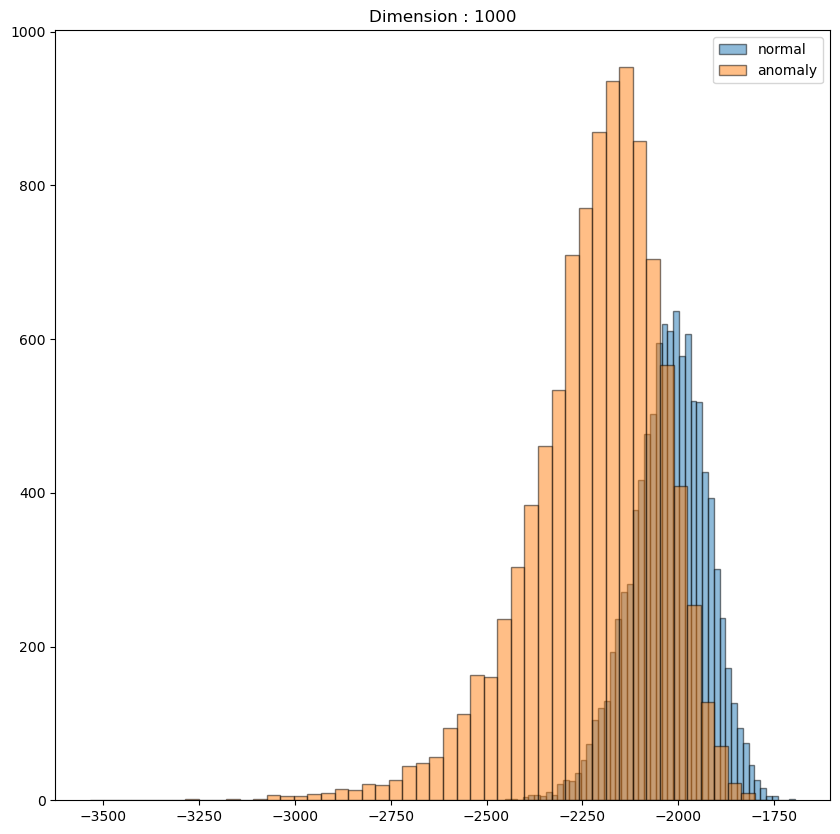}
    \end{subfigure}
    \begin{subfigure}{0.48\textwidth}
        \centering
        \includegraphics[width=\linewidth]{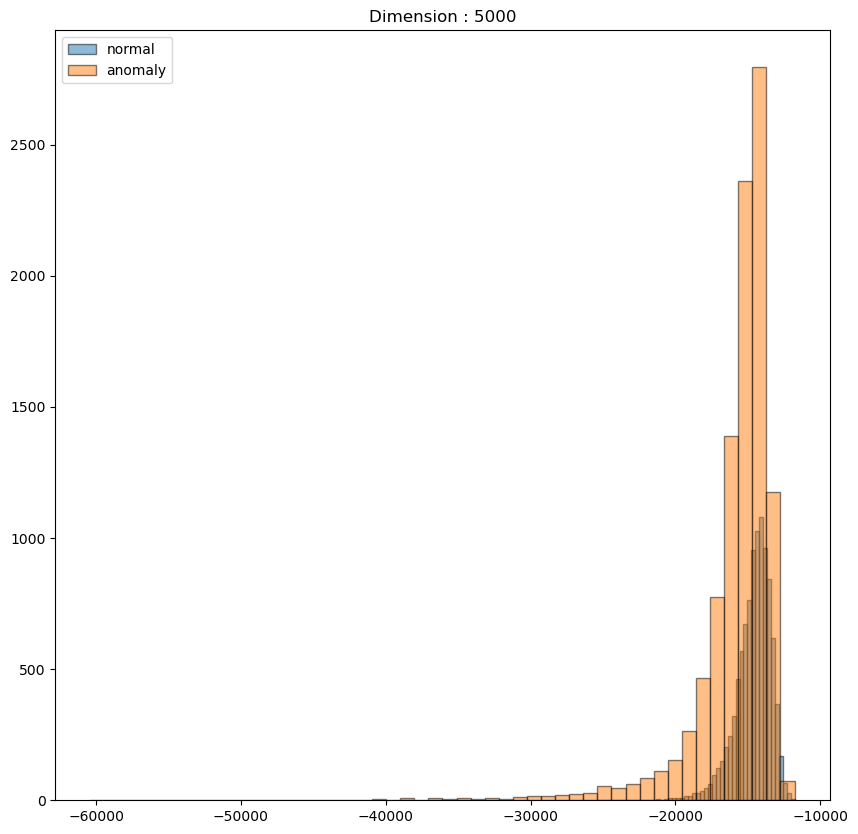}
    \end{subfigure}
    \caption{Histogram of log-likelihood values for the 1st subfigure in Figure~\ref{fig:dimsionality_test}}
    \label{fig:dim_histogram_1}
\end{figure}

\begin{figure}[h!]

    \begin{center}
    \begin{subfigure}{0.48\textwidth}
        \centering
        \includegraphics[width=\linewidth]{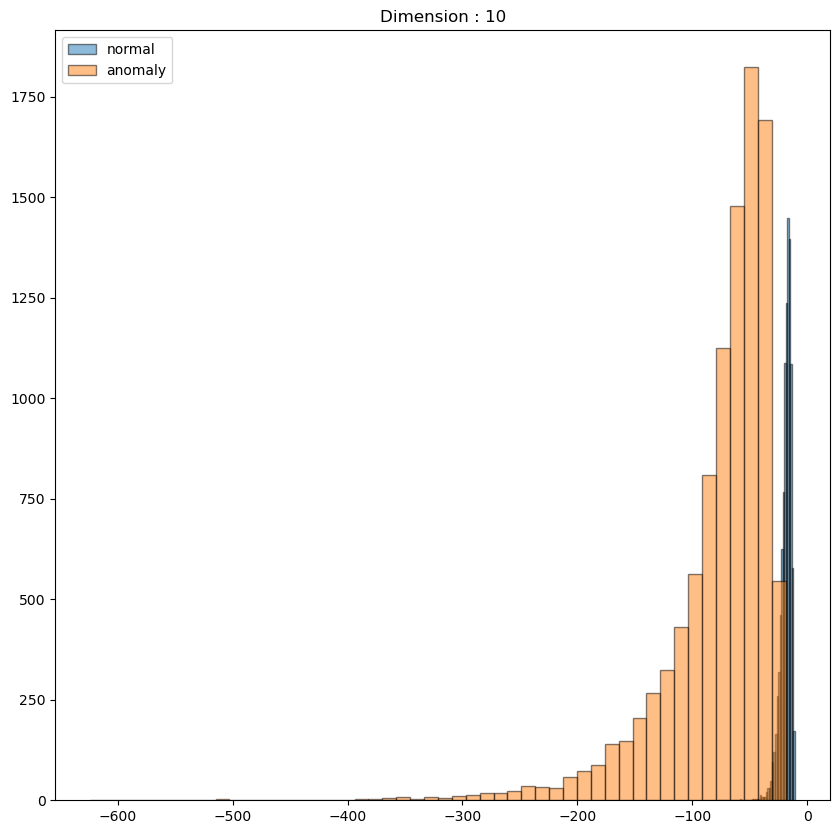}
    \end{subfigure}
    \begin{subfigure}{0.48\textwidth}
        \centering
        \includegraphics[width=\linewidth]{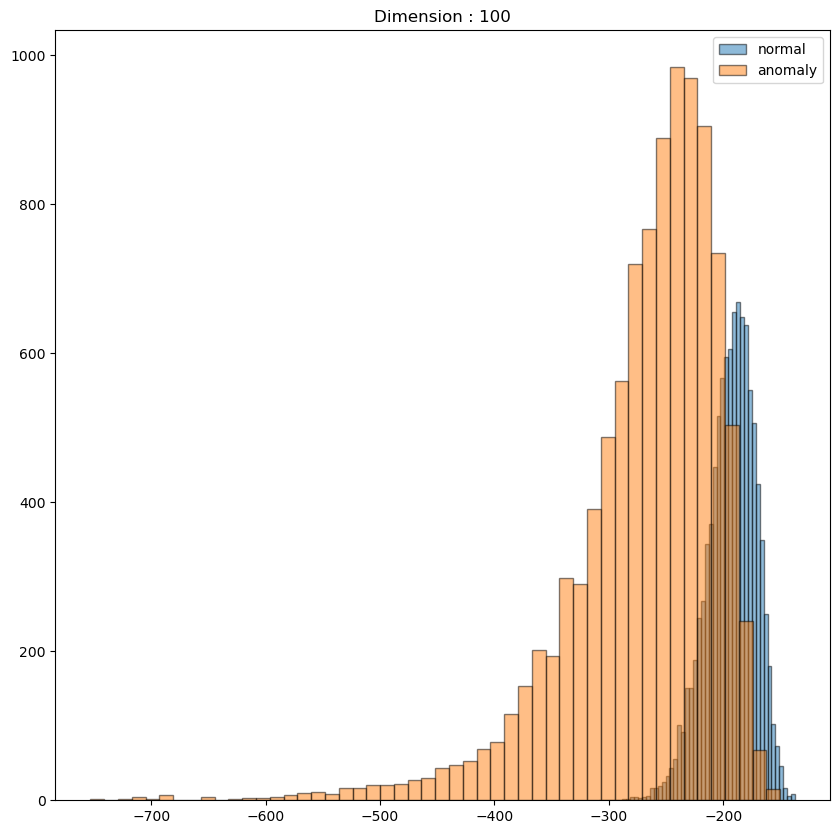}
    \end{subfigure}
    \begin{subfigure}{0.48\textwidth}
        \centering
        \includegraphics[width=\linewidth]{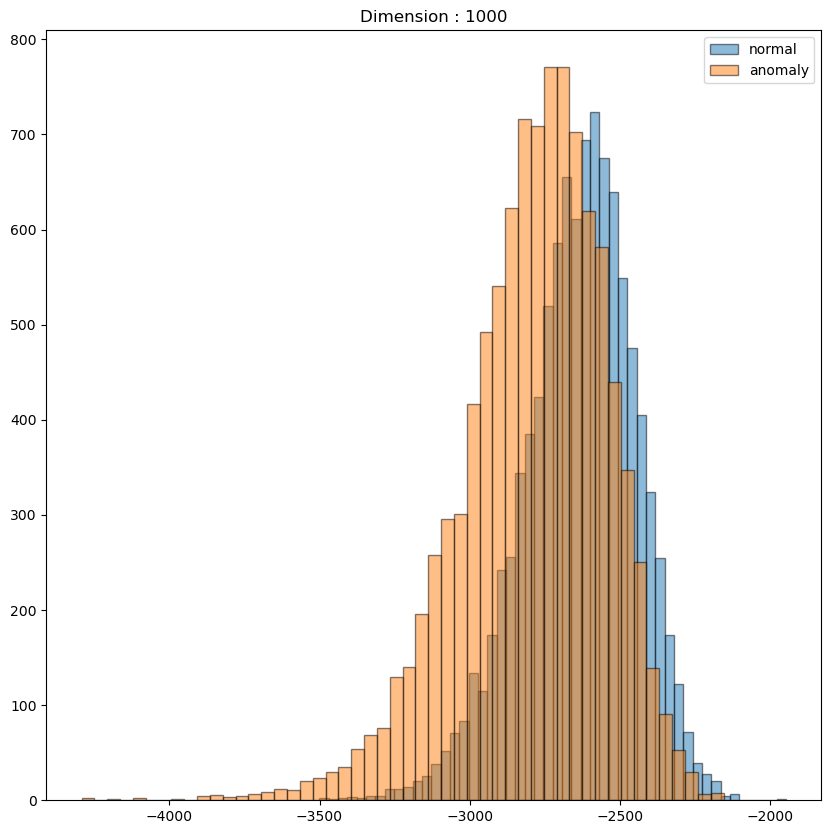}
    \end{subfigure}
    \begin{subfigure}{0.48\textwidth}
        \centering
        \includegraphics[width=\linewidth]{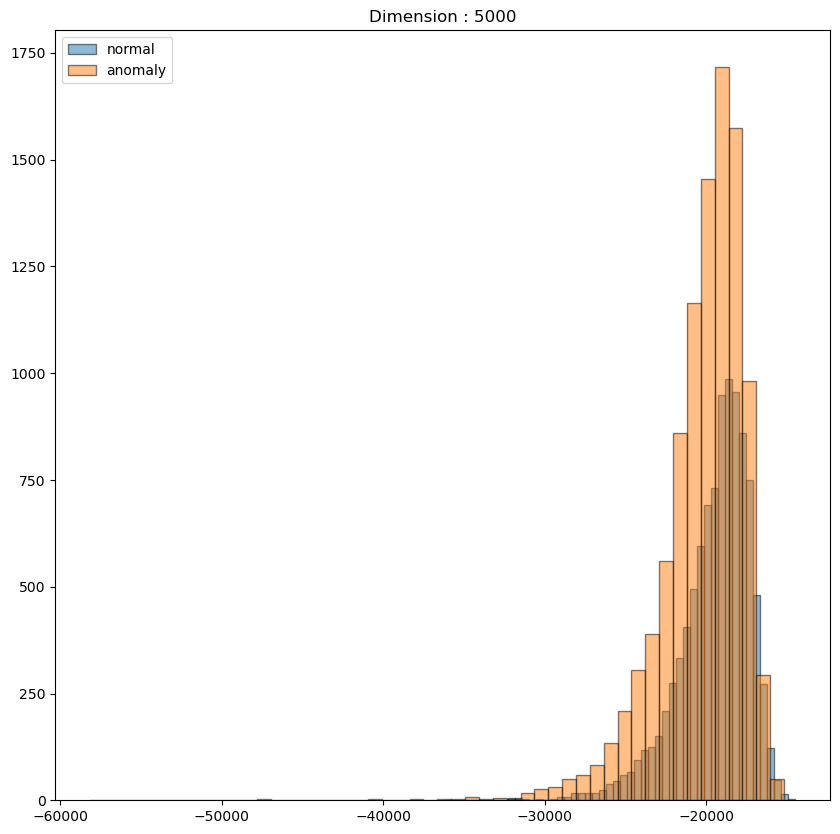}
    \end{subfigure}
    \caption{Histogram of log-likelihood values for the 2nd subfigure in Figure~\ref{fig:dimsionality_test}}
    \label{fig:dim_histogram_2}
    \end{center}
\end{figure}

\begin{figure}[h!]

    \begin{center}
    \begin{subfigure}{0.48\textwidth}
        \centering
        \includegraphics[width=\linewidth]{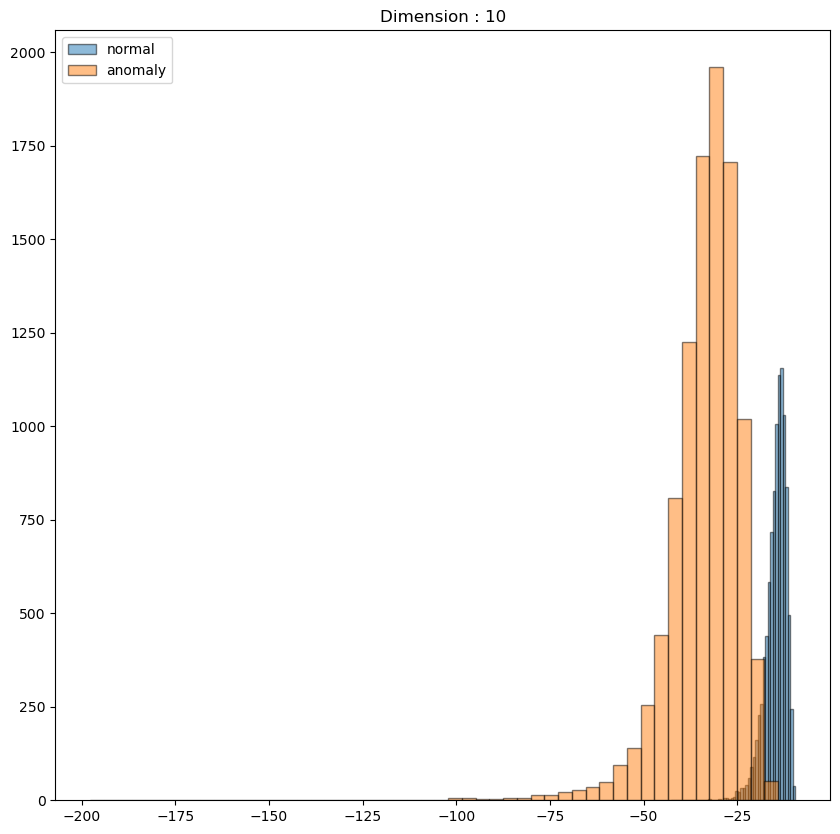}
    \end{subfigure}
    \begin{subfigure}{0.48\textwidth}
        \centering
        \includegraphics[width=\linewidth]{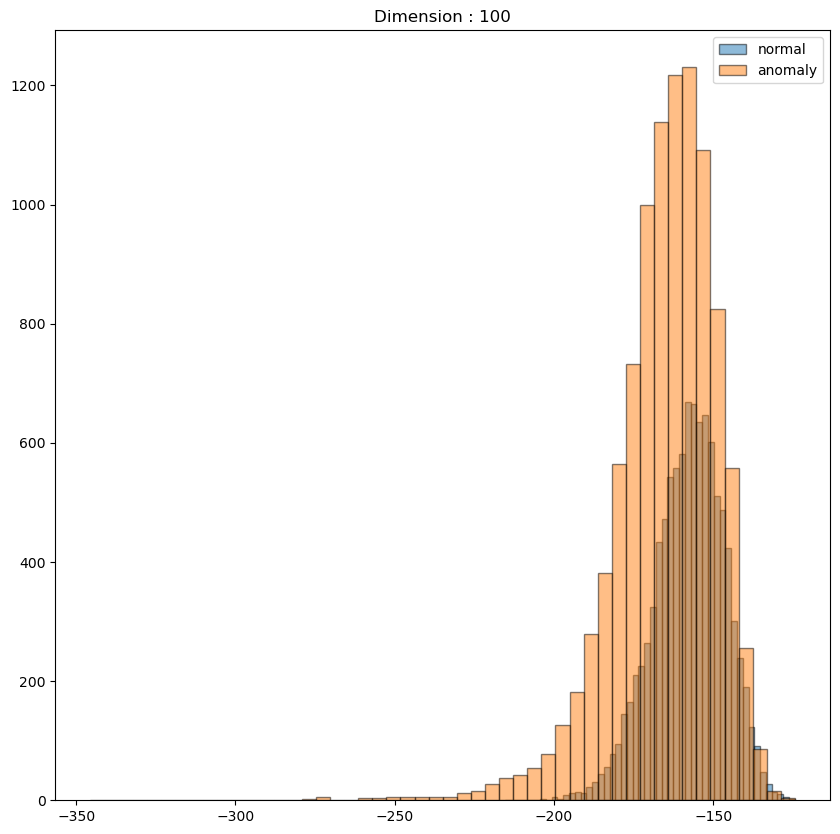}
    \end{subfigure}
    \begin{subfigure}{0.48\textwidth}
        \centering
        \includegraphics[width=\linewidth]{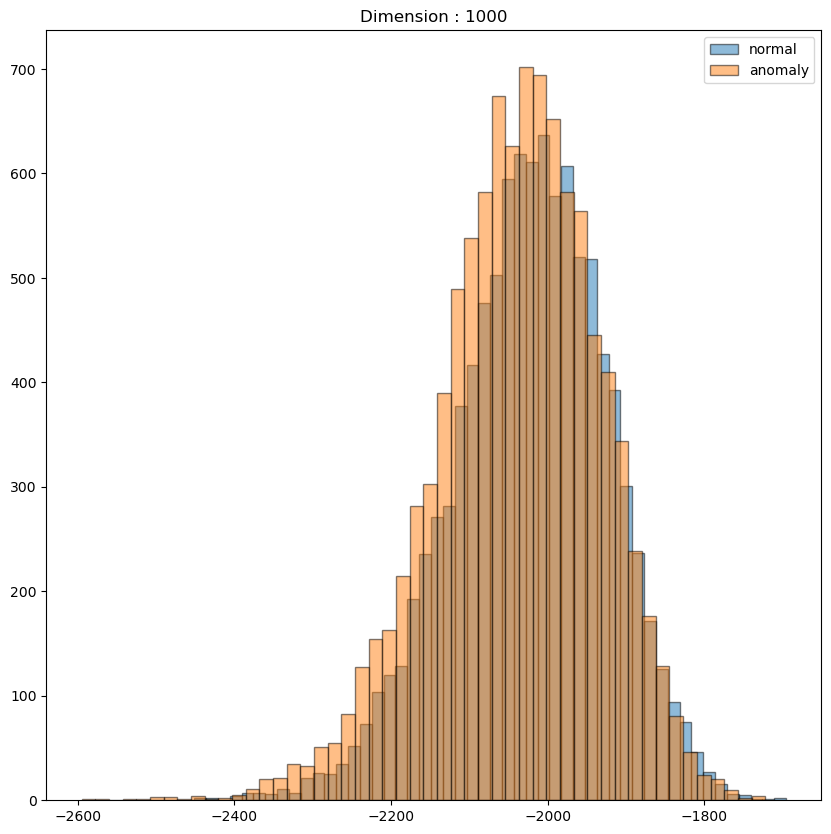}
    \end{subfigure}
    \begin{subfigure}{0.48\textwidth}
        \centering
        \includegraphics[width=\linewidth]{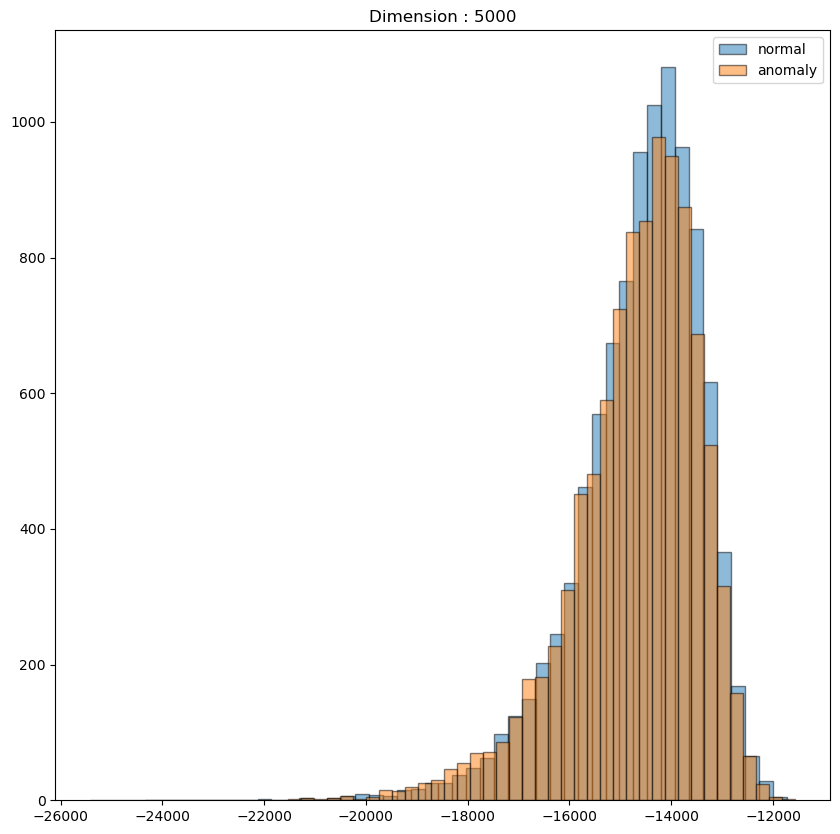}
    \end{subfigure}
    \caption{Histogram of log-likelihood values for the 3rd subfigure in Figure ~\ref{fig:dimsionality_test}}
    \label{fig:dim_histogram_3}
    \end{center}
\end{figure}

\begin{figure}[h!]

    \begin{center}
    \begin{subfigure}{0.48\textwidth}
        \centering
        \includegraphics[width=\linewidth]{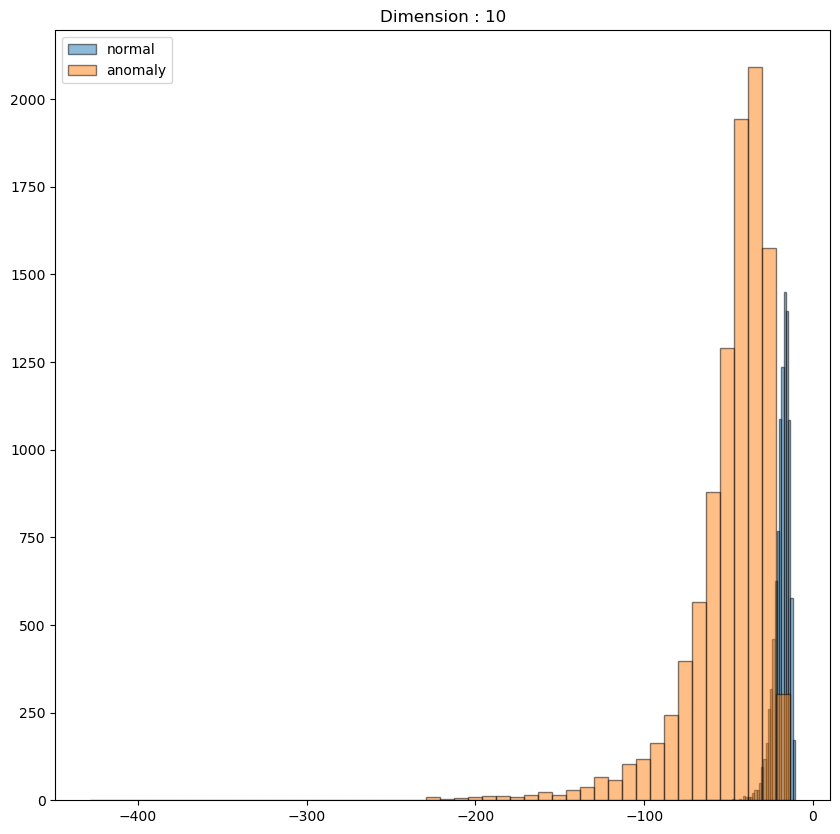}
    \end{subfigure}
    \begin{subfigure}{0.48\textwidth}
        \centering
        \includegraphics[width=\linewidth]{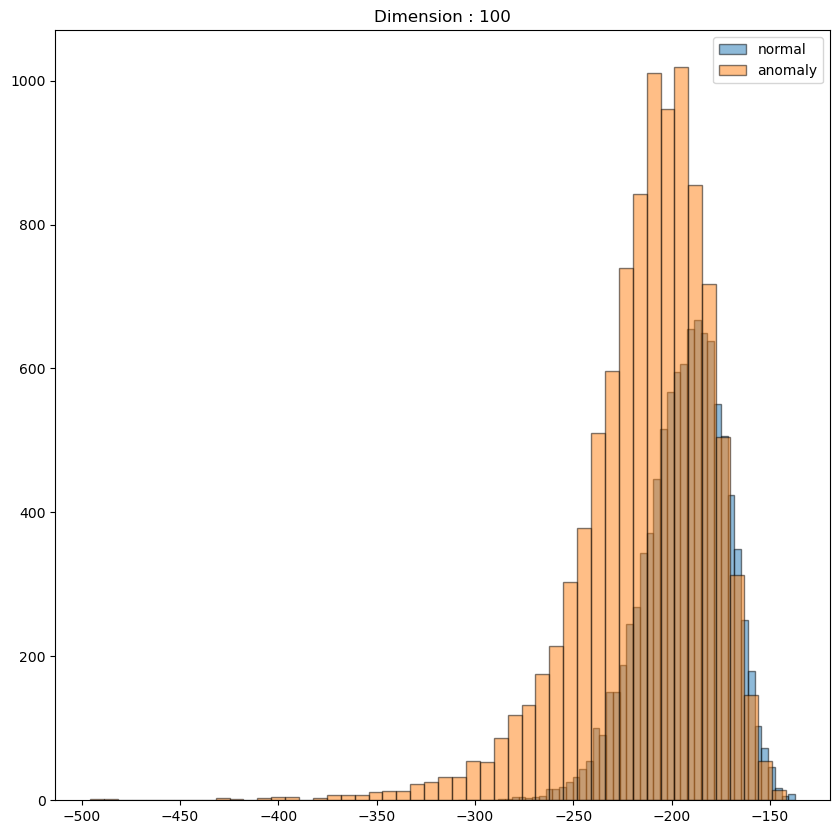}
    \end{subfigure}
    \begin{subfigure}{0.48\textwidth}
        \centering
        \includegraphics[width=\linewidth]{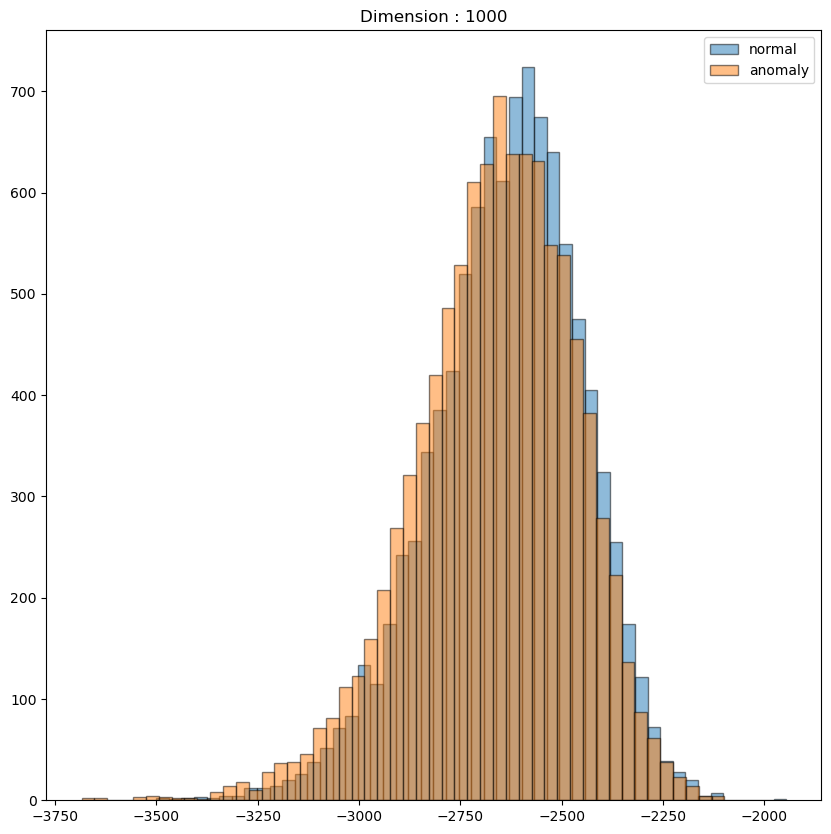}
    \end{subfigure}
    \begin{subfigure}{0.48\textwidth}
        \centering
        \includegraphics[width=\linewidth]{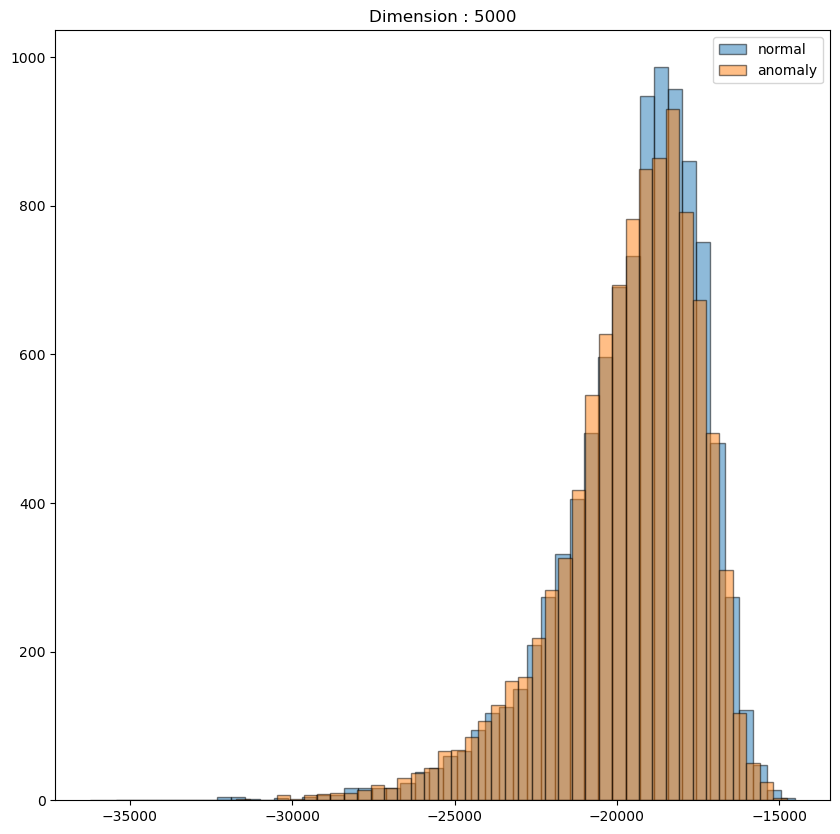}
    \end{subfigure}
    \caption{Histogram of log-likelihood values for the 4th subfigure in Figure~\ref{fig:dimsionality_test}}
    \label{fig:dim_histogram_4}
    \end{center}
\end{figure}

\begin{figure}[h!]

    \begin{center}
    \begin{subfigure}{0.48\textwidth}
        \centering
        \includegraphics[width=\linewidth]{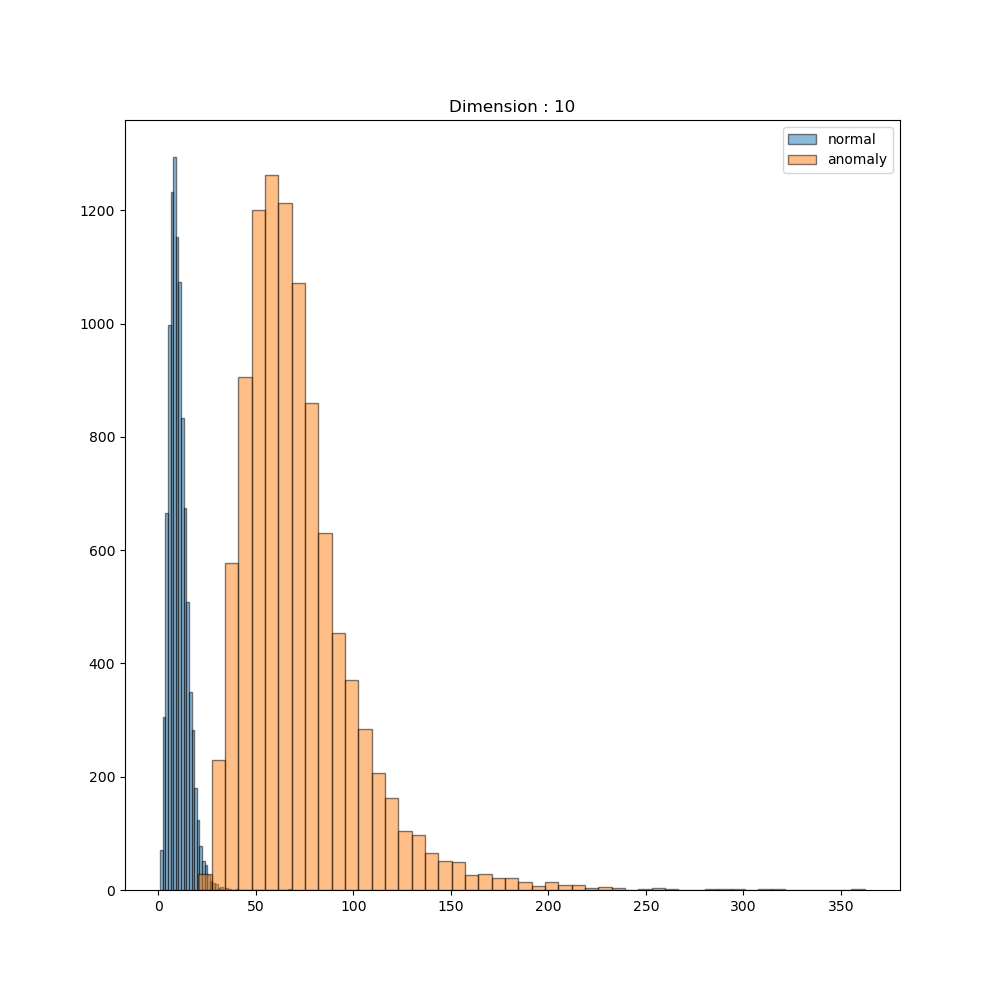}
    \end{subfigure}
    \begin{subfigure}{0.48\textwidth}
        \centering
        \includegraphics[width=\linewidth]{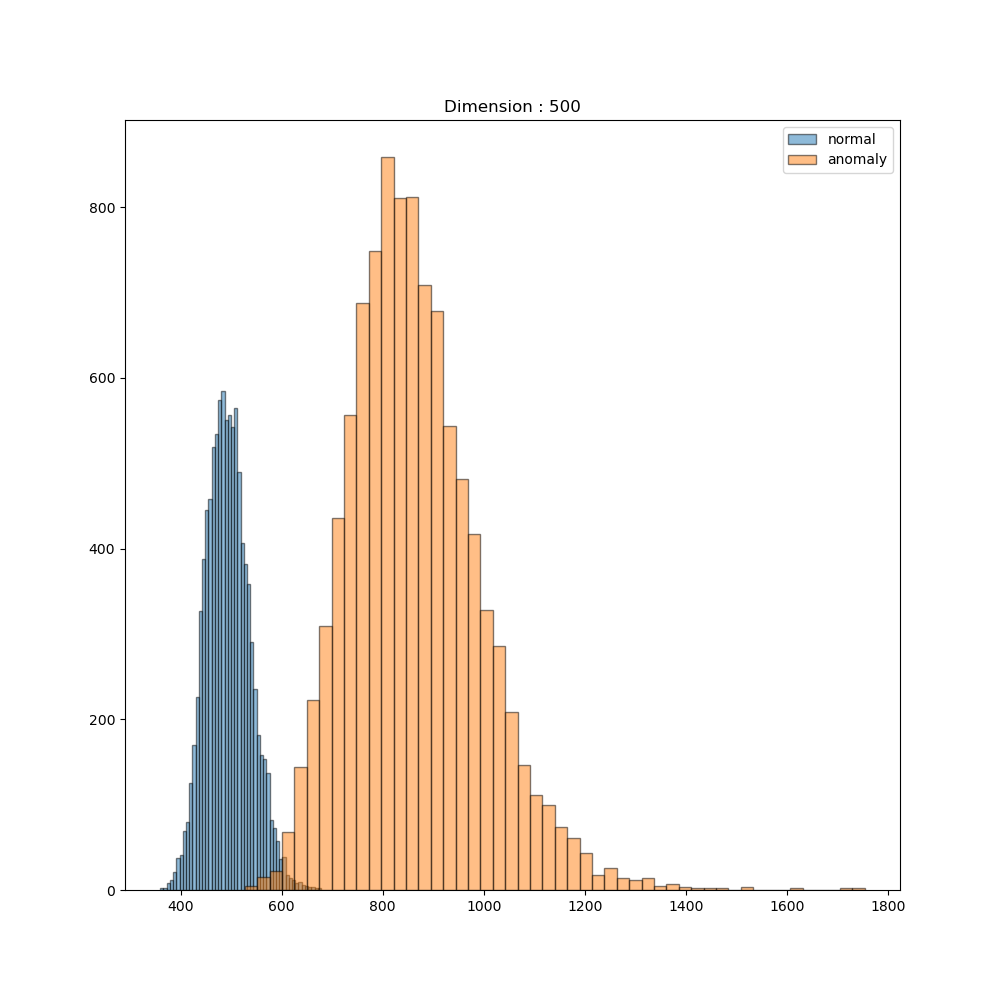}
    \end{subfigure}
    \begin{subfigure}{0.48\textwidth}
        \centering
        \includegraphics[width=\linewidth]{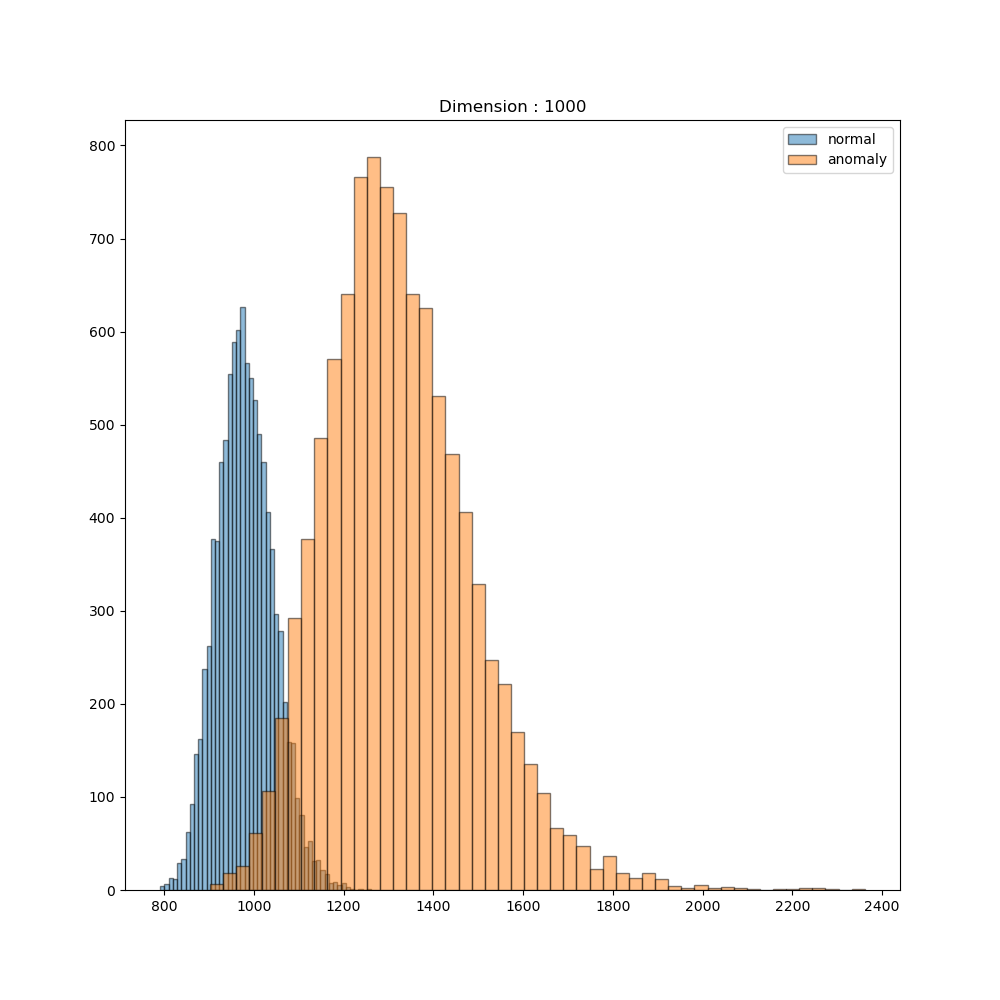}
    \end{subfigure}
    \begin{subfigure}{0.48\textwidth}
        \centering
        \includegraphics[width=\linewidth]{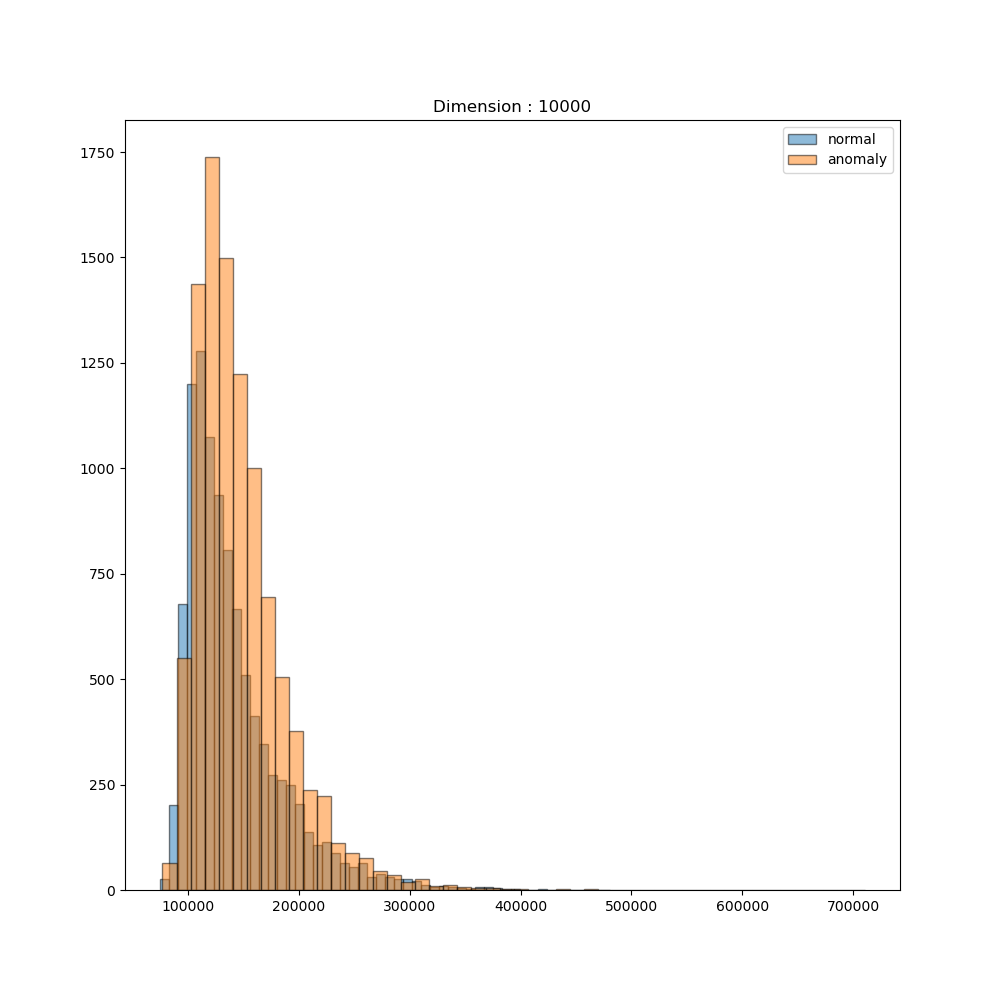}
    \end{subfigure}
    \caption{Histogram of latent norm values for the 1st subfigure in Figure~\ref{fig:dimsionality_test_realnvp}}
    \label{fig:realnvp_normhist_1}
    \end{center}
\end{figure}

\begin{figure}[h!]

    \begin{center}
    \begin{subfigure}{0.48\textwidth}
        \centering
        \includegraphics[width=\linewidth]{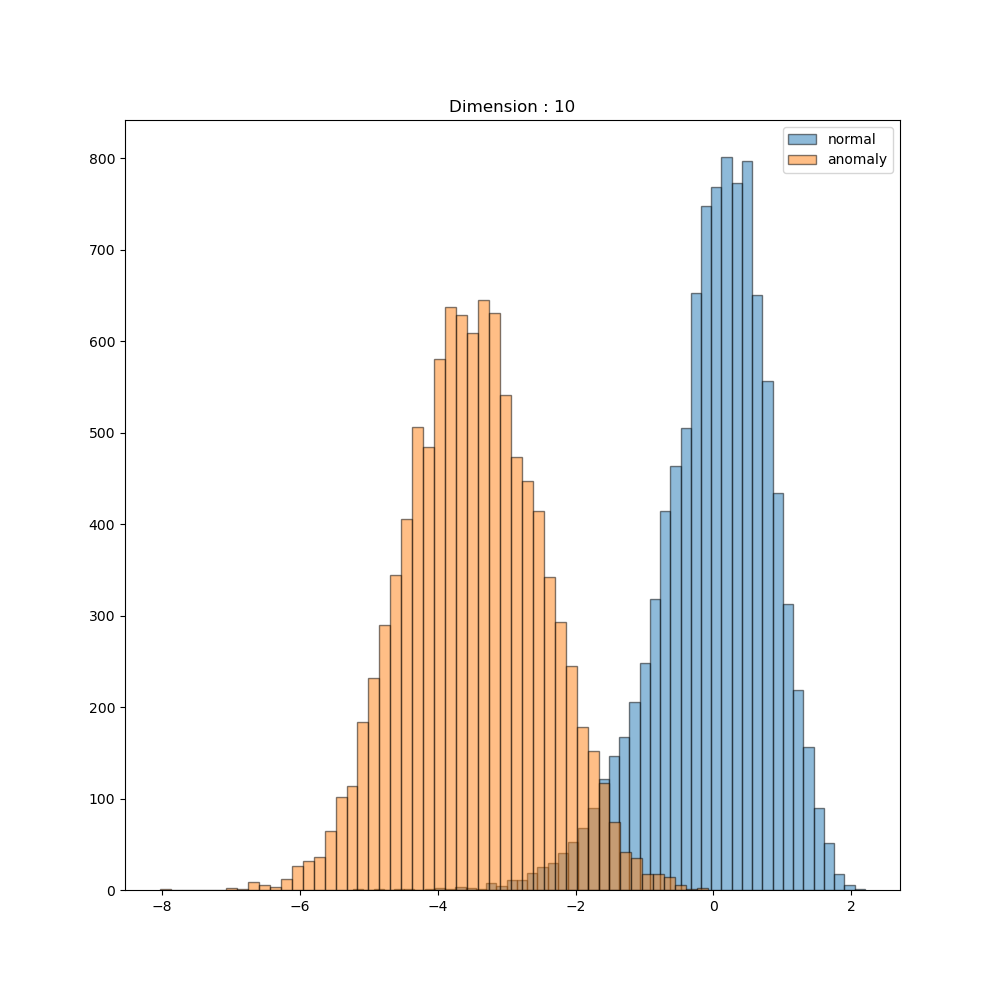}
    \end{subfigure}
    \begin{subfigure}{0.48\textwidth}
        \centering
        \includegraphics[width=\linewidth]{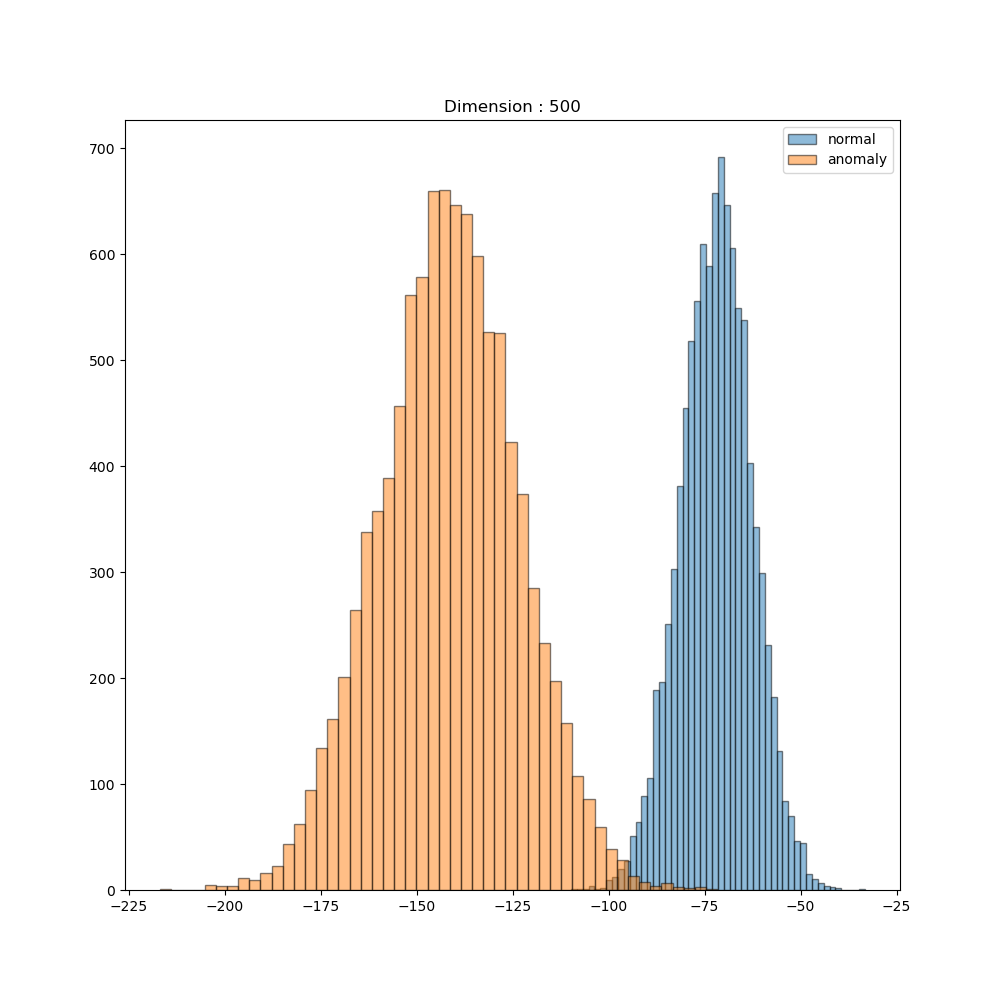}
    \end{subfigure}
    \begin{subfigure}{0.48\textwidth}
        \centering
        \includegraphics[width=\linewidth]{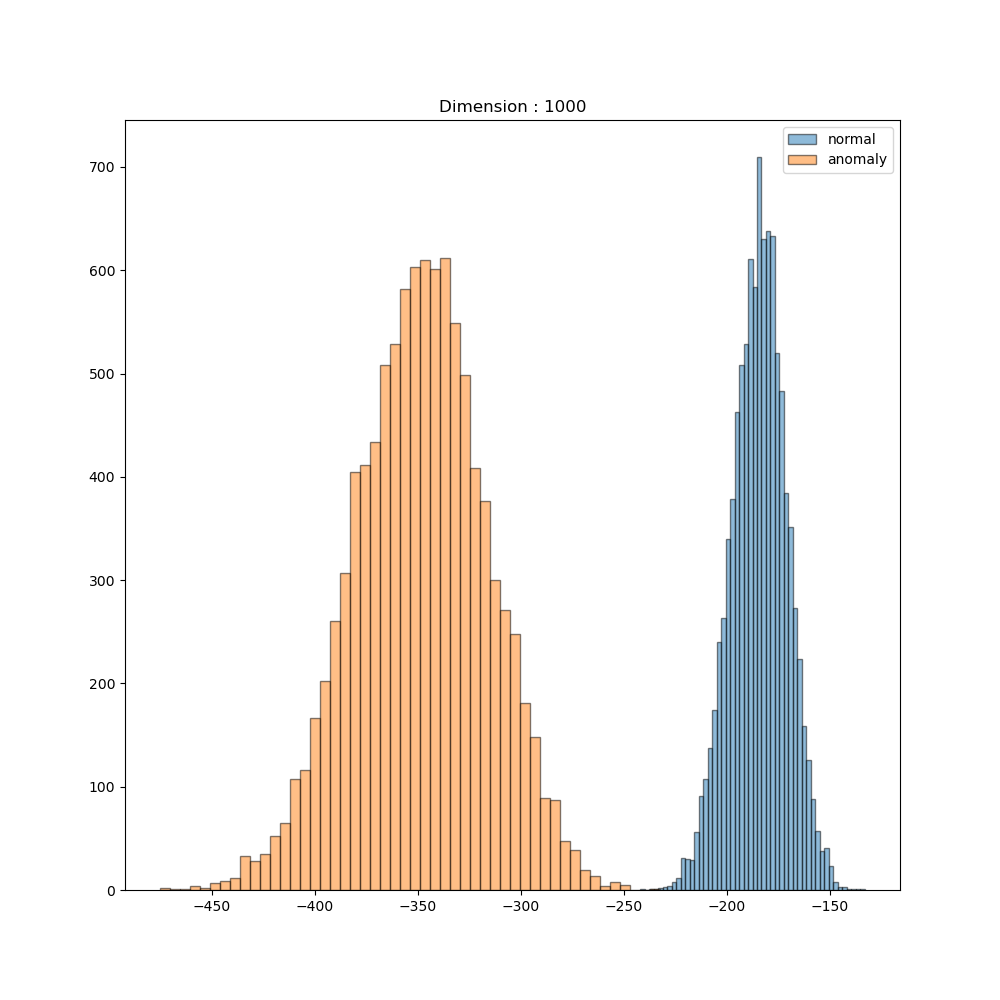}
    \end{subfigure}
    \begin{subfigure}{0.48\textwidth}
        \centering
        \includegraphics[width=\linewidth]{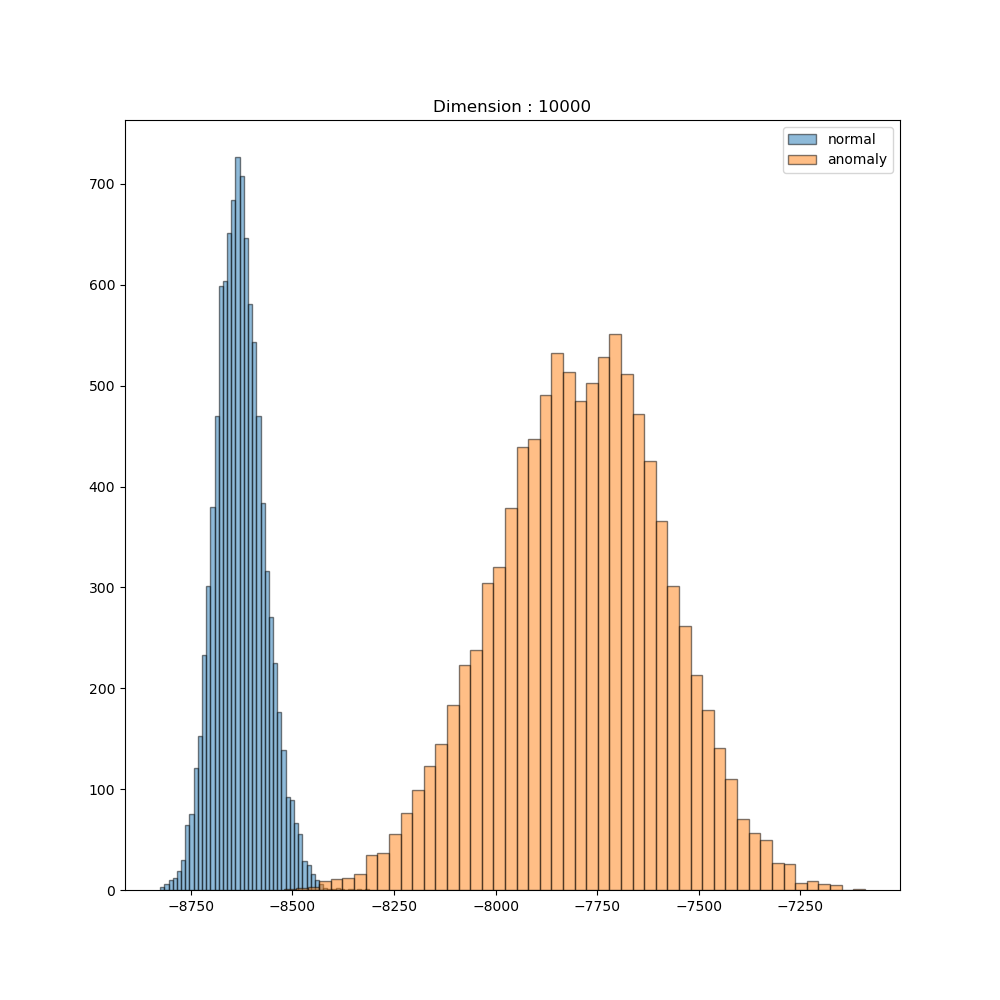}
    \end{subfigure}
    \caption{Histogram of volume values for the 1st subfigure in Figure~\ref{fig:dimsionality_test_realnvp}}
    \label{fig:realnvp_volumehist_1}
    \end{center}
\end{figure}

\begin{figure}[h!]

    \begin{center}
    \begin{subfigure}{0.48\textwidth}
        \centering
        \includegraphics[width=\linewidth]{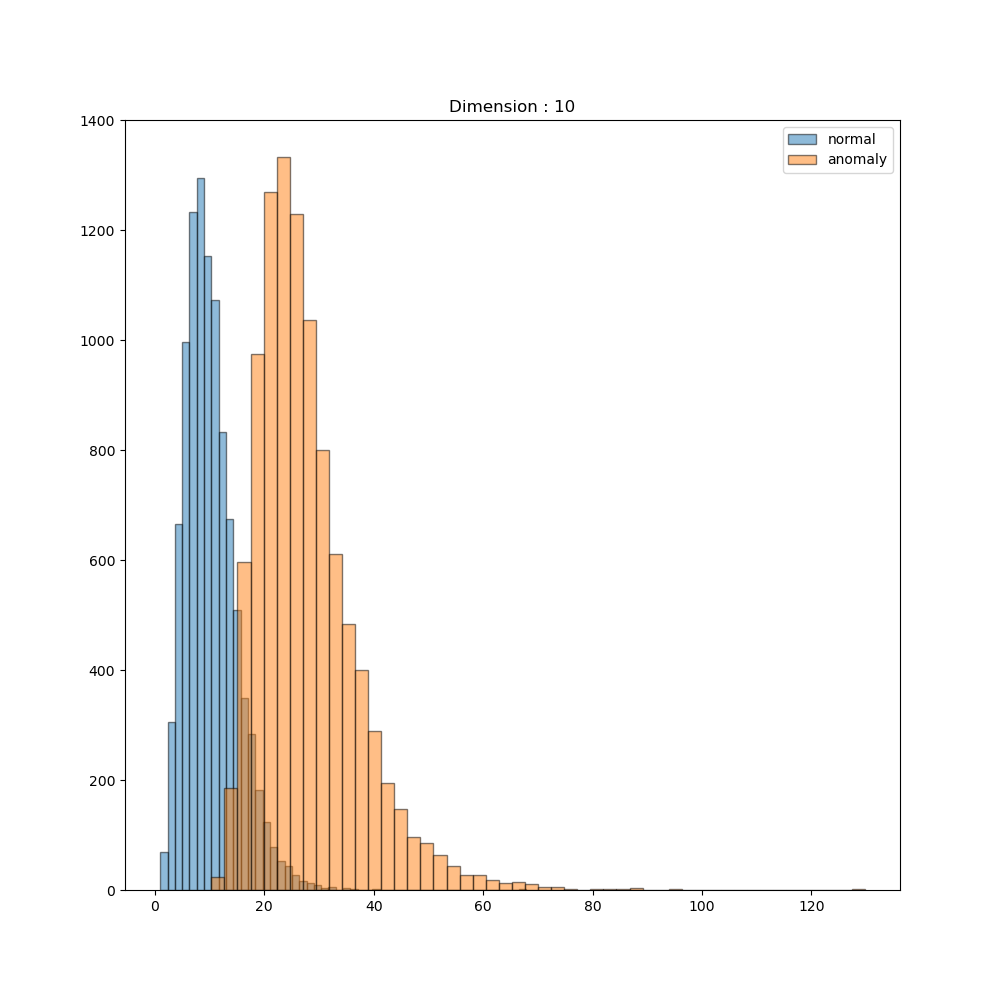}
    \end{subfigure}
    \begin{subfigure}{0.48\textwidth}
        \centering
        \includegraphics[width=\linewidth]{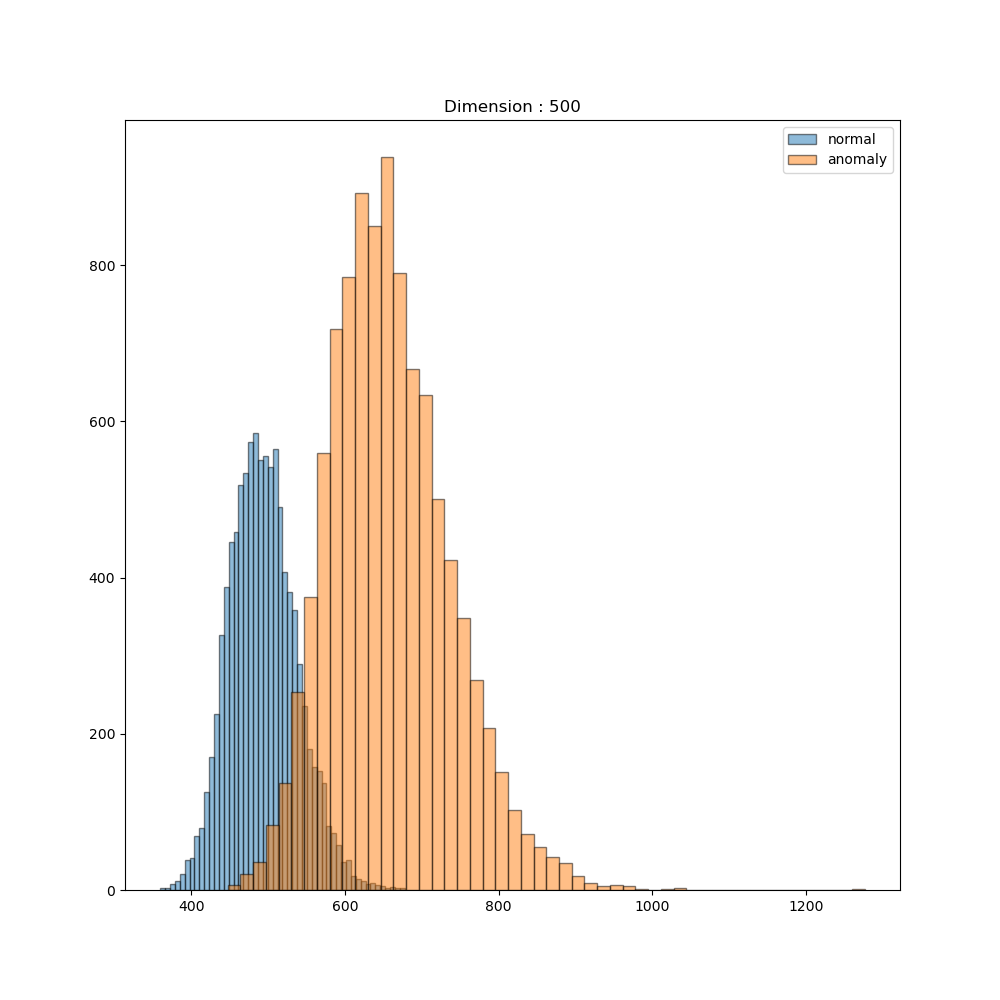}
    \end{subfigure}
    \begin{subfigure}{0.48\textwidth}
        \centering
        \includegraphics[width=\linewidth]{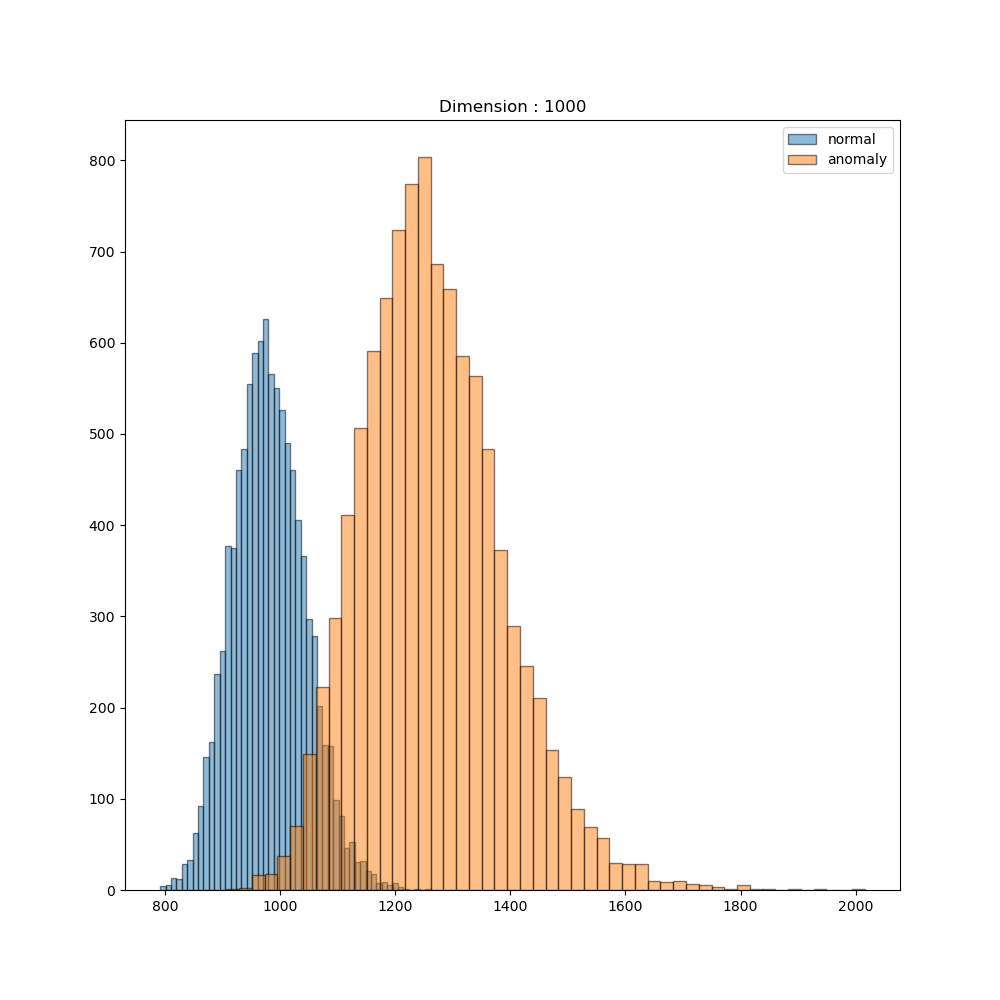}
    \end{subfigure}
    \begin{subfigure}{0.48\textwidth}
        \centering
        \includegraphics[width=\linewidth]{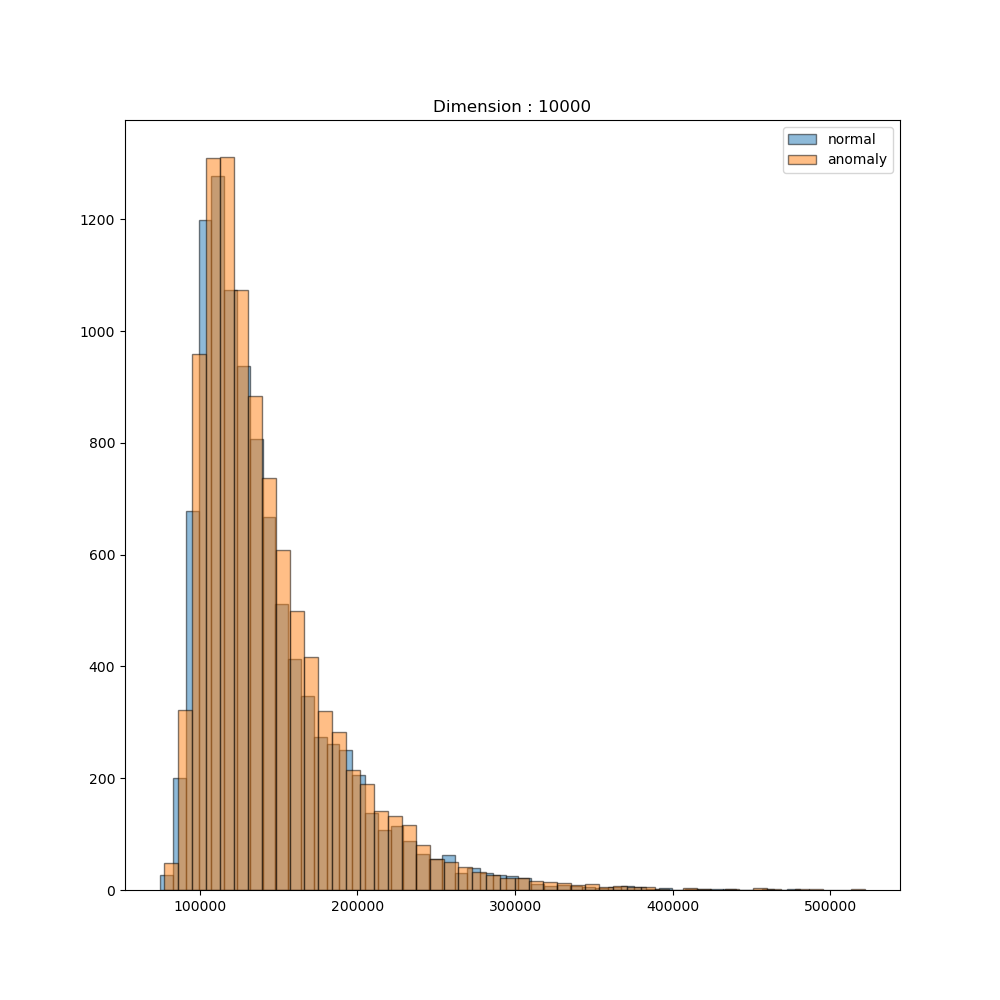}
    \end{subfigure}
    \caption{Histogram of latent norm values for the 3rd subfigure in Figure~\ref{fig:dimsionality_test_realnvp}}
    \label{fig:realnvp_normhist_2}
    \end{center}
\end{figure}

\begin{figure}[h!]

    \begin{center}
    \begin{subfigure}{0.48\textwidth}
        \centering
        \includegraphics[width=\linewidth]{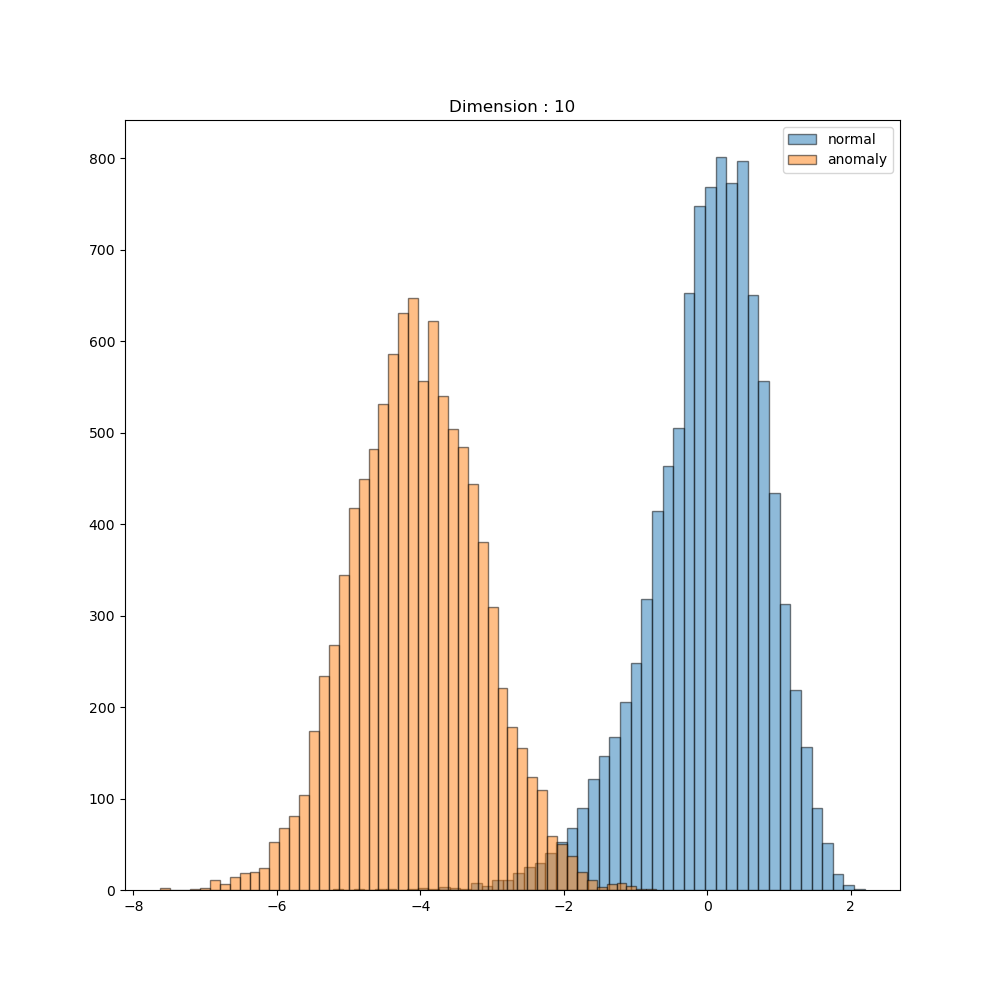}
    \end{subfigure}
    \begin{subfigure}{0.48\textwidth}
        \centering
        \includegraphics[width=\linewidth]{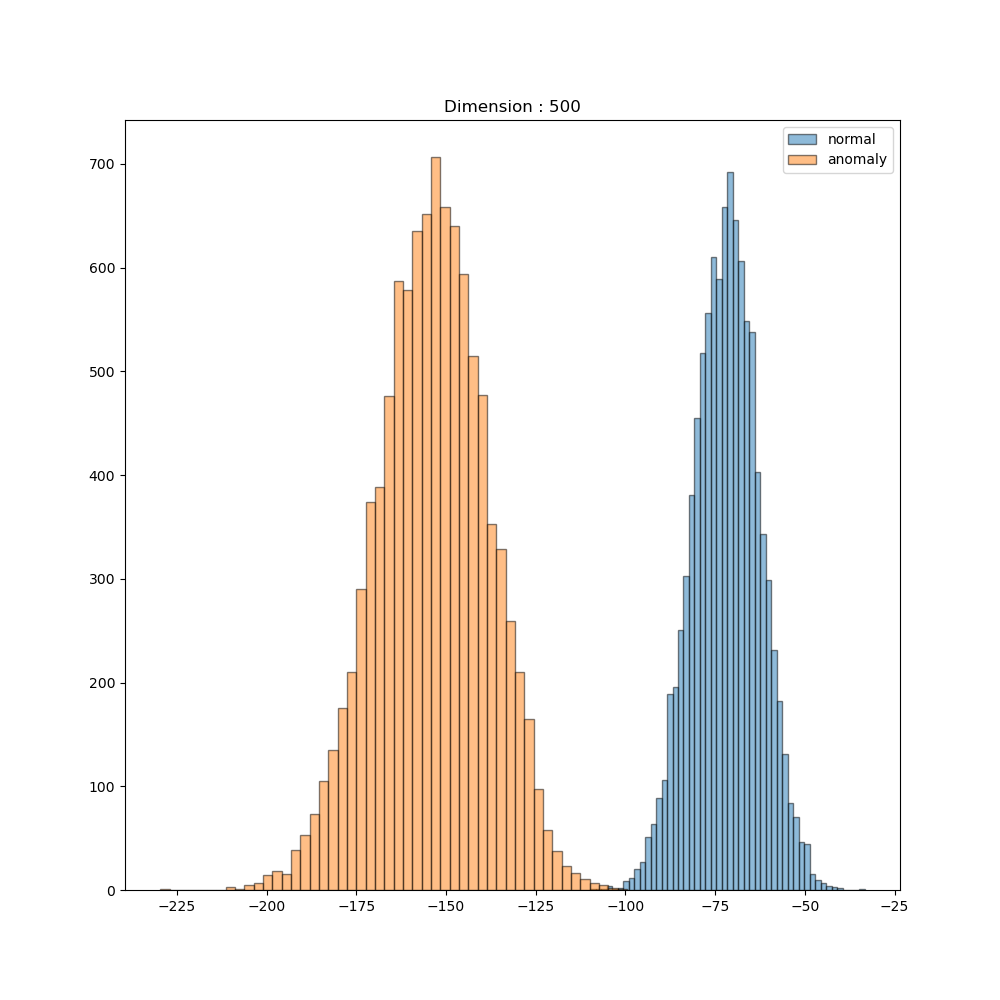}
    \end{subfigure}
    \begin{subfigure}{0.48\textwidth}
        \centering
        \includegraphics[width=\linewidth]{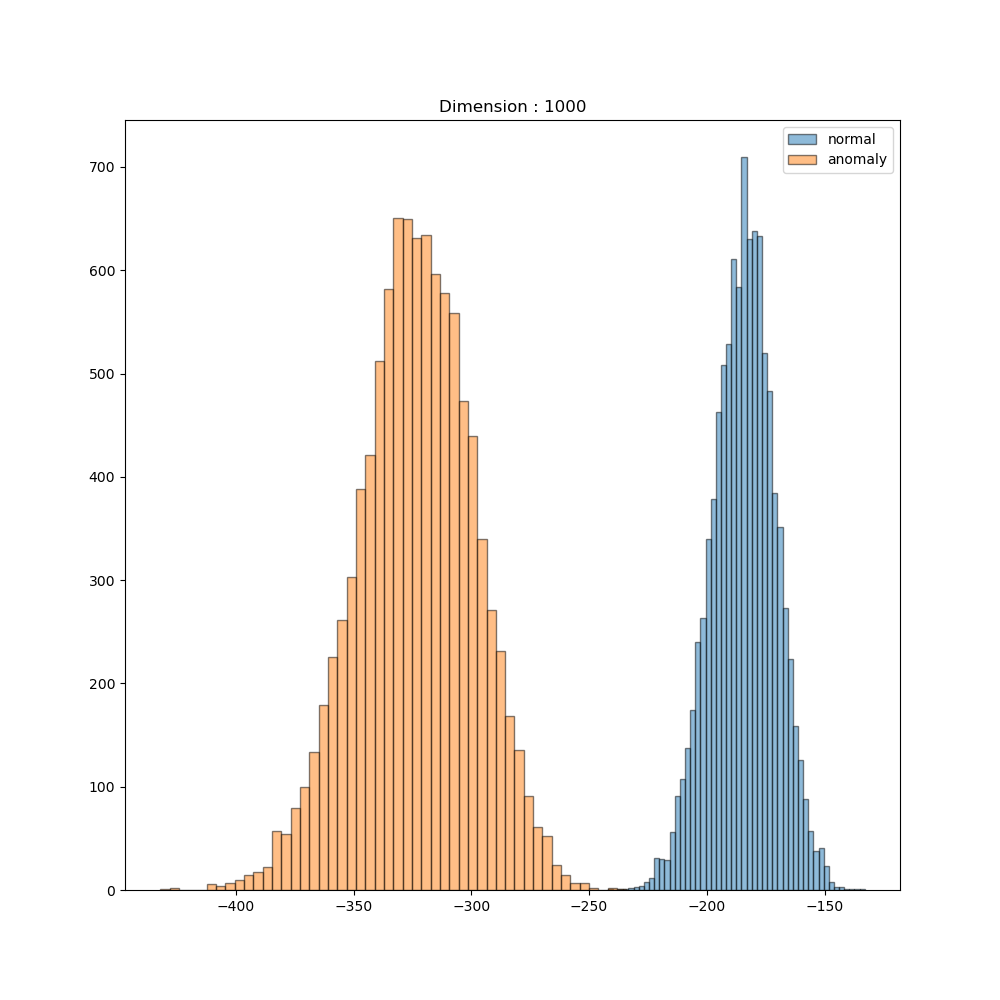}
    \end{subfigure}
    \begin{subfigure}{0.48\textwidth}
        \centering
        \includegraphics[width=\linewidth]{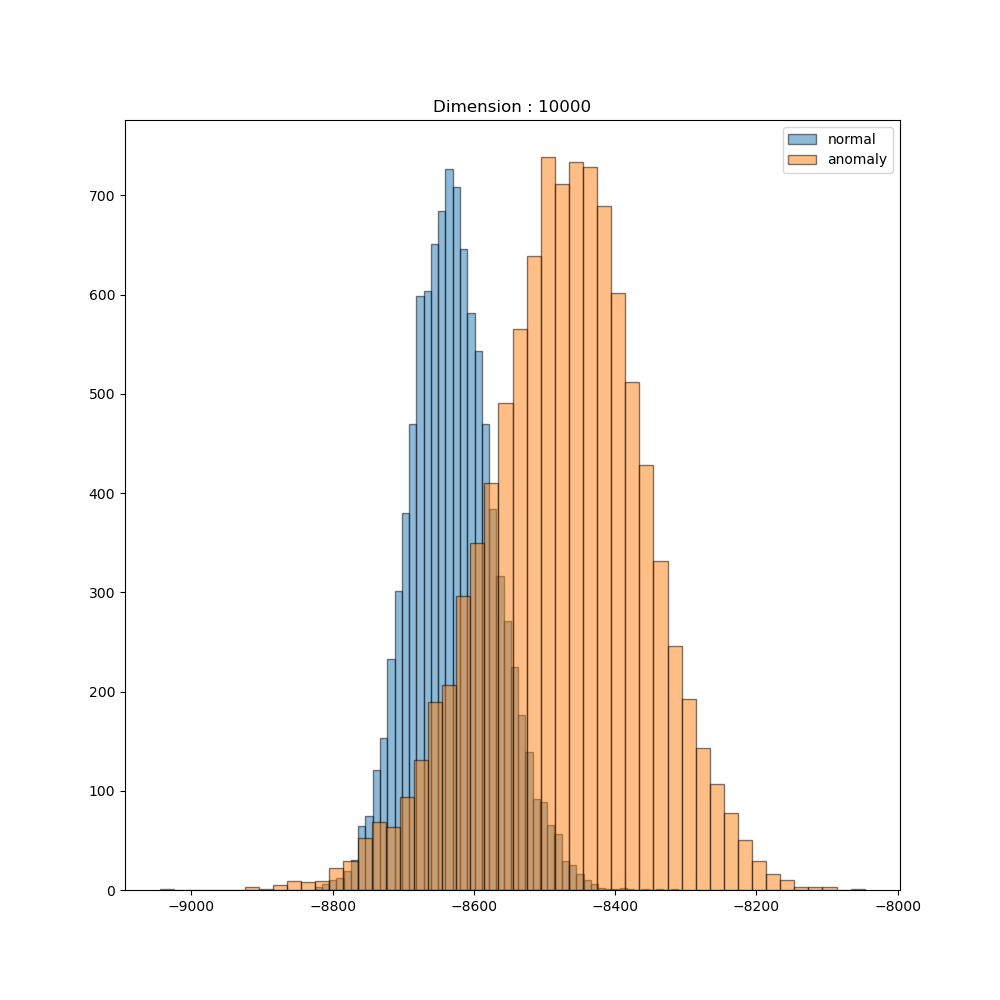}
    \end{subfigure}
    \caption{Histogram of volume values for the 3rd subfigure in Figure~\ref{fig:dimsionality_test_realnvp}}
    \label{fig:realnvp_volumehist_2}
    \end{center}
\end{figure}

\clearpage

\subsection{Experiments of Section \ref{high_dim_perspective} with Real Data}
\label{real_experiment_dimension_appendix}

\begin{table}[!h]
\caption{AUROC scores for likelihood tests as a function of dimensionality (number of retained principal components) using RealNVP with MLP (images preprocessed by PCA). The left block corresponds to $\mathbb{H}(P)>\mathbb{H}(Q)$ and the right block corresponds to $\mathbb{H}(P)<\mathbb{H}(Q)$.}
\centering
\label{tab:pca_auroc}
\small
\setlength{\tabcolsep}{6pt}
\resizebox{0.98\linewidth}{!}{
\begin{tabular}{lcccc @{\hspace{10pt}} lcccc}
\toprule
In-dist ($P$) / Out-dist ($Q$) & 1024 & 512 & 256 & 30 &
In-dist ($P$) / Out-dist ($Q$) & 1024 & 512 & 256 & 30 \\
\midrule
CIFAR-10 / SVHN   & 0.1262 & 0.0950 & 0.1729 & 0.3547 &
SVHN / CIFAR-10   & 0.9932 & 0.9807 & 0.9460 & 0.8670 \\
CIFAR-100 / SVHN  & 0.0668 & 0.0975 & 0.2015 & 0.3410 &
SVHN / CIFAR-100  & 0.9926 & 0.9783 & 0.9436 & 0.8723 \\
CelebA / SVHN     & 0.1598 & 0.2179 & 0.2988 & 0.4728 &
SVHN / CelebA     & 0.9969 & 0.9857 & 0.9709 & 0.9389 \\
\midrule
CIFAR-100 / CIFAR-10 & 0.4804 & 0.5058 & 0.4982 & 0.5021 &
CIFAR-10 / CIFAR-100 & 0.5415 & 0.5287 & 0.5257 & 0.5630 \\
CelebA / CIFAR-10    & 0.8424 & 0.8511 & 0.8344 & 0.6492 &
CIFAR-10 / CelebA    & 0.6350 & 0.5878 & 0.6747 & 0.7326 \\
\bottomrule
\end{tabular}
}
\end{table}

\begin{table}[!h]
\caption{AUROC scores for likelihood tests as a function of dimensionality (number of retained independent components) using RealNVP with MLP (images preprocessed by ICA). The left block corresponds to $\mathbb{H}(P)>\mathbb{H}(Q)$ and the right block corresponds to $\mathbb{H}(P)<\mathbb{H}(Q)$.}
\centering
\label{tab:ica_auroc_appendix}
\small
\setlength{\tabcolsep}{6pt}
\resizebox{0.98\linewidth}{!}{
\begin{tabular}{lcccc @{\hspace{10pt}} lcccc}
\toprule
In-dist ($P$) / Out-dist ($Q$) & 1024 & 512 & 256 & 30 &
In-dist ($P$) / Out-dist ($Q$) & 1024 & 512 & 256 & 30 \\
\midrule
CIFAR-10 / SVHN   & 0.2311 & 0.2924 & 0.2984 & 0.3143 &
SVHN / CIFAR-10   & 0.9917 & 0.9843 & 0.9486 & 0.8520 \\
CIFAR-100 / SVHN  & 0.0843 & 0.1160 & 0.2036 & 0.3490 &
SVHN / CIFAR-100  & 0.9933 & 0.9536 & 0.9137 & 0.8622 \\
CelebA / SVHN     & 0.1207 & 0.1782 & 0.2745 & 0.4711 &
SVHN / CelebA     & 0.9976 & 0.9811 & 0.9722 & 0.9481 \\
\midrule
CIFAR-100 / CIFAR-10 & 0.5133 & 0.5169 & 0.5046 & 0.5037 &
CIFAR-10 / CIFAR-100 & 0.4970 & 0.5081 & 0.5146 & 0.5588 \\
CelebA / CIFAR-10    & 0.8511 & 0.8501 & 0.8364 & 0.6515 &
CIFAR-10 / CelebA    & 0.3988 & 0.4041 & 0.4692 & 0.6725 \\
\bottomrule
\end{tabular}
}
\end{table}

First, we will explain the experimental settings in Tables \ref{tab:pca_auroc} and \ref{tab:ica_auroc_appendix} for experiment on RealNVP. In both experiments, the model used RealNVP with 16 coupling layers, and each coupling layer consisted of 2 layers, and the hidden dimension of the corresponding layer was set to 256. AdamW was used as the optimizer, and the learning rate was set to 5e-4 and the weight decay was set to 1e-4. The batch size was 256, the scheduler was CosineAnnealingWarmRestarts, and the restart period was set as the entire epoch. The epoch was set to 300 for all cases. The standard train/test splits provided by torchvision \citep{marcel2010torchvision} were used, and the implementation of the model is referenced from \citet{stimper2023normflows}.

Although PCA does not satisfy the independence assumption of Theorem \ref{theorem:gap_expectation_likelihood_indep}, as it does not preserve statistical independence across components, it remains a widely adopted dimensionality reduction method. Therefore, we evaluated the impact of dimensionality adjustment using PCA on AUROC, as presented in Table \ref{tab:pca_auroc}. Additionally, Table \ref{tab:ica_auroc_appendix} presents the corresponding results using ICA-based dimensionality reduction. 

In the results, AUROC tends to increase as the dimension decreases for the top three cases in Tables \ref{tab:pca_auroc} and \ref{tab:ica_auroc_appendix}, which satisfy the $\mathbb{H}(P) > \mathbb{H}(Q)$ condition, and tends to decrease for the corresponding reverse cases. However, the two cases below do not follow the trend of the top three cases, which can be interpreted as a consequence of the small difference in entropy between the two distributions. This difference is small because the complexity histograms of the two datasets in Figure 2 of \citet{Serrà2020Input} mostly overlap.

We also adjusted the dimension of the image using the bilinear interpolation resize method provided by torchvision for the raw image, and performed a likelihood test after obtaining the likelihood of the image through Glow which  consists of a CNN, the results are included in Table \ref{tab:image_auroc_glow}. In this experiment, 16 Glow layer blocks were used, and a model consisting of 3 multi-scale layers was used. Other experimental settings are the same as Table \ref{tab:ica_auroc_appendix}. Since this experiment uses raw images, independence between pixels is not guaranteed, so the theorem presented in Appendix \ref{proof_theorem} cannot be applied. However, this experiment was conducted to check the performance trend according to dimension in a situation where there is no independence assumption.

\begin{table}[!h]
\caption{AUROC scores for likelihood tests as a function of image size using Glow (images resized by bilinear interpolation). The left block corresponds to $\mathbb{H}(P) > \mathbb{H}(Q)$, and the right block corresponds to $\mathbb{H}(P) < \mathbb{H}(Q)$.}
\vspace{0.15cm}
\centering
\label{tab:image_auroc_glow}
\small
\setlength{\tabcolsep}{7pt}
\resizebox{0.98\linewidth}{!}{
\begin{tabular}{lccc @{\hspace{10pt}} lccc}
\toprule
In-dist ($P$) / Out-dist ($Q$) & $32{\times}32$ & $16{\times}16$ & $8{\times}8$ &
In-dist ($P$) / Out-dist ($Q$) & $32{\times}32$ & $16{\times}16$ & $8{\times}8$ \\
\midrule
CIFAR-10 / SVHN   & 0.0716 & 0.3586 & 0.4512 &
SVHN / CIFAR-10   & 0.9902 & 0.9777 & 0.9195 \\
CIFAR-100 / SVHN  & 0.0846 & 0.4448 & 0.3918 &
SVHN / CIFAR-100  & 0.9900 & 0.9798 & 0.9481 \\
CelebA / SVHN     & 0.1541 & 0.3056 & 0.7037 &
SVHN / CelebA     & 0.9850 & 0.9968 & 0.9982 \\
\midrule
CIFAR-100 / CIFAR-10 & 0.4857 & 0.4933 & 0.5016 &
CIFAR-10 / CIFAR-100 & 0.5259 & 0.5446 & 0.5567 \\
CelebA / CIFAR-10    & 0.7481 & 0.7137 & 0.7557 &
CIFAR-10 / CelebA    & 0.5087 & 0.6181 & 0.6751 \\
\bottomrule
\end{tabular}
}
\end{table}

Surprisingly, we can see that there are cases where the AUROC exceeds 0.5 when reducing the size of the image pairs in Table \ref{tab:image_auroc_glow} (see CelebA vs SVHN case). In addition, when SVHN is set to in-distribution and CelebA is set to out-of-distribution, the performance tends to increase as the dimension decreases, which is a result that is not explained by the theorems in Appendix \ref{proof_theorem}. We argue that this is because resizing an image via bilinear interpolation strengthens the correlation between neighboring pixels, which lowers the entropy of images whose textures are complex (i.e., high entropy). Therefore, although it is difficult to confirm the effect of dimension on AUROC through the methodology, it can be seen that not only can the performance be improved by the simple image resize methodology for cases where likelihood inversion occurs, but also it is possible to increase the AUROC to more than 0.5.

Furthermore, we performed dimension reduction using PCA on the InternetAds dataset, which has the largest dimension ($d=1555$) in ADBench, and checked the trend of performance changes. The results are included in Table \ref{tab:tabular_real_dim_test}.

\begin{table}[!t]
\caption{AUROC on the InternetAds dataset for NF-SLT and MCM as a function of the retained dimensionality. The ratio denotes the fraction of principal components retained out of 500 PCs (after PCA reduction to 500 dimensions).}
\vspace{0.15cm}
\label{tab:tabular_real_dim_test}
\centering
\small
\setlength{\tabcolsep}{8pt}
\resizebox{0.8\linewidth}{!}{
\begin{tabular}{lcccccc}
\toprule
Ratio  & 1.00 & 0.95 & 0.90 & 0.85 & 0.80 & 0.75 \\
\midrule
NF-SLT & 0.8180 & 0.8561 & 0.8579 & 0.8377 & 0.8246 & 0.8240 \\
MCM    & 0.7242 & 0.7217 & 0.7189 & 0.7161 & 0.7158 & 0.6978 \\
\bottomrule
\end{tabular}
}
\end{table}

Through the results in Table~\ref{tab:tabular_real_dim_test}, we can see that MCM's performance decreases due to information loss as the dimension decreases. However, NF-SLT's performance increases when the dimension is reduced, and even when the ratio is reduced to 0.75, the performance is higher than when the entire set of 500 PCs is used. Note that the ratio 1.00 corresponds to the 500 retained principal components rather than the original 1555 features, for which NF-SLT attains an AUROC of 0.8746 (Table~\ref{tab:auroc_whole_deep}). Hence this trend should be read as the effect of further reducing the dimension within the PCA-reduced representation, since the comparison with the original features is confounded by the information loss of PCA itself. This tendency is in line with Theorem~\ref{theorem:gap_expectation_likelihood_indep} in that reducing the dimension improves the performance of NF-SLT within this representation. However, the improvement is not monotone, which is plausible because the entropy relationship between the normal and anomalous distributions is unknown for real data and PCA also discards information. Therefore this result is consistent with our claim that the low dimensionality of typical tabular data contributes to the strong performance of NF-SLT, though it is not conclusive evidence that high dimensionality degrades the likelihood test on tabular data. 

In practice, extremely high-dimensional tabular datasets, such as genomic or metagenomic data, often contain tens of thousands of features. However, these datasets present specific challenges for anomaly detection (AD) due to their inherent feature homogeneity and extreme dimensional complexity. Unlike general tabular data, which typically includes heterogeneous features (e.g., categorical, continuous, nominal), genomic data is characterized by a large number of biologically similar variables with high mutual correlation and tightly constrained value distributions. This fundamental difference means that directly applying standard AD methods to raw genomic or metagenomic data is generally impractical without substantial preprocessing. For instance, metagenomic datasets often include tens of thousands of microbial species abundances or gene expression profiles, where each feature can be highly correlated with others due to shared biological pathways or taxonomic hierarchies. This extreme feature homogeneity can lead to severe overfitting and degraded AD performance if handled without prior dimensionality reduction. Therefore, standard bioinformatics workflows commonly incorporate domain-specific dimensionality reduction techniques, such as sparse Canonical Correlation Analysis (CCA) \citep{witten2009penalized}, Gene Set Variation Analysis (GSVA) \citep{hanzelmann2013gsva}, or biologically-informed network modeling \citep{wang2024biologically}, to project these high-dimensional inputs into lower-dimensional, interpretable representations. For example, \citet{hanzelmann2013gsva} demonstrated the use of GSVA to condense over 25,000 genes into biologically meaningful gene sets, significantly reducing the feature space and improving downstream analysis.

Additionally, a recent study by \citet{wang2024biologically} on BioNet highlighted that even deep learning models for high-dimensional biological data often rely on domain-specific preprocessing to improve interpretability and reduce overfitting. This approach integrates biological knowledge to guide feature selection, effectively reducing the dimensional complexity of the input data. Such preprocessing steps are critical for achieving meaningful predictions and robust model performance, as demonstrated in the context of glioblastoma heterogeneity assessment \citep{wang2024biologically}. Given these domain-specific preprocessing requirements, our study focuses on heterogeneous tabular data that include diverse feature types such as categorical, continuous, and nominal variables, rather than directly addressing raw, extremely high-dimensional genomic or metagenomic datasets. This choice aligns with the practical constraints observed in bioinformatics, where specialized preprocessing is essential for effective analysis of such data. Lastly, if there are heterogeneous high-dimensional tabular datasets that do not require such specialized preprocessing, we would be open to exploring them as part of future work, as our current study aims to provide a broadly applicable approach for anomaly detection across diverse real-world datasets.

\section{Proof of Theorems}
\label{proof_theorem}
First, we prove that when the entropy of a normal distribution $P$ is larger than the entropy of an abnormal distribution $Q$, $P$ is a $d$-dimensional i.i.d. distribution, and $Q$ has identical one-dimensional marginals, the expectation gap of likelihood estimated by the density estimation model $P_{\theta}$ trained on data sampled from the normal distribution $P$ decreases as $d$ increases.

\renewcommand{\thetheorem}{5.\arabic{theorem}}
\setcounter{theorem}{0}
\begin{lemma}
\label{lemma:kldivergence_approx}
Let $P, P_{\theta}, Q$ be probability distributions on $\mathbb{R}^d$ with density functions $p, p_{\theta}, q$, where $\theta \in \mathbb{R}^{d_\theta}$ denotes the model parameter. Suppose $D_{KL}(Q||P) < \infty$, and $Q$ and $P_{\theta}$ share the same support as $P$ for $\theta$ in a neighborhood of $\theta_0$. If $p_{\theta}(x)$ converges to $p(x)$ pointwise as $\theta \rightarrow \theta_0$ and $\log \frac{p(x)}{p_{\theta}(x)}$ is uniformly integrable with respect to both $P$ and $Q$ for $\theta$ in a neighborhood of $\theta_0$, then:
\begin{equation}
\lim_{\theta \rightarrow \theta_0} D_{KL}(Q||P_{\theta}) = D_{KL}(Q||P), \qquad \lim_{\theta \rightarrow \theta_0} D_{KL}(P||P_{\theta}) = 0.
\end{equation}
\end{lemma}
\begin{proof}
We first prove $\lim_{\theta \rightarrow \theta_0} D_{KL}(Q||P_{\theta}) = D_{KL}(Q||P)$. Decompose $D_{KL}(Q||P_{\theta})$ as follows:
\begin{align}
D_{KL}(Q||P_{\theta}) = D_{KL}(Q||P) + \int q(x) \log\frac{p(x)}{p_{\theta}(x)}dx.
\end{align}
By the uniform integrability of $\log\frac{p(x)}{p_\theta(x)}$ with respect to $Q$, the decomposition is valid. It suffices to show that $\int q(x) \log\frac{p(x)}{p_{\theta}(x)}dx \rightarrow 0$ as $\theta \rightarrow \theta_0$. Since $p_{\theta}(x)$ converges to $p(x)$ pointwise as $\theta \rightarrow \theta_0$, we have $\log\frac{p(x)}{p_{\theta}(x)} \rightarrow 0$ pointwise. Since $\theta \in \mathbb{R}^{d_\theta}$ and $\log\frac{p(x)}{p_{\theta}(x)}$ is uniformly integrable with respect to $Q$, the Vitali Convergence Theorem yields:
\begin{equation}
\lim_{\theta \rightarrow \theta_0} \int q(x) \log\frac{p(x)}{p_{\theta}(x)}dx = \int q(x)\lim_{\theta \rightarrow \theta_0}\log\frac{p(x)}{p_{\theta}(x)}dx = 0.
\end{equation}
Therefore, $\lim_{\theta \rightarrow \theta_0} D_{KL}(Q||P_{\theta}) = D_{KL}(Q||P)$. The second result follows by the same argument with $Q$ replaced by $P$, using the uniform integrability of $\log\frac{p(x)}{p_\theta(x)}$ with respect to $P$ and the shared support assumption, giving $\lim_{\theta \rightarrow \theta_0} D_{KL}(P||P_{\theta}) = D_{KL}(P||P) = 0$.
\end{proof}

\begin{theorem}[Impact of Dimensionality on Likelihood Gap: i.i.d Case]
\label{theorem:gap_expectation_likelihood_iid}
Let $P=\prod^d_{i=1}p(x_i)$ be an i.i.d. $d$-dimensional continuous probability density model in $\mathbb{R}^d$. Let $Q$ be a probability distribution on $\mathbb{R}^d$ with a density whose one-dimensional marginals are all equal to $q$. Assume the same conditions as Lemma~\ref{lemma:kldivergence_approx} and $\mathbb{E}_{x\sim q}[|\log p(x)|]<\infty$. Let $P_{\theta}$ be a well-trained density estimation model that approximates $P$ (i.e., $p_{\theta}(x) \rightarrow p(x)$ pointwise as $\theta \rightarrow \theta_0$). If $\mathbb{H}(p)-\mathbb{H}(q) > D_{KL}(q||p)$, then in the limit as $\theta \rightarrow \theta_0$, the gap between the expectations of the likelihood for $P$ and $Q$ decreases linearly with respect to $d$.
\end{theorem}

\begin{proof}
Decompose the likelihood gap as:
\begin{align}
\mathbb{E}_{\mathbf{x}\sim P}[\log p_\theta(\mathbf{x})]-\mathbb{E}_{\mathbf{x}\sim Q}[\log p_\theta(\mathbf{x})]=D_{KL}(Q\|P_\theta)-D_{KL}(P\|P_\theta)+\mathbb{H}(Q)-\mathbb{H}(P).
\end{align}
Taking the limit as $\theta\to\theta_0$ and applying Lemma~\ref{lemma:kldivergence_approx} yields $\lim_{\theta\to\theta_0}D_{KL}(Q\|P_\theta)=D_{KL}(Q\|P)$ and $\lim_{\theta\to\theta_0}D_{KL}(P\|P_\theta)=0$. Therefore:
\begin{align}
\lim_{\theta\to\theta_0}\left[\mathbb{E}_{\mathbf{x}\sim P}[\log p_\theta(\mathbf{x})]-\mathbb{E}_{\mathbf{x}\sim Q}[\log p_\theta(\mathbf{x})]\right]=D_{KL}(Q\|P)+\mathbb{H}(Q)-\mathbb{H}(P).
\end{align}
Since $D_{KL}(Q\|P)+\mathbb{H}(Q)=\mathbb{E}_{\mathbf{x}\sim Q}[-\log p(\mathbf{x})]$ and the components of $P$ are i.i.d., the linearity of expectation together with the fact that the $i$-th marginal of $Q$ is $q$ for every $i$ gives:
\begin{align}
D_{KL}(Q\|P)+\mathbb{H}(Q)&=\mathbb{E}_{\mathbf{x}\sim Q}\left[-\sum_{i=1}^{d}\log p(x_i)\right]=\sum_{i=1}^{d}\mathbb{E}_{x_i\sim q}[-\log p(x_i)]\nonumber\\
&=\sum_{i=1}^{d}\left(D_{KL}(q\|p)+\mathbb{E}_{x_i\sim q}[-\log q(x_i)]\right)=d\left(D_{KL}(q\|p)+\mathbb{H}(q)\right).
\end{align}
Hence the dependence among the coordinates of $Q$ does not affect the result. Together with $\mathbb{H}(P)=\sum_{i=1}^{d}\mathbb{E}_{x_i\sim p}[-\log p(x_i)]=d\,\mathbb{H}(p)$, this yields:
\begin{align}
&\lim_{\theta\to\theta_0}\left[\mathbb{E}_{\mathbf{x}\sim P}[\log p_\theta(\mathbf{x})]-\mathbb{E}_{\mathbf{x}\sim Q}[\log p_\theta(\mathbf{x})]\right]\nonumber\\
&=d\left(D_{KL}(q\|p)+\mathbb{H}(q)-\mathbb{H}(p)\right).
\end{align}
If $\mathbb{H}(p)-\mathbb{H}(q)>D_{KL}(q\|p)$, then $D_{KL}(q\|p)-\mathbb{H}(p)+\mathbb{H}(q)<0$, and the gap decreases linearly with respect to $d$.
\end{proof}

Through Theorem~\ref{theorem:gap_expectation_likelihood_iid}, even in the i.i.d. case where the generative model perfectly approximates the data distribution in the limit, the likelihood gap decreases linearly with dimensionality, becoming increasingly severe as $d$ grows.


\begin{proposition}[$\mathbb{H}(P)-\mathbb{H}(Q) > D_{KL}(Q||P)$ for Gaussian Case]
\label{prop:gaussian_gap}
Let P and Q be $\mathcal{N }(\mu_P,\sigma_P^{2}I), \mathcal{N}(\mu_Q,\sigma_Q^{2}I)$ such that $\sigma_P,\sigma_Q$ are non-zero constant. Then, $\mathbb{H}(P)-\mathbb{H}(Q) > D_{KL}(Q||P)$ holds if $||\mu_P-\mu_Q||^2<d(\sigma_P^2-\sigma^2_Q)$.
\end{proposition}

\begin{proof}
Entropy of $P$ and $Q$ can be expressed as follows:
\begin{align}
    \begin{split}
        \mathbb{H}(P) = \frac{d}{2}\log(2\pi e\sigma_P^{2})\\
        \mathbb{H}(Q) = \frac{d}{2}\log(2\pi e\sigma_Q^{2})
    \end{split}
\end{align}

Therefore, we can derive entropy gap of $P$ and $Q$ below:
\begin{align}
    \begin{split}
        \mathbb{H}(P) -\mathbb{H}(Q) = \frac{d}{2}\log\left(\frac{\sigma_P^{2}}{\sigma_Q^{2}}\right)
    \end{split}
\end{align}

Also, since the KL divergence of $Q$ and $P$ is known to be in closed form in the case of Gaussian, it can be expressed as follows:
\begin{align}
    \begin{split}
        &D_{KL}(Q||P) = \frac{1}{2}\left[\sigma_P^{-2}(\mu_P-\mu_Q)^T(\mu_P-\mu_Q) + \log\left(\frac{\det\sigma_P^{2}I}{\det\sigma_Q^{2}I}\right)+\text{tr}(\sigma_P^{-2}\sigma_Q^{2}I)-d\right] \\
        &=\frac{1}{2}\left[\frac{||\mu_P-\mu_Q||^2}{\sigma_P^{2}} + 2d\log\left(\frac{\sigma_P}{\sigma_Q}\right)+\frac{d\sigma_Q^{2}}{\sigma_P^{2}}-d\right]
    \end{split}
\end{align}

Then, we can derive:
\begin{align}
    \begin{split}
        &D_{KL}(Q||P) - \mathbb{H}(P) + \mathbb{H}(Q)\\
        &=\frac{1}{2}\left[\frac{||\mu_P-\mu_Q||^2}{\sigma_P^{2}} + 2d\log\left(\frac{\sigma_P}{\sigma_Q}\right)+\frac{d\sigma_Q^2}{\sigma_P^2}-d\right] -\frac{d}{2}\log\left(\frac{\sigma_P^2}{\sigma_Q^2}\right) \\
        &=\frac{1}{2}\left[\frac{||\mu_P-\mu_Q||^2}{\sigma_P^{2}} + 2d\log\left(\frac{\sigma_P}{\sigma_Q}\right)+\frac{d\sigma_Q^2}{\sigma_P^2}-d -2d\log\left(\frac{\sigma_P}{\sigma_Q}\right)\right]\\
        &=\frac{1}{2}\left[\frac{||\mu_P-\mu_Q||^2}{\sigma_P^{2}} +\frac{d\sigma_Q^2}{\sigma_P^2}-d \right]\\
        &=\frac{1}{2}\left[\frac{||\mu_P-\mu_Q||^2}{\sigma_P^{2}} +d\left(\frac{\sigma^2_Q}{\sigma^2_P}-1\right)\right]
    \end{split}
\end{align}
Therefore, if $||\mu_P-\mu_Q||^2<d(\sigma_P^2-\sigma^2_Q)$, $D_{KL}(Q||P) - \mathbb{H}(P) + \mathbb{H}(Q) < 0$
\end{proof}
Proposition~\ref{prop:gaussian_gap} states that if the variance of $P$ is sufficiently larger than the variance of $Q$ so that $d$ times the difference between the variances of $P$ and $Q$ exceeds the squared norm of the mean difference, i.e., $\|\mu_P - \mu_Q\|^2 < d(\sigma_P^2 - \sigma_Q^2)$, the inequality holds and likelihood inversion can happen. This is consistent with experimental results of Figure 5(a) of \citet{nalisnick2018do} because for the CIFAR-10 vs SVHN case where likelihood inversion occurs, we can confirm that SVHN is concentrated on the pixel mean of CIFAR-10. Since it is impossible to accurately estimate entropy and KL divergence for the actual data distribution, we cannot verify whether the inequality holds, but this observation provides indirect support for it. Additionally, we extend Theorem \ref{theorem:gap_expectation_likelihood_iid} by relaxing the identically distributed assumption.

\begin{theorem}[Impact of Dimensionality on Likelihood Gap]
\label{theorem:gap_expectation_likelihood_indep_appendix}
Let $P=\prod^d_{i=1}p_i(x_i)$ be an independent $d$-dimensional continuous probability density model in $\mathbb{R}^d$ and $Q$ be a distribution on $\mathbb{R}^d$ with a density whose $i$-th one-dimensional marginal is $q_i$. Assume the conditions of Lemma~\ref{lemma:kldivergence_approx} and $\mathbb{E}_{x_i\sim q_i}[|\log p_i(x_i)|]<\infty$ for all $i\in[d]$. Let $P_{\theta}$ be a well-trained density estimation model that approximates $P$ (i.e., $p_{\theta}(x) \rightarrow p(x)$ pointwise as $\theta \rightarrow \theta_0$). If $\mathbb{H}(P)-\mathbb{H}(Q) > D_{KL}(Q||P)$, then in the limit as $\theta \rightarrow \theta_0$, the lower bound of the gap between the expectations of the likelihood for $P$ and $Q$ decreases linearly with respect to $d$.
\end{theorem}

\begin{proof}
By the same decomposition as in Theorem~\ref{theorem:gap_expectation_likelihood_iid}, taking the limit as $\theta \rightarrow \theta_0$ and applying Lemma~\ref{lemma:kldivergence_approx} yields $\lim_{\theta \rightarrow \theta_0}D_{KL}(Q||P_{\theta}) = D_{KL}(Q||P)$ and $\lim_{\theta \rightarrow \theta_0}D_{KL}(P||P_{\theta}) = 0$. Therefore:
\begin{align}
\lim_{\theta \rightarrow \theta_0}\left[\mathbb{E}_{\mathbf{x}\sim P}[\log p_{\theta}(\mathbf{x})] - \mathbb{E}_{\mathbf{x}\sim Q}[\log p_{\theta}(\mathbf{x})]\right]
= D_{KL}(Q||P) + \mathbb{H}(Q) - \mathbb{H}(P).
\end{align}
Since $D_{KL}(Q||P)+\mathbb{H}(Q)=\mathbb{E}_{\mathbf{x}\sim Q}[-\log p(\mathbf{x})]$ and the components of $P$ are mutually independent, the linearity of expectation together with the fact that the $i$-th marginal of $Q$ is $q_i$ gives
\begin{align}
D_{KL}(Q||P)+\mathbb{H}(Q)=\sum^d_{i=1}\mathbb{E}_{x_i\sim q_i}[-\log p_i(x_i)]=\sum^d_{i=1}\left(D_{KL}(q_i||p_i)+\mathbb{E}_{x_i\sim q_i}[-\log q_i(x_i)]\right).
\end{align}
Hence the dependence among the coordinates of $Q$ does not affect the result. Together with $\mathbb{H}(P)=\sum^d_{i=1}\mathbb{E}_{x_i\sim p_i}[-\log p_i(x_i)]$ this yields:
\begin{align}
\begin{split}
&\lim_{\theta \rightarrow \theta_0}\left[\mathbb{E}_{\mathbf{x}\sim P}[\log p_{\theta}(\mathbf{x})] - \mathbb{E}_{\mathbf{x}\sim Q}[\log p_{\theta}(\mathbf{x})]\right] \\
&= \sum^d_{i=1}\left(D_{KL}(q_i||p_i) + \mathbb{E}_{x_i\sim q_i}[-\log q_i(x_i)] - \mathbb{E}_{x_i\sim p_i}[-\log p_i(x_i)]\right).
\end{split}
\end{align}
Define $\delta_i \triangleq \mathbb{E}_{x_i\sim q_i}[-\log q_i(x_i)] - \mathbb{E}_{x_i\sim p_i}[-\log p_i(x_i)]$ and $k = \argmin_{i\in[d]}\delta_i$. The assumption $\mathbb{H}(P)-\mathbb{H}(Q) > D_{KL}(Q||P)$ implies that the limit above is negative. Since $D_{KL}(q_i||p_i)\ge 0$ we obtain $\sum^d_{i=1}\delta_i < 0$, so there exists $j \in [d]$ such that $\delta_j < 0$, and by definition of $k$:
\begin{align}
0 > \delta_j \ge \delta_k.
\end{align}
Therefore, the following lower bound holds:
\begin{align}
\lim_{\theta \rightarrow \theta_0}\left[\mathbb{E}_{\mathbf{x}\sim P}[\log p_{\theta}(\mathbf{x})] - \mathbb{E}_{\mathbf{x}\sim Q}[\log p_{\theta}(\mathbf{x})]\right] = \sum^d_{i=1}D_{KL}(q_i||p_i) + \sum^d_{i=1}\delta_i \ge \sum^d_{i=1}\delta_i \ge d\,\delta_k,
\end{align}
which shows that the lower bound decreases linearly with respect to $d$.
\end{proof}

Through Theorem~\ref{theorem:gap_expectation_likelihood_indep_appendix}, if $\mathbb{H}(P)-\mathbb{H}(Q) > D_{KL}(Q||P)$ is satisfied, the lower bound of the likelihood gap decreases linearly with dimensionality in the limit of a perfectly trained model, even when the component distributions $\{p_i\}$ and $\{q_i\}$ are not identically distributed across dimensions. Furthermore, we show in Corollary~\ref{cor:upper_bound_auroc_appendix} that under additional assumptions, the upper bound of AUROC also decays polynomially with $d$.

\begin{corollary}[Special Case of Theorem~\ref{theorem:gap_expectation_likelihood_indep}: Guaranteed Negative Difference]
\label{cor:special_case_gap_indep}
Let $P=\prod^d_{i=1}p_i(x_i)$ be an independent $d$-dimensional continuous probability density model in $\mathbb{R}^d$ and $Q$ be a distribution on $\mathbb{R}^d$ with a density whose $i$-th one-dimensional marginal is $q_i$. Assume the conditions of Lemma~\ref{lemma:kldivergence_approx} and $\mathbb{E}_{x_i\sim q_i}[|\log p_i(x_i)|]<\infty$ for all $i\in[d]$. Let $P_{\theta}$ be a well-trained density estimation model that approximates $P$ (i.e., $p_{\theta}(x) \rightarrow p(x)$ pointwise as $\theta \rightarrow \theta_0$). If $\forall i\in[d],\; \mathbb{H}(p_i)-\mathbb{H}(q_i) > D_{KL}(q_i||p_i)$, then in the limit as $\theta \rightarrow \theta_0$, the lower and upper bounds of the gap between the expectations of the likelihood for $P$ and $Q$ decrease linearly with respect to $d$.
\end{corollary}

\begin{proof}
Following the proof of Theorem~\ref{theorem:gap_expectation_likelihood_indep_appendix}, taking the limit as $\theta \rightarrow \theta_0$ and applying Lemma~\ref{lemma:kldivergence_approx} yields:
\begin{align}
&\lim_{\theta \rightarrow \theta_0}\left[\mathbb{E}_{\mathbf{x}\sim P}[\log p_{\theta}(\mathbf{x})] - \mathbb{E}_{\mathbf{x}\sim Q}[\log p_{\theta}(\mathbf{x})]\right] \notag\\
&= \sum^d_{i=1}\left(D_{KL}(q_i||p_i) + \mathbb{E}_{x_i\sim q_i}[-\log q_i(x_i)] - \mathbb{E}_{x_i\sim p_i}[-\log p_i(x_i)]\right).
\end{align}
Define $k = \argmax_{i\in[d]}\left(D_{KL}(q_i||p_i) + \mathbb{E}_{x_i\sim q_i}[-\log q_i(x_i)] - \mathbb{E}_{x_i\sim p_i}[-\log p_i(x_i)]\right)$. Since $\mathbb{H}(p_i)-\mathbb{H}(q_i) > D_{KL}(q_i||p_i) \ge 0$ for all $i \in [d]$, we have for every $j\in[d]$:
\begin{align}
0 &> D_{KL}(q_k||p_k) + \mathbb{E}_{x_k\sim q_k}[-\log q_k(x_k)] - \mathbb{E}_{x_k\sim p_k}[-\log p_k(x_k)] \notag\\
&\ge D_{KL}(q_j||p_j) + \mathbb{E}_{x_j\sim q_j}[-\log q_j(x_j)] - \mathbb{E}_{x_j\sim p_j}[-\log p_j(x_j)].
\end{align}
Therefore, the upper bound satisfies:
\begin{align}
\begin{split}
&\sum^d_{i=1}\left(D_{KL}(q_i||p_i) + \mathbb{E}_{x_i\sim q_i}[-\log q_i(x_i)] - \mathbb{E}_{x_i\sim p_i}[-\log p_i(x_i)]\right) \\
&\le d\left(D_{KL}(q_k||p_k) + \mathbb{E}_{x_k\sim q_k}[-\log q_k(x_k)] - \mathbb{E}_{x_k\sim p_k}[-\log p_k(x_k)]\right) < 0.
\end{split}
\end{align}
Moreover, the upper bound above shows that the limit is negative, and by the proof of Theorem~\ref{theorem:gap_expectation_likelihood_indep_appendix} it equals $D_{KL}(Q||P)+\mathbb{H}(Q)-\mathbb{H}(P)$. Hence $\mathbb{H}(P)-\mathbb{H}(Q) > D_{KL}(Q||P)$, and the lower bound follows from Theorem~\ref{theorem:gap_expectation_likelihood_indep_appendix}. This shows that both the lower and upper bounds decrease linearly with respect to $d$.
\end{proof}

The results of Corollary~\ref{cor:special_case_gap_indep} confirm that when the per-dimension condition holds for all $i \in [d]$, both bounds of the expectation gap are guaranteed to be negative and decrease linearly with dimensionality in the limit. We show in Corollary~\ref{cor:upper_bound_auroc_appendix} that under an additional moment assumption, the upper bound of AUROC also decays polynomially with $d$.

\begin{corollary}[Dimensionality and AUROC Upper Bound]
\label{cor:upper_bound_auroc_appendix}
Building on the assumptions of Corollary~\ref{cor:special_case_gap_indep}, suppose the $n$-th absolute central moment of the log-likelihood difference, $\log p_\theta(Y) - \log p_\theta(X)$, scales as $\mathcal{O}(d^k)$ for some $n > 1$ and $k < n$. In this case, if the average log-likelihood gap becomes negative, the maximum achievable AUROC for distinguishing samples from $P$ and $Q$ is inversely related to the dimensionality $d$. This indicates that as the dimension increases, the likelihood test becomes fundamentally less effective at separating normal and abnormal samples.
\end{corollary}
\begin{proof}
By \citet{zhang2021understanding}, AUROC of likelihood test is defined by
\begin{align}
    \text{AUROC} = \Pr(\log p_\theta(X) > \log p_\theta(Y)), \: X\sim P, \:Y \sim Q.
\end{align}
From Corollary~\ref{cor:special_case_gap_indep}, the average log-likelihood gap satisfies $\mathbb{E}_{X\sim P, Y\sim Q}[\log p_\theta(X) - \log p_\theta(Y)] < 0$ in the limit, which implies $\mu = \mathbb{E}_{X\sim P, Y\sim Q}[\log p_\theta(Y) - \log p_\theta(X)] > 0$. Then we can derive:
\begin{align}
    \begin{split}
        1-\text{AUROC}
        &=\Pr(\log p_\theta(Y) - \log p_\theta(X) > 0) \\
        &=\Pr(\mu-(\log p_\theta(Y) - \log p_\theta(X)) < \mu) \\
        &\ge\Pr(|\mu-(\log p_\theta(Y) - \log p_\theta(X))| < \mu) \\
        &\ge 1-\frac{\mathbb{E}[|\log p_\theta(Y) - \log p_\theta(X)-\mu|^n]}{\mu^n} \:(\because \text{polynomial Markov inequality}).\\
    \end{split}
\end{align}
So, AUROC is bounded above:
\begin{align}
    \text{AUROC} \le \frac{\mathbb{E}[|\log p_\theta(Y) - \log p_\theta(X)-\mu|^n]}{\mu^n}.
\end{align}
Define $\Delta_{\min}\triangleq\min_{i\in[d]}\left(\mathbb{H}(p_i)-\mathbb{H}(q_i)-D_{KL}(q_i||p_i)\right)$, which is positive by the assumption of Corollary~\ref{cor:special_case_gap_indep}. Since we know from Corollary~\ref{cor:special_case_gap_indep} that the upper bound of $-\mu = \mathbb{E}_{X\sim P, Y\sim Q}[\log p_\theta(X) - \log p_\theta(Y)]$ decreases linearly in $d$, we can derive:
\begin{align}
    \mu \ge d\,\Delta_{\min} > 0,
\end{align}
which gives:
\begin{align}
    \frac{1}{\mu^n} \le \frac{1}{d^n\Delta_{\min}^n}.
\end{align}
Therefore, $\exists M\in \mathbb{R}$ such that:
\begin{align}
    \text{AUROC}
    \le \frac{\mathcal{O}(d^k)}{d^n\Delta_{\min}^n}
    \le \frac{d^{k-n}|M|}{\Delta_{\min}^n}.
\end{align}
Because $k-n<0$, the upper bound of AUROC decays as $d^{-(n-k)}$ and thus vanishes as $d\to\infty$, provided that $\Delta_{\min}$ does not vanish as $d$ grows (e.g., in the i.i.d. case).
\end{proof}

We note that the assumption on the $n$-th absolute central moment of $\log p_\theta(Y) - \log p_\theta(X)$ is mild. To see this, suppose that $X$ and $Y$ are drawn independently and that the components of both $P$ and $Q$ are independent. In the limit $\theta\to\theta_0$ the model density $p_\theta$ coincides with the density of $P$ and hence factorizes as $p_\theta(x)=\prod_{i=1}^{d}p_i(x_i)$, where $p_i$ denotes the $i$-th marginal. Consequently $\log p_\theta(X)=\sum_{i=1}^{d}\log p_i(X_i)$ and $\log p_\theta(Y)=\sum_{i=1}^{d}\log p_i(Y_i)$ are both sums of independent terms. Since the variance of a sum of independent random variables is the sum of their variances, and since $X$ and $Y$ are independent, we obtain $$\text{Var}(\log p_\theta(Y)-\log p_\theta(X))=\text{Var}(\log p_\theta(Y))+\text{Var}(\log p_\theta(X))=\sum_{i=1}^{d}\left[\text{Var}(\log p_i(Y_i))+\text{Var}(\log p_i(X_i))\right].$$ Hence if there exists a constant $M$ independent of $d$ such that $\text{Var}(\log p_i(X_i))\le M$ and $\text{Var}(\log p_i(Y_i))\le M$ for every $i$, then the variance is bounded by $2dM=\mathcal{O}(d)$, and setting $n=2$ and $k=1$ satisfies $k<n$. As a concrete example, when $P$ and $Q$ are isotropic Gaussians, $\log p_\theta(X)$ is an affine function of a $\chi^2(d)$ variable, whose variance scales as $\mathcal{O}(d)$, which is exactly the case where every coordinate contributes a constant variance. Furthermore, this condition concerns the variance of the log-density rather than that of the feature itself, and therefore it can hold even when the feature has infinite variance. For instance, let $X_i$ follow a Pareto distribution with shape $\alpha$ and scale $x_m$, whose density is $p_i(x)=\alpha x_m^{\alpha}x^{-(\alpha+1)}$ for $x\ge x_m$. Then $$\log p_i(X_i)=\log\alpha+\alpha\log x_m-(\alpha+1)\log X_i.$$ Because additive constants do not change the variance, this gives $\text{Var}(\log p_i(X_i))=(\alpha+1)^2\,\text{Var}(\log X_i)$. Let $Z=\log(X_i/x_m)=\log X_i-\log x_m$, so that $\text{Var}(\log X_i)=\text{Var}(Z)$. Its tail satisfies $$P(Z>z)=P(X_i>x_m e^{z})=\left(\frac{x_m}{x_m e^{z}}\right)^{\alpha}=e^{-\alpha z}\quad\text{for } z\ge 0,$$ which shows that $Z$ follows an exponential distribution with rate $\alpha$ and $\text{Var}(Z)=1/\alpha^{2}$. Therefore $$\text{Var}(\log p_i(X_i))=(\alpha+1)^2\cdot\frac{1}{\alpha^{2}}=\left(1+\frac{1}{\alpha}\right)^2<\infty\quad\text{for every }\alpha>0,$$ although $\text{Var}(X_i)=\infty$ for $\alpha\le 2$. This shows that heavy-tailed features do not violate the assumption. More generally, since $n$ can be freely chosen, it suffices to find \textit{any} $n > 1$ for which the $n$-th absolute central moment of $\log p_\theta(Y) - \log p_\theta(X)$ grows as $\mathcal{O}(d^k)$ with $k < n$. This condition fails only in the pathological case where the $n$-th central moment grows as $\mathcal{O}(d^n)$ for every $n > 1$, which requires an unusually strong dependence on $d$ across all orders (e.g., variance $\propto d^2$, skewness $\propto d^3$, and so on). Such behavior is rarely observed in practice, and the assumption is therefore broadly applicable.

The preceding results require the components of $P$ to be mutually independent across all dimensions, which may not reflect the correlation structures present in real-world tabular data where feature dependencies are ubiquitous. Hence we consider a block-dependent structure in which arbitrary dependence is permitted within each block while the blocks of $P$ remain mutually independent. This structure captures a broader class of data-generating processes, including datasets with correlated feature groups, without sacrificing the core theoretical conclusion. Accordingly, we restate Theorem~\ref{theorem:gap_block_dependent} and provide its proof.

\begin{theorem}[Impact of Dimensionality on Likelihood Gap: Block-Dependent Case]
\label{theorem:gap_block_dependent_appendix}
Let $P = \prod_{k=1}^{n} P_k$ and $Q$ be a probability density model whose $k$-th block marginal is $Q_k$, where each $P_k$ and $Q_k$ is a continuous probability density model in $\mathbb{R}^{d_k}$ with $\sum_{k=1}^{n} d_k = d$, and $P$ and $Q$ satisfy the same conditions as Lemma~\ref{lemma:kldivergence_approx}. Let $P_{\theta}$ be a well-trained density estimation model that approximates $P$ (i.e., $p_{\theta}(x) \rightarrow p(x)$ pointwise as $\theta \rightarrow \theta_0$). The components $P_k$ are mutually independent across $k$, with arbitrary dependence structure within each block. If $\mathbb{H}(P) - \mathbb{H}(Q) > D_{KL}(Q\|P)$, then:
\begin{align}
\lim_{\theta \rightarrow \theta_0}\left[\mathbb{E}_{x \sim P}[\log p_\theta(x)] - \mathbb{E}_{x \sim Q}[\log p_\theta(x)]\right] \geq d \min_{k \in [n]} \frac{\delta_k}{d_k},
\end{align}
where $\delta_k \triangleq \mathbb{H}(Q_k) - \mathbb{H}(P_k)$.
\end{theorem}

\begin{proof}
By the same decomposition as in Theorem~\ref{theorem:gap_expectation_likelihood_indep_appendix}, we have:
\begin{align}
\mathbb{E}_{x \sim P}[\log p_\theta(x)] - \mathbb{E}_{x \sim Q}[\log p_\theta(x)] = D_{KL}(Q\|P_{\theta}) - D_{KL}(P\|P_{\theta}) + \mathbb{H}(Q) - \mathbb{H}(P).
\end{align}
Taking the limit as $\theta \rightarrow \theta_0$ and applying Lemma~\ref{lemma:kldivergence_approx} yields $\lim_{\theta \rightarrow \theta_0}D_{KL}(Q\|P_{\theta}) = D_{KL}(Q\|P)$ and $\lim_{\theta \rightarrow \theta_0}D_{KL}(P\|P_{\theta}) = 0$. Let $p_k$ denote the density of $P_k$. Since $D_{KL}(Q\|P)+\mathbb{H}(Q)=\mathbb{E}_{x\sim Q}[-\log p(x)]$ and the blocks of $P$ are mutually independent, the linearity of expectation gives
\begin{align}
D_{KL}(Q\|P)+\mathbb{H}(Q)=\sum_{k=1}^{n}\mathbb{E}_{x_k\sim Q_k}[-\log p_k(x_k)]=\sum_{k=1}^{n}\left(D_{KL}(Q_k\|P_k)+\mathbb{H}(Q_k)\right),
\end{align}
and $\mathbb{H}(P)=\sum_{k=1}^{n}\mathbb{H}(P_k)$. Therefore:
\begin{align}
\lim_{\theta \rightarrow \theta_0}\left[\mathbb{E}_{x \sim P}[\log p_\theta(x)] - \mathbb{E}_{x \sim Q}[\log p_\theta(x)]\right] = \sum_{k=1}^{n} \left(D_{KL}(Q_k\|P_k) + \delta_k\right) \geq \sum_{k=1}^{n} \delta_k,
\end{align}
where the inequality follows from $D_{KL}(Q_k\|P_k) \geq 0$. Since $\sum_{k=1}^{n} d_k = d$:
\begin{align}
\sum_{k=1}^{n} \delta_k = \sum_{k=1}^{n} d_k \cdot \frac{\delta_k}{d_k} \geq d \min_{k \in [n]} \frac{\delta_k}{d_k}.
\end{align}
The assumption $\mathbb{H}(P) - \mathbb{H}(Q) > D_{KL}(Q\|P)$ implies that the limit above is negative. Since $D_{KL}(Q_k\|P_k)\ge 0$ we obtain $\sum_{k=1}^{n} \delta_k < 0$, and hence $\min_{k \in [n]} \frac{\delta_k}{d_k} < 0$, which completes the proof.
\end{proof}

Theorem~\ref{theorem:gap_block_dependent} establishes that even under arbitrary intra-block dependence, the expected log-likelihood gap remains negative and admits a closed-form lower bound $d\min_{k\in[n]}\frac{\delta_k}{d_k}$, extending the core conclusion of Theorem~\ref{theorem:gap_expectation_likelihood_indep_appendix} beyond the independence assumption. When dimensions are added to the minimizing block $k^* = \argmin_{k\in[n]}\frac{\delta_k}{d_k}$, the linear scaling of the lower bound is not guaranteed, as the per-dimension entropy gap $\frac{\delta_{k^*}}{d_{k^*}}$ may be diluted. However, when dimensions are added to any block other than $k^*$, $\min_{k\in[n]}\frac{\delta_k}{d_k}$ cannot increase while $d$ strictly increases, so the lower bound becomes strictly more negative. When an entirely new block is appended, since $\min_{k\in[n+1]}\frac{\delta_k}{d_k} \leq \min_{k\in[n]}\frac{\delta_k}{d_k} < 0$ holds regardless of the per-dimension entropy gap of the new block, the lower bound satisfies $(d+d_{n+1})\min_{k\in[n+1]}\frac{\delta_k}{d_k} \leq (d+d_{n+1})\min_{k\in[n]}\frac{\delta_k}{d_k} \leq d\min_{k\in[n]}\frac{\delta_k}{d_k} < 0$. In the latter two cases, the lower bound is guaranteed to become strictly more negative, which indicates that the potential severity of the likelihood inversion grows with dimensionality.

\section{Ablation Study of Anomaly Detection}
\label{detection_performance_ano_types}

\begin{table}[!h]
\caption{Anomaly detection performance across various anomaly types for deep models.}
\label{tab:ano_type_performance}
\centering
\setlength{\tabcolsep}{6pt}
\resizebox{0.98\linewidth}{!}{
\begin{tabular}{llccccccc}
\toprule
Type & Metric & GOAD & DeepSVDD & NeuTraLAD & ICL & MCM & DRL & NF-SLT \\
\midrule
\multirow{2}{*}{Local}
& AUROC & 0.7297 & 0.9605 & 0.9664 & 0.9646 & \textbf{0.9948} & 0.9545 & 0.9919 \\
& AUPRC & 0.7922 & 0.9678 & 0.9702 & 0.9697 & \textbf{0.9953} & 0.9635 & 0.9926 \\
\midrule
\multirow{2}{*}{Clustered}
& AUROC & \textbf{1} & \textbf{1} & 0.9174 & \textbf{1} & \textbf{1} & \textbf{1} & \textbf{1} \\
& AUPRC & \textbf{1} & \textbf{1} & 0.8620 & \textbf{1} & \textbf{1} & \textbf{1} & \textbf{1} \\
\midrule
\multirow{2}{*}{Global}
& AUROC & \textbf{1} & \textbf{1} & \textbf{1} & \textbf{1} & \textbf{1} & \textbf{1} & \textbf{1} \\
& AUPRC & \textbf{1} & \textbf{1} & \textbf{1} & \textbf{1} & \textbf{1} & \textbf{1} & \textbf{1} \\
\midrule
\multirow{2}{*}{Dependency}
& AUROC & 0.9986 & 0.9960 & 0.9974 & 0.9997 & 0.9990 & 0.9997 & \textbf{0.9999} \\
& AUPRC & 0.9983 & 0.9953 & 0.9966 & 0.9996 & 0.9983 & 0.9997 & \textbf{1} \\
\bottomrule
\end{tabular}
}
\end{table}

We apply the protocol for generating four types of anomalies suggested in \citet{han2022adbench} to the real dataset Satellite to verify whether NF-SLT can detect various types of anomalies well. Descriptions of the four types of anomalies and how synthetic anomaly data were generated are as follows: 
\begin{itemize}
    \item \textbf{Local anomalies} are data points that significantly differ from others in their immediate surroundings. After generating normal data using Gaussian Mixture Model (GMM, \citet{milligan1985algorithm, steinbuss2021benchmarking}), local anomalies are generated by multiplying the covariance by the scaling factor $\alpha = 5$.  
    \item \textbf{Global anomalies} are data points distributed widely across the space but remain distinctly separate from the normal data distribution. Normal data is created in the same way as when creating local anomalies, and global anomalies are sampled from uniform distribution $Unif(\alpha\cdot\min(X^k), \alpha\cdot\max(X^k))$ with the scaling factor $\alpha$ set to 1.1. $\bold{X^k}$ refers to the $k$-th feature and creates anomalies by adjusting the maximum and minimum boundary values of the feature.
    \item \textbf{Dependency anomalies} are data points that fail to conform to the usual relationships or dependencies observed in normal data. Normal data is created by capturing the dependency between features of normal data using Vine Copula \citep{aas2009pair}. By estimating the probability density function of each feature through Kernel Density Estimation (KDE, \citet{hastie2009elements}) and dependency anomalies with independent features are created.
    \item \textbf{Clustered anomalies} are groups of similar data points that collectively differ greatly from the normal data distribution. Normal data is created in the same way as when creating local anomalies, clustered anomalies are generated by multiplying the mean vector estimated through GMM by the scaling factor $\alpha = 5$.
\end{itemize}

The comparison models were the six deep models used for the CV/NLP embedding datasets in the main experiment of the paper, and the performance was measured using AUROC and AUPRC. According to the results of Table \ref{tab:ano_type_performance}, it can be confirmed that NF-SLT achieves the best or tied for the best performance among the deep models for all anomaly types except for local anomalies. In addition, since it shows performance next to MCM even for local anomaly types, the results verify the performance of NF-SLT in detecting various anomaly types.

Additionally, we evaluated NF-SLT on datasets with a significant categorical composition using simple preprocessing techniques (e.g., one-hot encoding). We conducted experiments on categorical-rich datasets from ADBench (InternetAds, campaign, census) and ODDs (nsl-kdd; \citet{rayana2016odds}). For the ADBench datasets, the input data were scaled using RobustScaler from Scikit-learn for all models except NeuTraLAD and DRL. NeuTraLAD was trained without feature scaling due to performance degradation observed when scaling was applied, and DRL employed the StandardScaler in accordance with the preprocessing configuration of the official implementation. The results of this experiment are reported in Table \ref{tab:categorical_data}, demonstrating that NF-SLT maintains strong performance even on datasets dominated by categorical features.

\begin{table}[!h]
\centering
\caption{AUROC on datasets with mostly categorical features. The number after the slash indicates the relative rank among the seven deep models.}
\label{tab:categorical_data}
\setlength{\tabcolsep}{6pt}
\resizebox{0.98\linewidth}{!}{
\begin{tabular}{lccccccc}
\toprule
Dataset & DeepSVDD & GOAD & NeuTraLAD & ICL & MCM & DRL & NF-SLT \\
\midrule
InternetAds & 0.6013 / 7 & 0.7356 / 5 & 0.7112 / 6 & \textbf{0.8815 / 1} & 0.7973 / 3 & 0.7603 / 4 & 0.8746 / 2 \\
campaign    & 0.7929 / 3 & 0.2837 / 7 & 0.7814 / 4 & 0.6933 / 6          & 0.7987 / 2 & 0.7525 / 5 & \textbf{0.8059 / 1} \\
census      & 0.6421 / 3 & 0.4140 / 6 & 0.4991 / 5 & 0.3917 / 7          & \textbf{0.7478 / 1} & 0.5646 / 4 & 0.7186 / 2 \\
nsl-kdd     & 0.8951 / 5 & 0.9026 / 4 & 0.9374 / 3 & 0.8798 / 6          & 0.8757 / 7 & 0.9459 / 2 & \textbf{0.9598 / 1} \\
\bottomrule
\end{tabular}
}
\end{table}
We additionally conducted a data contamination sensitivity test for NF-SLT.
For this experiment, we selected the annthyroid, cardio, and letter datasets from ADBench. NF-SLT was trained on datasets consisting primarily of normal samples with a small proportion of anomaly samples intentionally mixed into the training data. After training, anomaly detection was performed, and the resulting performance was evaluated using AUROC, which is reported in Table \ref{tab:contamination}. In Table \ref{tab:contamination}, the contamination ratio denotes the proportion of anomaly samples relative to the number of normal samples in the training set.

\begin{table}[!h]
\centering
\caption{Contamination sensitivity analysis of NF-SLT.}
\label{tab:contamination}
\setlength{\tabcolsep}{8pt}
\begin{tabular}{lccc}
\toprule
Contamination Ratio & annthyroid & cardio & letter \\
\midrule
1\% & 0.9048 & 0.9032 & 0.9195 \\
3\% & 0.8806 & 0.8797 & 0.9126 \\
5\% & 0.8752 & 0.8621 & 0.9005 \\
\bottomrule
\end{tabular}
\end{table}

Despite a relatively high contamination level of up to 5\%, the observed degradation in AUROC remained limited to approximately 2--4 points as the contamination ratio increased from 1\% to 5\%. These results demonstrate that NF-SLT exhibits strong robustness to data contamination, maintaining stable anomaly detection performance even in the presence of non-negligible anomalous samples during training.

\newpage

\section{Robustness of Intrinsic Dimension Estimator}
\label{sec:id_robustness}

We further assess the robustness of our intrinsic-dimension (ID) analysis along four axes: (i) estimator choice, (ii) sub-sampling, (iii) feature scaling, and (iv) categorical-feature handling.

\textbf{Alternative ID estimator beyond TwoNN and MLE}
To evaluate estimator dependence, we additionally estimate intrinsic dimension using a PCA-based estimator \citep{fan2010intrinsic} (denoted IPCA) and compare it with TwoNN and MLE on representative tabular datasets:

\begin{table}[!h]
    \centering
    \small
    \setlength{\tabcolsep}{4.0pt}
    \caption{ID estimates from an alternative estimator IPCA alongside TwoNN and MLE on representative tabular datasets.}
    \label{tab:id_alt_estimator}
    \begin{tabular}{lcccccccccc}
        \toprule
        Method & magicgamma & satellite & landsat & waveform & Wilt & annthyroid & breastw & cover & fault & fraud \\
        \midrule
        TwoNN & 7 & 15 & 14 & 17 & 5 & 3 & 4 & 3 & 5 & 4 \\
        MLE   & 5 &  9 & 11 & 16 & 4 & 3 & 4 & 4 & 5 & 8 \\
        IPCA  & 3 &  2 &  2 &  2 & 1 & 1 & 3 & 1 & 5 & 4 \\
        \bottomrule
    \end{tabular}
\end{table}

IPCA yields equal or smaller ID estimates than both TwoNN and MLE on all ten datasets, whereas TwoNN and MLE remain broadly consistent with each other. Since our analysis relies on qualitative ID trends rather than the absolute scale of a particular estimator, we use TwoNN as the primary estimator throughout the paper, with the agreement between TwoNN and MLE serving as supporting evidence.

\textbf{Robustness to sub-sampling}
We next estimate ID using TwoNN under different sub-sampling ratios:

\begin{table}[!h]
    \centering
    \small
    \setlength{\tabcolsep}{6.0pt}
    \caption{Stability of TwoNN ID estimates under different sub-sampling ratios on representative tabular datasets.}
    \label{tab:id_subsampling}
    \begin{tabular}{lcccccccc}
        \toprule
        Ratio & magicgamma & satellite & landsat & waveform & Wilt & annthyroid & breastw & fault \\
        \midrule
        100\% & 7 & 15 & 14 & 17 & 5 & 3 & 4 & 5 \\
         90\% & 7 & 15 & 14 & 17 & 5 & 3 & 4 & 5 \\
         80\% & 7 & 14 & 14 & 17 & 5 & 3 & 4 & 5 \\
         50\% & 7 & 14 & 13 & 16 & 5 & 4 & 4 & 4 \\
        \bottomrule
    \end{tabular}
\end{table}

The estimated IDs remain nearly unchanged across ratios, indicating that TwoNN is stable under sub-sampling at the levels considered.

\textbf{Robustness to different feature scalings}
We also examine the sensitivity of TwoNN ID estimates to standard feature scaling choices:

\begin{table}[!h]
    \centering
    \small
    \setlength{\tabcolsep}{6.0pt}
    \caption{Sensitivity of TwoNN ID estimates to common feature-scaling choices on representative tabular datasets.}
    \label{tab:id_scaling}
    \begin{tabular}{lcccccccc}
        \toprule
        Scaler & magicgamma & satellite & landsat & waveform & Wilt & annthyroid & breastw & fault \\
        \midrule
        No Scaler       & 7 & 15 & 14 & 17 & 5 & 3 & 4 & 5 \\
        Standard Scaler & 7 & 15 & 14 & 16 & 4 & 5 & 4 & 5 \\
        Minmax Scaler   & 7 & 15 & 14 & 17 & 4 & 4 & 5 & 4 \\
        \bottomrule
    \end{tabular}
\end{table}

TwoNN produces comparable ID estimates across scaling schemes, suggesting that our ID characterization is not driven by a particular normalization choice.

\textbf{Categorical variables and one-hot encoding}
Finally, using the Credit Card Client dataset with categorical features from~\citet{yeh2009comparisons}, we compare label encoding (default) and one-hot encoding. TwoNN returns the same intrinsic dimension (ID $=6$) under both preprocessing choices, indicating that the estimate is not sensitive to the categorical-feature encoding strategy.

Overall, these results suggest that the TwoNN estimates are stable across sub-sampling, feature scaling, and categorical preprocessing, while their absolute scale depends on the choice of estimator.

\section{Model Details and Hyperparameter Searching Space}
\label{model_detail_hyp}
Anomaly detection in tabular data has garnered significant attention due to its relevance in various real-world applications, such as fraud detection, medical diagnosis, etc. However, anomaly detection tasks often suffer from imbalance problems; hence several unsupervised methods have been developed, as they are not affected by this problem. The learning method of anomaly detection varies depending on the presence of a label; however, this section only introduces unsupervised methods that learn without having information about the label. Because existing approaches can be broadly categorized into shallow models and deep models, we organize this section into shallow and deep model categories.

\textbf{Shallow Models} Principal Component Analysis (PCA; \citealp{wold1987principal}) is a dimension reduction method that projects data to a lower dimension than the original using Singular Value Decomposition (SVD). This algorithm can also be utilized as an anomaly detection method by using the reconstruction error when the latent vector projected in low dimensions is reconstructed to the original data as an anomaly score \citep{shyu2003novel}. The Local Outlier Factor (LOF; \citealp{breunig2000lof}) is a method that adopts neighbors to determine anomaly scores from a local perspective. Isolation forest (IF; \citealp{liu2008isolation}) uses the concept that normal data is easier to isolate than outliers, one-class SVM (OCSVM; \citealp{scholkopf1999support}) determines the decision boundary by determining a support vector that can sufficiently explain the given training data well, and determines that data that exist outside this boundary are outliers. COPOD \citep{li2020copod} and ECOD \citep{li2022ecod} are similar anomaly detection algorithms that utilize the Empirical Cumulative Distribution Function (ECDF), and measure anomaly scores using the extremeness of input data under ECDF.

\textbf{Deep Models} DAGMM \citep{zong2018deep} utilizes a deep autoencoder to generate a low-dimensional representation and reconstruction error for each input data point, which is further fed into a Gaussian Mixture Model (GMM). DeepSVDD \citep{ruff2018deep} trains to bring the input data representation closer to the predefined center using a neural network, sets the boundary of the hypersphere and determines that data existing outside this are outliers. GOAD \citep{Bergman2020Classification-Based} is a model that generalizes the transformation-based self-supervised method used in the image domain and applies it to the tabular domain. NeuTraLAD \citep{qiu2021neural} and ICL \citep{shenkar2022anomaly} applied the contrastive learning method, one of the self-supervised methods, to the tabular domain to improve performance. MCM \citep{yin2024mcm} performs anomaly detection by extending the masking self-supervised method adopted in the NLP or image domain to the tabular domain. DO2HSC \citep{zhang2024deep} improves the limitations of DeepSVDD with a hypersphere assumption through orthogonal projection and a double hypersphere decision boundary but was excluded from the experiment. This is because the implementation code showed that the model uses the statistics of the test data when performing orthogonal projection, which was judged to be biased when compared to other models. NPT-AD \citep{thimonier2024beyond} was excluded from the experiments due to its excessively long training time for high-dimensional datasets with numerous data points.

Additionally, we provide details of the 13 comparison models and NF-SLT we implemented and the hyperparameter search space for each of them. Because it is impractical to search over every possible hyperparameter, we selected a small set of hyperparameters that we consider important for each model. For each search space, we chose commonly used values, ran all hyperparameter combinations 10 times on the 47 tabular datasets from ADBench, and then selected the combination with the highest average AUROC across datasets as the representative setting for that model. In Table~\ref{tab:hyperparam_searching_space}, we record the hyperparameter search space used in our experiments, and in Table~\ref{tab:optimal_hyperparam_comb}, we record the optimal hyperparameter combination finally selected from the hyperparameter search space. In the description of the deep model where the implementation code resides, we only specify the hyperparameters we modified. The computing environment used in the experiment was AMD Ryzen 9 7950X for CPU and RTX 4090 for GPU.

Regarding feature scaling, NeuTraLAD was trained without RobustScaler, as applying it led to a noticeable degradation in performance in our preliminary experiments. DRL employed the StandardScaler instead, following the preprocessing configuration specified in the official implementation, to ensure consistency with the original experimental setup and to enable a fair comparison with the reported results.

\textbf{PCA} We implemented PCA via the PyOD library, and chose n\_component\_ratio as the hyperparameter to explore. $n_{\mathrm{component\_ratio}}$ is a hyperparameter for what proportion of the dimensionality of the original input data is retained.

\textbf{LOF} We implemented LOF via the PyOD library and chose $N$ as the hyperparameter to explore. $N$ is a hyperparameter for how many neighbors to consider.

\textbf{IF} We implemented IF via the PyOD library and chose $N$ as the hyperparameter to explore. $N$ is the hyperparameter for how many trees to ensemble.

\textbf{OCSVM} We have implemented OCSVM via the scikit-learn library and have chosen kernel as the hyperparameter to explore. Kernel is the hyperparameter for which function to choose as the kernel function of OCSVM.

\textbf{COPOD} We implemented COPOD via the PyOD library and did not set a hyperparameter searching space because the model is hyperparameter-free.

\textbf{ECOD} We implemented ECOD via the PyOD library and did not set a hyperparameter searching space because the model is hyperparameter-free like COPOD.

\textbf{DAGMM} We ran DAGMM experiments based on the model implemented in \url{https://github.com/mperezcarrasco/PyTorch-DAGMM}. We set the batch size to 512, the epoch to 200. We chose $n_{gmm}$, $\lambda_{energy}$, and $\lambda_{cov}$ as hyperparameters to explore. $n_{gmm}$ is the hyperparameter for how many Gaussian mixture components to assume, $\lambda_{energy}$ is the weight of the sample energy term in the loss function, and $\lambda_{cov}$ is the weight of the regularization term on the covariance to avoid the singularity problem. However, all combinations of these hyperparameters yielded nearly the same performance for DAGMM.

\textbf{DeepSVDD} We implemented DeepSVDD ourselves based on the original paper~\cite{ruff2018deep}. First, we pretrain using Autoencoder for 100 epochs and define the average value of the latent vector obtained by encoding the train data using the corresponding encoder part as the center. Then, we implemented DeepSVDD by training the encoder for 200 epochs with the loss function of the MSE (Mean Squared Error) of the center. The encoder consists of 3 layers with the same hidden size, the batch size is 512, the optimizer is AdamW \citep{loshchilov2018decoupled}, the activation function is ReLU, the weight decay is 1e-4, and the learning rate scheduler is CosineAnnealingWarmRestarts \citep{loshchilov2017sgdr}. We set learning rate and latent dimension as the hyperparameters to search for (latent dimension and hidden size are the same).

\textbf{GOAD} We ran GOAD experiments based on the model implemented in \url{https://github.com/lironber/GOAD}. We set the batch size to 512, the epoch to 200. We chose $n_{rot}$, $d_{latent}$, and $d_{out}$ as hyperparameters to explore. $n_{rot}$ is a hyperparameter for how many transformations to perform, and $d_{latent}$ and $d_{out}$ represent the latent dimension and the dimension after transforming the input data, respectively.

\textbf{NeuTraLAD} We ran NeuTraLAD experiments based on the model implemented in \url{https://github.com/boschresearch/NeuTraL-AD}.
We set the batch size to 512, the epoch to 200, number of transformation to 11. The hidden and latent dimensions were selected using the automatic dimension selection implemented in the implementation code. We chose $n_{enclayer}$, $n_{translayer}$, and trans\_method hyperparameters to explore. $n_{enclayer}$ represents the number of encoder layers, and $n_{translayer}$ represents the number of layers in the block that performs the transformation. Finally, trans\_method is a hyperparameter for how we want to perform the transformation.

\textbf{ICL} We ran ICL experiments based on author's official supplementary at \url{https://openreview.net/forum?id=_hszZbt46bT}. We set 
the batch size 512. We chose $\tau$ and $d_{latent}$ as hyperparameters to explore. Since earlystopping is implemented in the implementation code, epoch used 2000 as it was specified in the implementation code. $\tau$ represents the softmax temperature, $d_{latent}$ represents the latent dimension.

\textbf{MCM} We ran MCM experiments based on author's official supplementary at \url{https://openreview.net/forum?id=lNZJyEDxy4}. According to the paper, the authors tuned only $\lambda$ and learning rate per dataset, so we set these two hyperparameters to explore. $\lambda$ represent the weight of mask diversity loss.

\textbf{DRL} We ran DRL experiments based on author's official supplementary at \url{https://openreview.net/forum?id=CJnceDksRd}. We adopted the hyperparameters used in the loss function directly from the official implementation. In contrast, we tuned the architectural hyperparameters, including the number of predefined bases in the latent space ($n_{basis}$), the number of encoder layers ($n_{enclayer}$), and the dimensionality of the latent representation ($d_{latent}$).

\textbf{NF-SLT} We implemented NICE based on \url{https://github.com/DakshIdnani/pytorch-nice}. The coupling layer blocks were set to 10, the epoch was set to 200, and the weight decay was set to 1e-4. The optimizer was AdamW, and the learning rate scheduler was CosineAnnealingWarmRestarts. The batch size was set to 512. The prior distribution of the latent vector was set to $\mathcal{N}(0,I_d)$, and then trained to minimize the negative log-likelihood over the training latent vector. Additionally, the coupling layer cannot handle odd-dimensional inputs, which was addressed by appending a single zero-padding dimension. We chose learning rate, $d_{latent}$, and $n_{layer}$ as hyperparameters to explore. $d_{latent}$ represents the latent dimension of the coupling layer block, and $n_{layer}$ represents the number of layers in the coupling layer block.

In addition, we quantify the hyperparameter sensitivity of NF-SLT by comparing the metrics in Table~\ref{tab:auroc_unfair_full} (dataset-wise tuned) and Table~\ref{tab:auroc_fair} (globally tuned), and report their differences in Table~\ref{tab:hyperparameter_sensitivity}. We adopt this benchmark-level comparison because deep models have heterogeneous hyperparameter search spaces (Table~\ref{tab:hyperparam_searching_space}), which makes sensitivity analysis with respect to a common set of hyperparameters impractical. Under this setup, a larger AUROC difference indicates higher sensitivity, while positive changes in the average-rank index and fail-ratio index indicate a deterioration in relative performance. Table~\ref{tab:hyperparameter_sensitivity} shows that even when using the dataset-wise optimal configuration from Table~\ref{tab:auroc_unfair_full}, NF-SLT exhibits only a marginal AUROC improvement relative to other methods, and its relative indices (average-rank and fail-ratio) in fact worsen. Hence NF-SLT is less sensitive to hyperparameter choices than the other deep comparison models, since a single globally tuned configuration already achieves performance close to that of the dataset-wise optimal configurations, whereas the other models gain substantially from per-dataset tuning.

\begin{table}[!h]
\centering
\caption{Hyperparameter search space for each model. $\eta$ denotes the learning rate.}
\label{tab:hyperparam_searching_space}
\small
\setlength{\tabcolsep}{7pt}
\renewcommand{\arraystretch}{1.05}

\resizebox{0.8\linewidth}{!}{
\begin{tabular}{lp{0.5\linewidth}}
\hline
Model & Hyperparameter search space \\
\hline
PCA &
$n_{\text{component\_ratio}} \in \{0.75, 0.8, 0.85, 0.9, 0.95\}$ \\
\hline
LOF &
$N \in \{1, 3, 5, 10, 20\}$ \\
\hline
IF &
$N \in \{30, 50, 100, 200, 400\}$ \\
\hline
OCSVM &
kernel $\in \{\texttt{linear}, \texttt{rbf}\}$ \\
\hline
COPOD & None \\
\hline
ECOD & None \\
\hline
DAGMM &
\begin{tabular}[t]{@{}l@{}}
$n_{\text{gmm}} \in \{3, 4, 5\}$ \\
$\lambda_{\text{energy}} \in \{0.1, 0.5, 1\}$ \\
$\lambda_{\text{cov}} \in \{5\mathrm{e}{-3}, 1\mathrm{e}{-2}, 5\mathrm{e}{-2}\}$
\end{tabular} \\
\hline
DeepSVDD &
\begin{tabular}[t]{@{}l@{}}
$\eta \in \{1\mathrm{e}{-4}, 5\mathrm{e}{-4}, 1\mathrm{e}{-3}\}$ \\
$d_{\text{latent}} \in \{64, 128, 256\}$
\end{tabular} \\
\hline
GOAD &
\begin{tabular}[t]{@{}l@{}}
$n_{\text{rot}} \in \{16, 32, 64\}$ \\
$d_{\text{latent}} \in \{8, 16, 32\}$ \\
$d_{\text{out}} \in \{4, 8, 16\}$
\end{tabular} \\
\hline
NeuTraLAD &
\begin{tabular}[t]{@{}l@{}}
$n_{\text{enclayer}} \in \{5, 6, 7\}$ \\
$n_{\text{translayer}} \in \{2, 3, 4\}$ \\
trans\_method $\in \{\texttt{mul}, \texttt{residual}\}$
\end{tabular} \\
\hline
ICL &
\begin{tabular}[t]{@{}l@{}}
$\tau \in \{1\mathrm{e}{-2}, 1\mathrm{e}{-1}, 1\}$ \\
$d_{\text{latent}} \in \{100, 200, 300\}$
\end{tabular} \\
\hline
MCM &
\begin{tabular}[t]{@{}l@{}}
$\lambda \in \{1\mathrm{e}{-2}, 1\mathrm{e}{-1}, 1, 10\}$ \\
$\eta \in \{1\mathrm{e}{-4}, 5\mathrm{e}{-4}, 1\mathrm{e}{-3}, 5\mathrm{e}{-3}, 1\mathrm{e}{-2}\}$
\end{tabular} \\
\hline
DRL &
\begin{tabular}[t]{@{}l@{}}
$n_{\text{basis}} \in \{5, 9, 13\}$ \\
$n_{\text{enclayer}} \in \{3, 4, 5\}$ \\
$d_{\text{latent}} \in \{64, 128, 256\}$
\end{tabular} \\
\hline
NF-SLT &
\begin{tabular}[t]{@{}l@{}}
$\eta \in \{1\mathrm{e}{-4}, 5\mathrm{e}{-4}, 1\mathrm{e}{-3}\}$ \\
$d_{\text{latent}} \in \{64, 128, 256\}$ \\
$n_{\text{layer}} \in \{2, 3, 4\}$
\end{tabular} \\
\hline
\end{tabular}
}

\end{table}

\begin{table}[!t]
\centering
\caption{Optimal hyperparameter combination for each model.}
\label{tab:optimal_hyperparam_comb}
\small
\setlength{\tabcolsep}{7pt}
\renewcommand{\arraystretch}{1.05}

\resizebox{0.6\linewidth}{!}{
\begin{tabular}{lp{0.3\linewidth}}
\hline
Model & Optimal hyperparameter \\
\hline
PCA   & $n_{\text{component\_ratio}} = 0.8$ \\
\hline
LOF   & $N = 20$ \\
\hline
IF    & $N = 400$ \\
\hline
OCSVM & kernel $= \texttt{linear}$ \\
\hline
COPOD & None \\
\hline
ECOD  & None \\
\hline
DAGMM &
\begin{tabular}[t]{@{}l@{}}
$n_{\text{gmm}} = 3$ \\
$\lambda_{\text{energy}} = 0.1$ \\
$\lambda_{\text{cov}} = 5\mathrm{e}{-3}$
\end{tabular} \\
\hline
DeepSVDD &
\begin{tabular}[t]{@{}l@{}}
$\eta = 1\mathrm{e}{-4}$ \\
$d_{\text{latent}} = 256$
\end{tabular} \\
\hline
GOAD &
\begin{tabular}[t]{@{}l@{}}
$n_{\text{rot}} = 32$ \\
$d_{\text{latent}} = 32$ \\
$d_{\text{out}} = 4$
\end{tabular} \\
\hline
NeuTraLAD &
\begin{tabular}[t]{@{}l@{}}
$n_{\text{enclayer}} = 5$ \\
$n_{\text{translayer}} = 3$ \\
trans\_method $= \texttt{mul}$
\end{tabular} \\
\hline
ICL &
\begin{tabular}[t]{@{}l@{}}
$\tau = 1\mathrm{e}{-1}$ \\
$d_{\text{latent}} = 300$
\end{tabular} \\
\hline
MCM &
\begin{tabular}[t]{@{}l@{}}
$\lambda = 0.1$ \\
$\eta = 5\mathrm{e}{-3}$
\end{tabular} \\
\hline
DRL &
\begin{tabular}[t]{@{}l@{}}
$n_{\text{basis}} = 9$ \\
$n_{\text{enclayer}} = 3$ \\
$d_{\text{latent}} = 256$
\end{tabular} \\
\hline
NF-SLT &
\begin{tabular}[t]{@{}l@{}}
$\eta = 1\mathrm{e}{-3}$ \\
$d_{\text{latent}} = 256$ \\
$n_{\text{layer}} = 2$
\end{tabular} \\
\hline
\end{tabular}
}

\end{table}
\clearpage

\begin{table}[!h]
\caption{Hyperparameter sensitivity of deep models. The columns that represent differences represent the differences between the indicators in Table \ref{tab:auroc_unfair_full} and Table \ref{tab:auroc_fair}. If it is sensitive to changes in hyperparameters, there will be a large difference in performance when experiments are performed with the optimal hyperparameter combination in the hyperparameter searching space for each dataset, compared to when experiments are performed with the same hyperparameter combination for all datasets.}
\centering
\label{tab:hyperparameter_sensitivity}
\small
\setlength{\tabcolsep}{8pt}
\begin{tabular}{lccc}
\toprule
Method     & AUROC Diff & Avg Rank Diff & Fail Ratio Diff \\
\midrule
DeepSVDD   & 0.0366                          & -1.02                              & -0.21           \\
GOAD       & 0.1124                          & -0.68                              & -0.04           \\
NeuTraLAD & 0.0310                          & -0.53                              & -0.06           \\
ICL        & 0.0284                          & -0.40                              & -0.15           \\
MCM        & 0.0302                          & -0.96                              & -0.04           \\
DRL        & 0.0233                          & -0.32                              & -0.02           \\
NF-SLT     & \textbf{0.0116}                 & \textbf{0.15}                      & \textbf{0.04}   \\
\bottomrule
\end{tabular}
\vspace{-0.4cm}
\end{table}

\section{Performance Comparison of NF-SLT Across Normalizing Flow Models}
\label{nice_realnvp_comp}

In this section, we set the flow model used for NF-SLT to RealNVP instead of NICE, and then compared the performance with NICE. RealNVP replaces the additive coupling layers of NICE with affine coupling layers. We set the hyperparameters of RealNVP to be the same as those of NICE, which recorded the performance in Table~\ref{tab:auroc_fair}, except for the learning rate. We set the learning rate to 1e-3, 5e-4, and 1e-4, and selected 5e-4, the learning rate that recorded the highest AUROC performance, to record the performance in Table~\ref{tab:nice_realnvp_comparison}. Table~\ref{tab:nice_realnvp_comparison} shows that RealNVP underperformed on all metrics compared to NICE. This indicates that the higher expressive power of affine coupling layers over additive coupling layers does not translate into better detection performance in tabular anomaly detection.

We do not include the Glow architecture as an additional comparison model for the following reasons. The main components of Glow are Actnorm and permutations implemented via 1×1 convolutions. First, we do not use batch normalization in our NICE and RealNVP implementations, and Actnorm is primarily designed to address situations where large batch sizes are unavailable, so its benefits are limited in our setting. Second, in our preliminary experiments, applying 1×1 convolution permutations led to worse performance than RealNVP. For these reasons, we omit Glow from our main comparisons.

\begin{table}[!h]
\caption{Performance comparison between NICE and RealNVP}
\label{tab:nice_realnvp_comparison}
\centering
\small
\setlength{\tabcolsep}{8pt}
\begin{tabular}{lccccc}
\toprule
Model   & AUROC $ \uparrow$ & AUPRC $ \uparrow$ & Avg. Rank $ \downarrow$ & Top2 Cum. Ratio $ \uparrow$   & Fail Ratio $ \downarrow$ \\
\midrule
NICE    & 0.8575 & 0.6398 & 3.74      & 0.40            & 0.06 \\
RealNVP & 0.8480 & 0.6385 & 4.45      & 0.30            & 0.09 \\
\bottomrule
\end{tabular}
\end{table}

Additionally, \citet{draxler2024universality} showed that a volume-preserving model like NICE is not a universal approximator. However, this limitation does not appear to be the main reason for the strong performance of NICE, since RealNVP, which is not volume-preserving, also performs better than most comparison models while still performing worse than NICE. Therefore, why the performance of RealNVP slightly decreases compared to NICE despite its higher expressive power remains an open question that can be connected to the theoretical background of flow models in future work.

\begin{table}[!h]
\centering
\caption{Performance comparison between NICE and MAF}
\label{tab:maf_comparison}
\small
\setlength{\tabcolsep}{6pt}
\resizebox{0.8\linewidth}{!}{
\begin{tabular}{lccccccc}
\toprule
Model & Ionosphere & Lymphography & WPBC & thyroid & yeast & glass & pendigits \\
\midrule
NICE  & \textbf{0.9581} & 0.9746          & 0.5051          & \textbf{0.9840} & \textbf{0.4652} & \textbf{0.8867} & \textbf{0.9930} \\
MAF   & 0.8406          & \textbf{0.9908} & \textbf{0.5070} & 0.9752          & 0.4359          & 0.7677          & 0.9926          \\
\bottomrule
\end{tabular}
}
\end{table}

Also, we attempted to run anomaly detection experiments with Masked Autoregressive Flow (MAF, \cite{papamakarios2017masked}) on all datasets. However, due to numerical stability issues during MAF training, we were unable to obtain reliable convergence on every benchmark. We therefore report AUROC scores only on the seven datasets where MAF trained stably. Each AUROC value is the mean over 10 independent runs. All experiments followed the protocol of our main experiments except for the MAF settings: a latent dimension of 128, a learning rate of 1e-4, a batch size of 100, and 10 stacked autoregressive layers. Our implementation was based on the reference code from \citet{stimper2023normflows}. As shown in Table~\ref{tab:maf_comparison}, NICE achieves higher AUROC than MAF on five of the seven tested datasets, which is consistent with our observation that a more advanced or expressive architecture does not necessarily yield better anomaly detection performance in the tabular domain.

\section{Typicality Test Performance}
\label{typicality_test_perf}
Anomaly detection via the typicality test builds on the observation that normal data may occupy a small region in input space while still lying in a high-probability typical set \citep{nalisnick2019detecting}. In this approach, we decide whether a given $\bold x\in\mathbb{R}^d$ is anomalous based on whether it belongs to the $(\epsilon, N)$-typical set $\mathcal{A}_\epsilon^N[p(\bold x)]$, where $p(\bold x)$ is a generative model trained on $\bold x_{train} \in \mathcal{D}_{train}$. Following \citet{zhang2021understanding}, the anomaly score $s$ is computed as in Equation \ref{eq:typ}.

\begin{align}
\label{eq:typ}
    \begin{split}
        s = \left|-\log p(\bold x) - \hat H_p\right| \:s.t\:\hat H_p=-\frac{1}{|\mathcal{D}_{train}|}\Sigma_{\bold x\in \mathcal{D}_{train}}\log p(\bold x)
    \end{split}
\end{align}

To assess whether the typicality test provides any benefit in the tabular domain, where counterintuitive phenomena rarely occur, we apply the typicality test using the final selected hyperparameters of NF-SLT in Table \ref{tab:auroc_fair} and report the results in Table~\ref{tab:typicality_test}.

\begin{table}[!h]
\caption{Performance comparison between SLT and the typicality test}
\centering
\label{tab:typicality_test}
\small
\setlength{\tabcolsep}{8pt}
\begin{tabular}{lccccc}
\toprule
Model & AUROC $ \uparrow$ & AUPRC $ \uparrow$ & Avg. Rank $ \downarrow$ & Top2 Cum. Ratio $ \uparrow$ & Fail Ratio $ \downarrow$ \\
\midrule
SLT             & 0.8575 & 0.6398 & 3.74 & 0.40 & 0.06 \\
Typicality Test & 0.8184 & 0.6270 & 5.83 & 0.26 & 0.30 \\
\bottomrule
\end{tabular}
\end{table}
Table \ref{tab:typicality_test} shows that the typicality test performs worse than the simple likelihood test on all aggregate performance metrics. For almost all datasets, the performance either changes little or degrades noticeably. The only exception is the “yeast” dataset, where NF-SLT attains an AUROC below 0.5, on which the AUROC of the typicality test improves by about 8\%. Although this dataset does not satisfy our definition of a counterintuitive case, it illustrates that the typicality test can occasionally outperform simple likelihood testing on specific datasets.

\section{Additional Experiment Results}
\label{add_experiment}
In this section, we report the additional experimental results. Table \ref{tab:auroc_unfair_full} records the AUROC and AUPRC under per-dataset optimal hyperparameter selection. Table~\ref{tab:auroc_whole_shallow} records the AUROC for the shallow model and Table~\ref{tab:auroc_whole_deep}  records the AUROC for the deep model. Table~\ref{tab:auprc_whole_shallow} shows the AUPRC for the shallow model and Table~\ref{tab:auprc_whole_deep} shows the AUPRC for the deep model. Table~\ref{tab:auroc_std_whole_shallow} shows standard deviation of the AUROC for the shallow model and Table~\ref{tab:auroc_std_whole_deep} shows standard deviation of the AUROC for the deep model.

\section{Limitations}
\label{limitation}
This study focuses on heterogeneous tabular datasets representative of general tabular data, as provided by ADBench. Our findings may not extend to high-dimensional tabular settings with strong inter-feature correlations, such as genomic or metagenomic data, where domain-specific preprocessing is typically required. In addition, our theoretical results assume that $P$ is a product of independent components or independent blocks. Arbitrary dependence is allowed within each block of $P$ and in $Q$, but strong dependence across blocks of $P$ is not covered. Moreover, the per-dimension condition in Corollary~\ref{cor:special_case_gap_indep} and the moment assumption in Corollary~\ref{cor:upper_bound_auroc} are sufficient conditions that may not hold for every dataset. Additionally, the datasets differ not only in dimensionality and feature correlation but also in domain, anomaly type, and intrinsic difficulty, and thus these factors may also contribute to the rarity of the counterintuitive phenomenon. While NF-SLT achieves competitive performance, we do not propose a normalizing flow architecture specifically tailored for tabular anomaly detection. Developing such architectures that better capture the unique characteristics of tabular data remains an important direction for future work.

\begin{table}[!h]
\caption{AUROC and AUPRC performance under per-dataset optimal hyperparameter selection.}
\centering
\label{tab:auroc_unfair_full}
\setlength{\tabcolsep}{6pt}
\begin{tabular}{lccccc}
\toprule
Method    & AUROC $ \uparrow$ & AUPRC $ \uparrow$ & Avg. Rank $ \downarrow$ & Top2 Cum. Ratio $ \uparrow$ & Fail Ratio $ \downarrow$ \\
\midrule
PCA       & 0.7752            & 0.5240            & 7.60                    & 0.09                        & 0.51                     \\
LOF       & 0.8447            & 0.5979            & 6.15                    & 0.23                        & 0.36                     \\
IF        & 0.8036            & 0.5099            & 6.85                    & 0.17                        & 0.36                     \\
OCSVM     & 0.6651            & 0.3895            & 10.51                   & 0.04                        & 0.77                     \\
COPOD     & 0.7471            & 0.4419            & 9.06                    & 0.06                        & 0.66                     \\
ECOD      & 0.7425            & 0.4530            & 9.45                    & 0.04                        & 0.70                     \\
DAGMM     & 0.6468            & 0.3473            & 12.19                   & 0.00                        & 0.96                     \\
DeepSVDD  & 0.8053            & 0.5840            & 6.72                    & 0.15                        & 0.28                     \\
GOAD      & 0.7210            & 0.5225            & 9.91                    & 0.06                        & 0.57                     \\
NeuTraLAD & 0.8391            & 0.6262            & 5.64                    & 0.28                        & 0.23                     \\
ICL       & 0.8492            & 0.6551            & 5.30                    & 0.19                        & 0.11                     \\
MCM       & 0.8166            & 0.5988            & 6.38                    & 0.21                        & 0.32                     \\
DRL       & 0.8596            & 0.6607            & 4.72                    & 0.26                        & \textbf{0.06}            \\
NF-SLT    & \textbf{0.8691}   & \textbf{0.6749}   & \textbf{3.89}           & \textbf{0.36}               & 0.11                     \\
\bottomrule
\end{tabular}
\end{table}

\begin{table}[!h]
\caption{AUROC performance of shallow models.}
\centering
\label{tab:auroc_whole_shallow}
\small
\setlength{\tabcolsep}{3.5pt}
\renewcommand{\arraystretch}{0.92}
\begin{tabular}{lcccccc}
\toprule
Dataset          & PCA    & LOF    & IF     & OCSVM  & COPOD  & ECOD   \\
\midrule
ALOI             & 0.5492 & 0.6971 & 0.5418 & 0.5038 & 0.5153 & 0.5304 \\
Cardiotocography & 0.8248 & 0.7207 & 0.8150 & 0.7349 & 0.6625 & 0.7856 \\
Hepatitis        & 0.8145 & 0.8249 & 0.7753 & 0.5905 & 0.7986 & 0.7473 \\
InternetAds      & 0.7012 & 0.8276 & 0.4242 & 0.3755 & 0.6764 & 0.6770 \\
Ionosphere       & 0.9037 & 0.9550 & 0.9141 & 0.5439 & 0.7975 & 0.7328 \\
Lymphography     & 0.9923 & 0.9843 & 0.9972 & 0.8761 & 0.9986 & 0.9974 \\
PageBlocks       & 0.9287 & 0.9667 & 0.9284 & 0.8506 & 0.8768 & 0.9149 \\
Pima             & 0.7051 & 0.7083 & 0.7337 & 0.5874 & 0.6564 & 0.5965 \\
SpamBase         & 0.8072 & 0.6613 & 0.8406 & 0.7484 & 0.6896 & 0.6578 \\
Stamps           & 0.9240 & 0.9022 & 0.9312 & 0.8522 & 0.9290 & 0.8717 \\
WBC              & 0.9932 & 0.9737 & 0.9968 & 0.9949 & 0.9959 & 0.9959 \\
WDBC             & 0.9955 & 0.9935 & 0.9951 & 0.8539 & 0.9957 & 0.9748 \\
WPBC             & 0.4838 & 0.5642 & 0.5136 & 0.5852 & 0.5116 & 0.4679 \\
Waveform         & 0.6460 & 0.7558 & 0.7354 & 0.6529 & 0.7312 & 0.6019 \\
Wilt             & 0.2911 & 0.8988 & 0.4776 & 0.9567 & 0.3440 & 0.3935 \\
annthyroid       & 0.7892 & 0.8794 & 0.9141 & 0.7075 & 0.7756 & 0.7888 \\
backdoor         & 0.6406 & 0.5726 & 0.7592 & 0.7588 & 0.7891 & 0.8459 \\
breastw          & 0.9877 & 0.9561 & 0.9946 & 0.9921 & 0.9938 & 0.9904 \\
campaign         & 0.7678 & 0.6713 & 0.7400 & 0.6838 & 0.7828 & 0.7695 \\
cardio           & 0.9648 & 0.9088 & 0.9508 & 0.7003 & 0.9216 & 0.9348 \\
celeba           & 0.7997 & 0.4437 & 0.7177 & 0.5923 & 0.7505 & 0.7569 \\
census           & 0.7073 & 0.5437 & 0.6270 & 0.7198 & 0.6741 & 0.6596 \\
cover            & 0.9574 & 0.9899 & 0.8596 & 0.6924 & 0.8840 & 0.9203 \\
donors           & 0.8880 & 0.9824 & 0.8992 & 0.8762 & 0.8151 & 0.8886 \\
fault            & 0.5456 & 0.6535 & 0.6593 & 0.4821 & 0.4557 & 0.4693 \\
fraud            & 0.9536 & 0.8706 & 0.9502 & 0.9288 & 0.9475 & 0.9496 \\
glass            & 0.6674 & 0.7700 & 0.8041 & 0.2892 & 0.7605 & 0.7136 \\
http             & 0.9994 & 0.9405 & 0.9927 & 0.7980 & 0.9916 & 0.9786 \\
landsat          & 0.4452 & 0.7469 & 0.6124 & 0.4463 & 0.4197 & 0.3670 \\
letter           & 0.5363 & 0.8344 & 0.6419 & 0.4499 & 0.5592 & 0.5718 \\
magicgamma       & 0.7026 & 0.8277 & 0.7749 & 0.8143 & 0.6814 & 0.6387 \\
mammography      & 0.8973 & 0.8508 & 0.8806 & 0.6572 & 0.9056 & 0.9064 \\
mnist            & 0.9060 & 0.7707 & 0.8680 & 0.7149 & 0.7763 & 0.7484 \\
musk             & 1.0000 & 1.0000 & 0.9817 & 0.1856 & 0.9458 & 0.9555 \\
optdigits        & 0.5740 & 0.9183 & 0.8436 & 0.7722 & 0.6819 & 0.6037 \\
pendigits        & 0.9500 & 0.9926 & 0.9706 & 0.6706 & 0.9051 & 0.9274 \\
satellite        & 0.6761 & 0.8473 & 0.8031 & 0.8220 & 0.6339 & 0.5836 \\
satimage-2       & 0.9758 & 0.9958 & 0.9935 & 0.5883 & 0.9745 & 0.9649 \\
shuttle          & 0.9931 & 0.9979 & 0.9966 & 0.9869 & 0.9945 & 0.9929 \\
skin             & 0.6027 & 0.9228 & 0.8899 & 0.4353 & 0.4702 & 0.4881 \\
smtp             & 0.8659 & 0.8541 & 0.9074 & 0.2851 & 0.9117 & 0.8801 \\
speech           & 0.4687 & 0.4952 & 0.4655 & 0.5275 & 0.4901 & 0.4687 \\
thyroid          & 0.9811 & 0.9344 & 0.9894 & 0.6363 & 0.9402 & 0.9777 \\
vertebral        & 0.4334 & 0.4128 & 0.4399 & 0.3849 & 0.3515 & 0.4337 \\
vowels           & 0.6664 & 0.9455 & 0.7857 & 0.3771 & 0.4955 & 0.5919 \\
wine             & 0.9332 & 0.9570 & 0.9175 & 0.8058 & 0.8725 & 0.7415 \\
yeast            & 0.4214 & 0.4731 & 0.4171 & 0.3535 & 0.3828 & 0.4447 \\
\midrule
Avg AUROC        & 0.7715 & 0.8169 & 0.8014 & 0.6562 & 0.7471 & 0.7425 \\
Avg.Rank         & 7.13   & 6.11   & 6.23   & 10.34  & 8.40   & 8.74   \\
Top1 Ratio       & 0.09   & 0.06   & 0.09   & 0.04   & 0.02   & 0.02   \\
Top2 Ratio       & 0.06   & 0.11   & 0.11   & 0.02   & 0.09   & 0.04   \\
Top2 Cum. Ratio & 0.15   & 0.17   & 0.19   & 0.06   & 0.11   & 0.06   \\
Fail Ratio       & 0.45   & 0.26   & 0.21   & 0.77   & 0.57   & 0.68   \\
\bottomrule
\end{tabular}
\end{table}

\begin{table}[!h]
\caption{AUROC performance of deep models.}
\centering
\label{tab:auroc_whole_deep}
\small
\setlength{\tabcolsep}{3.5pt}
\renewcommand{\arraystretch}{0.92}
\resizebox{\linewidth}{!}{
\begin{tabular}{lcccccccc}
\toprule
Dataset           & DAGMM  & DeepSVDD & GOAD   & NeuTraLAD & ICL    & MCM    & DRL    & NF-SLT \\
\midrule
ALOI              & 0.5024 & 0.5653   & 0.5319 & 0.5700    & 0.5892 & 0.4831 & 0.5858 & 0.5479 \\
Cardiotocography  & 0.6067 & 0.6759   & 0.3689 & 0.7626    & 0.6221 & 0.5713 & 0.6862 & 0.7558 \\
Hepatitis         & 0.6167 & 0.6844   & 0.6054 & 0.6516    & 0.6070 & 0.7985 & 0.7735 & 0.7475 \\
InternetAds       & 0.5590 & 0.6013   & 0.7356 & 0.7112    & 0.8815 & 0.7973 & 0.7603 & 0.8746 \\
Ionosphere        & 0.6580 & 0.9597   & 0.9459 & 0.9708    & 0.9556 & 0.9574 & 0.9660 & 0.9581 \\
Lymphography      & 0.8643 & 0.9756   & 0.9847 & 0.9779    & 0.9624 & 0.9934 & 0.9847 & 0.9746 \\
PageBlocks        & 0.7966 & 0.8866   & 0.9345 & 0.9803    & 0.9691 & 0.8828 & 0.9556 & 0.9656 \\
Pima              & 0.5759 & 0.5899   & 0.4432 & 0.5414    & 0.6547 & 0.6891 & 0.7036 & 0.7219 \\
SpamBase          & 0.5564 & 0.5074   & 0.3366 & 0.6406    & 0.8037 & 0.7379 & 0.8416 & 0.7724 \\
Stamps            & 0.7417 & 0.8361   & 0.6627 & 0.8836    & 0.8122 & 0.8720 & 0.9037 & 0.9327 \\
WBC               & 0.8804 & 0.9873   & 0.5576 & 0.8944    & 0.9696 & 0.9254 & 0.9751 & 0.9726 \\
WDBC              & 0.8121 & 0.9932   & 0.7814 & 0.9970    & 0.9856 & 0.9680 & 0.9827 & 0.9785 \\
WPBC              & 0.4724 & 0.4579   & 0.3565 & 0.4216    & 0.4328 & 0.5068 & 0.4876 & 0.5051 \\
Waveform          & 0.5135 & 0.6609   & 0.4180 & 0.8186    & 0.6172 & 0.6641 & 0.7146 & 0.7357 \\
Wilt              & 0.4790 & 0.6008   & 0.7298 & 0.7527    & 0.8043 & 0.8655 & 0.8486 & 0.9066 \\
annthyroid        & 0.8361 & 0.7689   & 0.5206 & 0.6685    & 0.8732 & 0.8640 & 0.7896 & 0.9181 \\
backdoor          & 0.4535 & 0.9473   & 0.2805 & 0.9456    & 0.9767 & 0.9222 & 0.9541 & 0.9343 \\
breastw           & 0.8740 & 0.9825   & 0.8402 & 0.9505    & 0.9835 & 0.9900 & 0.9771 & 0.9842 \\
campaign          & 0.5948 & 0.7929   & 0.2837 & 0.7814    & 0.6933 & 0.7987 & 0.7525 & 0.8059 \\
cardio            & 0.6091 & 0.9139   & 0.5964 & 0.7367    & 0.8456 & 0.8093 & 0.8888 & 0.9174 \\
celeba            & 0.5308 & 0.3604   & 0.5204 & 0.6888    & 0.4595 & 0.6400 & 0.7811 & 0.7340 \\
census            & 0.4619 & 0.6421   & 0.4140 & 0.4991    & 0.3917 & 0.7478 & 0.5646 & 0.7186 \\
cover             & 0.7412 & 0.6819   & 0.2571 & 0.9096    & 0.9859 & 0.8949 & 0.6940 & 0.9658 \\
donors            & 0.6243 & 0.9970   & 0.4472 & 0.9985    & 0.9630 & 0.9987 & 0.9439 & 0.9990 \\
fault             & 0.5491 & 0.7146   & 0.7028 & 0.7680    & 0.7991 & 0.5955 & 0.7564 & 0.7518 \\
fraud             & 0.8141 & 0.9517   & 0.4119 & 0.9448    & 0.9524 & 0.8488 & 0.9533 & 0.9564 \\
glass             & 0.5561 & 0.4046   & 0.4495 & 0.8305    & 0.9027 & 0.8783 & 0.9064 & 0.8867 \\
http              & 0.9966 & 0.9996   & 0.7812 & 0.9380    & 0.9999 & 0.9925 & 0.9999 & 1.0000 \\
landsat           & 0.4874 & 0.6759   & 0.6427 & 0.7582    & 0.7075 & 0.5724 & 0.7842 & 0.6543 \\
letter            & 0.5177 & 0.7565   & 0.6528 & 0.8227    & 0.9312 & 0.3674 & 0.9040 & 0.9258 \\
magicgamma        & 0.6099 & 0.7267   & 0.5677 & 0.7317    & 0.7453 & 0.8143 & 0.8307 & 0.8863 \\
mammography       & 0.5903 & 0.7263   & 0.8477 & 0.6377    & 0.8108 & 0.8321 & 0.8818 & 0.8761 \\
mnist             & 0.6168 & 0.5940   & 0.7840 & 0.9700    & 0.8931 & 0.8631 & 0.8095 & 0.9015 \\
musk              & 0.7128 & 0.9868   & 0.9999 & 1.0000    & 1.0000 & 0.9972 & 1.0000 & 1.0000 \\
optdigits         & 0.4249 & 0.7386   & 0.2742 & 0.8566    & 0.8192 & 0.9279 & 0.9142 & 0.9205 \\
pendigits         & 0.6830 & 0.9510   & 0.2513 & 0.9701    & 0.9619 & 0.9913 & 0.9689 & 0.9930 \\
satellite         & 0.6913 & 0.8275   & 0.7789 & 0.8186    & 0.8190 & 0.8008 & 0.8902 & 0.8276 \\
satimage-2        & 0.9234 & 0.9971   & 0.9960 & 0.9987    & 0.9974 & 0.9913 & 0.9957 & 0.9966 \\
shuttle           & 0.6404 & 0.9980   & 0.3801 & 0.9994    & 0.9998 & 0.9979 & 0.9993 & 0.9984 \\
skin              & 0.7063 & 0.7270   & 0.8856 & 0.9128    & 0.8954 & 0.8175 & 0.8827 & 0.9675 \\
smtp              & 0.7976 & 0.8855   & 0.6588 & 0.8583    & 0.8805 & 0.9268 & 0.9124 & 0.9201 \\
speech            & 0.5008 & 0.4688   & 0.5138 & 0.5123    & 0.5888 & 0.4110 & 0.5495 & 0.5795 \\
thyroid           & 0.9408 & 0.9256   & 0.4299 & 0.9619    & 0.9653 & 0.9584 & 0.9523 & 0.9840 \\
vertebral         & 0.4778 & 0.4794   & 0.5692 & 0.6381    & 0.5621 & 0.3544 & 0.5347 & 0.5483 \\
vowels            & 0.5295 & 0.9377   & 0.9316 & 0.7861    & 0.9914 & 0.6873 & 0.9775 & 0.9852 \\
wine              & 0.7212 & 0.9533   & 0.9430 & 0.9680    & 0.9575 & 0.4924 & 0.9415 & 0.9497 \\
yeast             & 0.5442 & 0.6315   & 0.5979 & 0.5438    & 0.5574 & 0.4654 & 0.4431 & 0.4652 \\
\midrule
Avg AUROC         & 0.6467 & 0.7687   & 0.6086 & 0.8081    & 0.8208 & 0.7864 & 0.8363 & 0.8575 \\
Avg.Rank          & 11.45  & 7.74     & 10.60  & 6.17      & 5.70   & 7.34   & 5.04   & 3.74   \\
Top1 Ratio        & 0.00   & 0.02     & 0.00   & 0.19      & 0.17   & 0.06   & 0.11   & 0.21   \\
Top2 Ratio        & 0.00   & 0.00     & 0.04   & 0.06      & 0.13   & 0.04   & 0.11   & 0.19   \\
Top2 Cum. Ratio & 0.00   & 0.02     & 0.04   & 0.26      & 0.30   & 0.11   & 0.21   & 0.40   \\
Fail Ratio        & 0.89   & 0.49     & 0.62   & 0.30      & 0.26   & 0.36   & 0.09   & 0.06   \\
\bottomrule
\end{tabular}
}
\end{table}

\begin{table}[!t]
\caption{AUPRC performance of shallow models.}
\centering
\label{tab:auprc_whole_shallow}
\small
\setlength{\tabcolsep}{3.5pt}
\renewcommand{\arraystretch}{0.92}
{
\begin{tabular}{lcccccc}
\toprule
Dataset          & PCA    & LOF    & IF     & OCSVM  & COPOD  & ECOD   \\
\midrule
ALOI             & 0.0711 & 0.1334 & 0.0649 & 0.0592 & 0.0607 & 0.0637 \\
Cardiotocography & 0.7324 & 0.6321 & 0.7085 & 0.5931 & 0.5516 & 0.6598 \\
Hepatitis        & 0.6398 & 0.6612 & 0.5448 & 0.4062 & 0.5783 & 0.4707 \\
InternetAds      & 0.5689 & 0.7937 & 0.2782 & 0.2448 & 0.6206 & 0.6216 \\
Ionosphere       & 0.9223 & 0.9609 & 0.9252 & 0.7024 & 0.7988 & 0.7743 \\
Lymphography     & 0.8923 & 0.8121 & 0.9741 & 0.6735 & 0.9830 & 0.9717 \\
PageBlocks       & 0.7340 & 0.8497 & 0.6965 & 0.7058 & 0.5254 & 0.6602 \\
Pima             & 0.6961 & 0.6776 & 0.7268 & 0.6157 & 0.6852 & 0.6323 \\
SpamBase         & 0.8427 & 0.7184 & 0.8696 & 0.8003 & 0.7039 & 0.6819 \\
Stamps           & 0.5743 & 0.5235 & 0.5807 & 0.4142 & 0.5632 & 0.4716 \\
WBC              & 0.9427 & 0.7288 & 0.9719 & 0.9510 & 0.9556 & 0.9556 \\
WDBC             & 0.9110 & 0.8804 & 0.9002 & 0.4891 & 0.9283 & 0.6830 \\
WPBC             & 0.3681 & 0.4122 & 0.3839 & 0.4560 & 0.3742 & 0.3492 \\
Waveform         & 0.0861 & 0.2901 & 0.1132 & 0.1099 & 0.1047 & 0.0771 \\
Wilt             & 0.0665 & 0.4182 & 0.0876 & 0.5389 & 0.0708 & 0.0799 \\
annthyroid       & 0.5127 & 0.5591 & 0.6231 & 0.3465 & 0.2944 & 0.4058 \\
backdoor         & 0.0783 & 0.0989 & 0.0942 & 0.1223 & 0.1275 & 0.1677 \\
breastw          & 0.9832 & 0.9181 & 0.9944 & 0.9919 & 0.9936 & 0.9906 \\
campaign         & 0.4878 & 0.3429 & 0.4645 & 0.3628 & 0.5141 & 0.4994 \\
cardio           & 0.8454 & 0.6653 & 0.7977 & 0.3625 & 0.7103 & 0.7070 \\
celeba           & 0.2083 & 0.0358 & 0.1306 & 0.0524 & 0.1656 & 0.1697 \\
census           & 0.2012 & 0.1174 & 0.1403 & 0.2093 & 0.1610 & 0.1547 \\
cover            & 0.1925 & 0.8391 & 0.0827 & 0.0534 & 0.1214 & 0.1865 \\
donors           & 0.3618 & 0.7986 & 0.4125 & 0.3834 & 0.3353 & 0.4142 \\
fault            & 0.5726 & 0.6218 & 0.6542 & 0.5049 & 0.4751 & 0.4906 \\
fraud            & 0.2661 & 0.0287 & 0.2209 & 0.3929 & 0.3669 & 0.3339 \\
glass            & 0.1791 & 0.2206 & 0.2177 & 0.0633 & 0.2001 & 0.2535 \\
http             & 0.9156 & 0.1117 & 0.5033 & 0.4461 & 0.4636 & 0.2534 \\
landsat          & 0.3169 & 0.6913 & 0.4300 & 0.3179 & 0.2977 & 0.2799 \\
letter           & 0.1466 & 0.4482 & 0.1661 & 0.1108 & 0.1284 & 0.1432 \\
magicgamma       & 0.7480 & 0.8612 & 0.8051 & 0.8429 & 0.7225 & 0.6807 \\
mammography      & 0.4513 & 0.3563 & 0.3854 & 0.0753 & 0.5459 & 0.5516 \\
mnist            & 0.6753 & 0.3479 & 0.5455 & 0.3116 & 0.3560 & 0.3045 \\
musk             & 1.0000 & 0.9995 & 0.7037 & 0.2154 & 0.4869 & 0.6140 \\
optdigits        & 0.0591 & 0.3212 & 0.1714 & 0.1326 & 0.0807 & 0.0650 \\
pendigits        & 0.4162 & 0.8520 & 0.5394 & 0.0763 & 0.2930 & 0.3924 \\
satellite        & 0.7711 & 0.8815 & 0.8455 & 0.8505 & 0.6957 & 0.6596 \\
satimage-2       & 0.8823 & 0.9305 & 0.9371 & 0.1132 & 0.8349 & 0.7419 \\
shuttle          & 0.9615 & 0.9784 & 0.9863 & 0.9732 & 0.9773 & 0.9474 \\
skin             & 0.3585 & 0.7164 & 0.6368 & 0.3285 & 0.2958 & 0.3030 \\
smtp             & 0.4636 & 0.4138 & 0.0088 & 0.0007 & 0.0081 & 0.5760 \\
speech           & 0.0361 & 0.0440 & 0.0389 & 0.0404 & 0.0377 & 0.0385 \\
thyroid          & 0.7791 & 0.6215 & 0.8265 & 0.3331 & 0.3077 & 0.6362 \\
vertebral        & 0.1922 & 0.2003 & 0.1971 & 0.1892 & 0.1680 & 0.2015 \\
vowels           & 0.1949 & 0.6011 & 0.2576 & 0.0839 & 0.0660 & 0.1421 \\
wine             & 0.7044 & 0.7307 & 0.6605 & 0.5416 & 0.5629 & 0.3364 \\
yeast            & 0.4707 & 0.5043 & 0.4787 & 0.4277 & 0.4696 & 0.4998 \\
\midrule
Avg AUPRC        & 0.5209 & 0.5606 & 0.5060 & 0.3833 & 0.4419 & 0.4530 \\
Avg.Rank         & 7.45   & 6.34   & 6.66   & 10.11  & 9.00   & 8.98   \\
Top1 Ratio       & 0.06   & 0.11   & 0.11   & 0.04   & 0.02   & 0.02   \\
Top2 Ratio       & 0.04   & 0.09   & 0.06   & 0.02   & 0.09   & 0.04   \\
Top2 Cum. Ratio & 0.11   & 0.19   & 0.17   & 0.06   & 0.11   & 0.06   \\
Fail Ratio       & 0.45   & 0.28   & 0.32   & 0.74   & 0.62   & 0.62   \\
\bottomrule
\end{tabular}
}
\end{table}

\begin{table}[!h]
\caption{AUPRC performance of deep models.}
\centering
\label{tab:auprc_whole_deep}
\small
\setlength{\tabcolsep}{3.5pt}
\renewcommand{\arraystretch}{0.92}
\resizebox{\linewidth}{!}{
\begin{tabular}{lcccccccc}
\toprule
Dataset          & DAGMM  & DeepSVDD & GOAD   & NeuTraLAD & ICL    & MCM    & DRL    & NF-SLT \\
\midrule
ALOI             & 0.0589 & 0.1186   & 0.0749 & 0.0937    & 0.1058 & 0.0578 & 0.1163 & 0.0831 \\
Cardiotocography & 0.4520 & 0.6307   & 0.3829 & 0.6679    & 0.5864 & 0.5245 & 0.6272 & 0.6742 \\
Hepatitis        & 0.4622 & 0.5382   & 0.4831 & 0.3848    & 0.3773 & 0.5708 & 0.5664 & 0.5310 \\
InternetAds      & 0.3792 & 0.4077   & 0.5967 & 0.4142    & 0.8345 & 0.7551 & 0.6564 & 0.8399 \\
Ionosphere       & 0.7007 & 0.9702   & 0.9615 & 0.9761    & 0.9673 & 0.9705 & 0.9731 & 0.9676 \\
Lymphography     & 0.5683 & 0.7607   & 0.8225 & 0.7839    & 0.7166 & 0.9187 & 0.8244 & 0.7691 \\
PageBlocks       & 0.5661 & 0.7479   & 0.8474 & 0.9072    & 0.8890 & 0.6263 & 0.8532 & 0.8644 \\
Pima             & 0.5831 & 0.6349   & 0.5237 & 0.5611    & 0.6805 & 0.6922 & 0.6713 & 0.7117 \\
SpamBase         & 0.6143 & 0.6307   & 0.5289 & 0.6962    & 0.8289 & 0.7449 & 0.8710 & 0.8063 \\
Stamps           & 0.3846 & 0.4953   & 0.3601 & 0.5896    & 0.4683 & 0.4703 & 0.5356 & 0.6273 \\
WBC              & 0.6306 & 0.8947   & 0.1882 & 0.5681    & 0.7894 & 0.7778 & 0.8169 & 0.7923 \\
WDBC             & 0.3081 & 0.8617   & 0.4889 & 0.9632    & 0.8215 & 0.7424 & 0.7713 & 0.6770 \\
WPBC             & 0.3795 & 0.3655   & 0.3177 & 0.3349    & 0.3449 & 0.4092 & 0.3802 & 0.3727 \\
Waveform         & 0.0607 & 0.1041   & 0.0460 & 0.2014    & 0.1301 & 0.1196 & 0.3720 & 0.1743 \\
Wilt             & 0.1019 & 0.1150   & 0.2689 & 0.2147    & 0.3039 & 0.3129 & 0.4195 & 0.3903 \\
annthyroid       & 0.5478 & 0.4005   & 0.3897 & 0.3503    & 0.5520 & 0.5679 & 0.4093 & 0.6227 \\
backdoor         & 0.0461 & 0.6203   & 0.0527 & 0.9027    & 0.8968 & 0.3615 & 0.8937 & 0.5688 \\
breastw          & 0.8930 & 0.9821   & 0.8605 & 0.9442    & 0.9790 & 0.9887 & 0.9692 & 0.9797 \\
campaign         & 0.2771 & 0.5130   & 0.1368 & 0.5227    & 0.4218 & 0.5070 & 0.4510 & 0.5213 \\
cardio           & 0.2816 & 0.7979   & 0.4940 & 0.4121    & 0.6640 & 0.6778 & 0.7576 & 0.7557 \\
celeba           & 0.0548 & 0.0317   & 0.0426 & 0.0887    & 0.0400 & 0.0663 & 0.1256 & 0.0955 \\
census           & 0.1036 & 0.1589   & 0.0909 & 0.1163    & 0.0898 & 0.2249 & 0.1368 & 0.2084 \\
cover            & 0.0998 & 0.0483   & 0.0120 & 0.4044    & 0.7338 & 0.4159 & 0.0401 & 0.7204 \\
donors           & 0.1729 & 0.9754   & 0.0994 & 0.9853    & 0.7310 & 0.9783 & 0.5842 & 0.9879 \\
fault            & 0.5743 & 0.7123   & 0.7061 & 0.7222    & 0.7884 & 0.6307 & 0.7388 & 0.7412 \\
fraud            & 0.0515 & 0.7615   & 0.1315 & 0.6194    & 0.6445 & 0.6662 & 0.6138 & 0.7082 \\
glass            & 0.1362 & 0.1954   & 0.1929 & 0.2957    & 0.4677 & 0.4202 & 0.3840 & 0.3655 \\
http             & 0.7535 & 0.9956   & 0.3368 & 0.8768    & 0.9894 & 0.5124 & 0.9846 & 0.9917 \\
landsat          & 0.3407 & 0.5269   & 0.3988 & 0.5465    & 0.6609 & 0.3967 & 0.5849 & 0.4795 \\
letter           & 0.1360 & 0.2934   & 0.2592 & 0.3964    & 0.6395 & 0.0893 & 0.6073 & 0.6673 \\
magicgamma       & 0.6468 & 0.7545   & 0.6289 & 0.7665    & 0.8042 & 0.8549 & 0.8631 & 0.9094 \\
mammography      & 0.0888 & 0.0964   & 0.3818 & 0.1382    & 0.3841 & 0.3291 & 0.5107 & 0.3885 \\
mnist            & 0.2500 & 0.2424   & 0.5855 & 0.8695    & 0.6798 & 0.5856 & 0.5094 & 0.7080 \\
musk             & 0.3674 & 0.7550   & 0.9980 & 1.0000    & 1.0000 & 0.9918 & 1.0000 & 1.0000 \\
optdigits        & 0.0485 & 0.1621   & 0.0360 & 0.2009    & 0.1563 & 0.2755 & 0.3410 & 0.2725 \\
pendigits        & 0.1507 & 0.3583   & 0.0283 & 0.5856    & 0.6152 & 0.8148 & 0.6390 & 0.8073 \\
satellite        & 0.7369 & 0.8576   & 0.7973 & 0.8667    & 0.8632 & 0.8446 & 0.8943 & 0.8658 \\
satimage-2       & 0.4356 & 0.9706   & 0.9641 & 0.9735    & 0.9586 & 0.8117 & 0.9284 & 0.8984 \\
shuttle          & 0.2613 & 0.9818   & 0.2937 & 0.9944    & 0.9972 & 0.9608 & 0.9868 & 0.9695 \\
skin             & 0.5066 & 0.4875   & 0.6741 & 0.7835    & 0.6725 & 0.5557 & 0.7323 & 0.9226 \\
smtp             & 0.1050 & 0.6110   & 0.2985 & 0.4615    & 0.5455 & 0.3936 & 0.6302 & 0.4675 \\
speech           & 0.0409 & 0.0312   & 0.0395 & 0.0362    & 0.0507 & 0.0316 & 0.0432 & 0.0529 \\
thyroid          & 0.6545 & 0.6129   & 0.2536 & 0.6164    & 0.5694 & 0.6664 & 0.6526 & 0.7108 \\
vertebral        & 0.2437 & 0.2224   & 0.2879 & 0.3645    & 0.3090 & 0.1731 & 0.2733 & 0.2617 \\
vowels           & 0.0798 & 0.5534   & 0.6197 & 0.1764    & 0.9254 & 0.1425 & 0.8255 & 0.8755 \\
wine             & 0.4028 & 0.7368   & 0.7695 & 0.8384    & 0.7650 & 0.1736 & 0.7375 & 0.7601 \\
yeast            & 0.5627 & 0.5992   & 0.5863 & 0.5513    & 0.5573 & 0.4973 & 0.4796 & 0.5052 \\
\midrule
Avg AUPRC        & 0.3468 & 0.5388   & 0.4114 & 0.5694    & 0.6170 & 0.5383 & 0.6124 & 0.6398 \\
Avg.Rank         & 11.28  & 7.23     & 9.98   & 5.96      & 5.72   & 7.17   & 4.87   & 4.15   \\
Top1 Ratio       & 0.00   & 0.06     & 0.00   & 0.21      & 0.09   & 0.02   & 0.11   & 0.17   \\
Top2 Ratio       & 0.00   & 0.09     & 0.04   & 0.09      & 0.17   & 0.04   & 0.06   & 0.17   \\
Top2 Cum. Ratio & 0.00   & 0.15     & 0.04   & 0.30      & 0.26   & 0.06   & 0.17   & 0.34   \\
Fail Ratio       & 0.85   & 0.40     & 0.60   & 0.32      & 0.26   & 0.38   & 0.11   & 0.06   \\
\bottomrule
\end{tabular}
}
\end{table}

\begin{table}[!h]
\caption{Standard deviation of the AUROC for shallow models.}
\centering
\label{tab:auroc_std_whole_shallow}
\small
\setlength{\tabcolsep}{3.5pt}
\renewcommand{\arraystretch}{0.92}
{
\begin{tabular}{lcccccc}
\toprule
Dataset          & PCA    & LOF    & IF     & OCSVM  & COPOD  & ECOD   \\
\midrule
ALOI             & 0.0012 & 0.0031 & 0.0025 & 0.0007 & 0.0009 & 0.0008 \\
Cardiotocography & 0.0055 & 0.0208 & 0.0114 & 0.0106 & 0.0046 & 0.0051 \\
Hepatitis        & 0.0137 & 0.0284 & 0.0341 & 0.0431 & 0.0275 & 0.0373 \\
InternetAds      & 0.0334 & 0.0070 & 0.0526 & 0.0094 & 0.0045 & 0.0044 \\
Ionosphere       & 0.0153 & 0.0106 & 0.0162 & 0.0667 & 0.0217 & 0.0161 \\
Lymphography     & 0.0061 & 0.0072 & 0.0036 & 0.1011 & 0.0028 & 0.0031 \\
PageBlocks       & 0.0032 & 0.0037 & 0.0030 & 0.0346 & 0.0032 & 0.0026 \\
Pima             & 0.0143 & 0.0119 & 0.0147 & 0.0381 & 0.0151 & 0.0163 \\
SpamBase         & 0.0079 & 0.0183 & 0.0063 & 0.0126 & 0.0062 & 0.0060 \\
Stamps           & 0.0251 & 0.0237 & 0.0210 & 0.0500 & 0.0159 & 0.0175 \\
WBC              & 0.0034 & 0.0074 & 0.0015 & 0.0014 & 0.0023 & 0.0023 \\
WDBC             & 0.0038 & 0.0036 & 0.0035 & 0.1023 & 0.0020 & 0.0042 \\
WPBC             & 0.0319 & 0.0283 & 0.0371 & 0.0563 & 0.0332 & 0.0323 \\
Waveform         & 0.0043 & 0.0116 & 0.0082 & 0.1012 & 0.0057 & 0.0067 \\
Wilt             & 0.0508 & 0.0070 & 0.0154 & 0.0069 & 0.0029 & 0.0033 \\
annthyroid       & 0.0192 & 0.0221 & 0.0078 & 0.0752 & 0.0024 & 0.0025 \\
backdoor         & 0.0031 & 0.0268 & 0.0115 & 0.0013 & 0.0009 & 0.0009 \\
breastw          & 0.0045 & 0.0080 & 0.0016 & 0.0021 & 0.0016 & 0.0019 \\
campaign         & 0.0022 & 0.0058 & 0.0105 & 0.0029 & 0.0008 & 0.0008 \\
cardio           & 0.0021 & 0.0083 & 0.0063 & 0.0175 & 0.0045 & 0.0041 \\
celeba           & 0.0004 & 0.0040 & 0.0103 & 0.0008 & 0.0004 & 0.0005 \\
census           & 0.0009 & 0.0025 & 0.0122 & 0.0009 & 0.0007 & 0.0007 \\
cover            & 0.0003 & 0.0006 & 0.0107 & 0.0030 & 0.0004 & 0.0002 \\
donors           & 0.0088 & 0.0006 & 0.0110 & 0.0182 & 0.0003 & 0.0002 \\
fault            & 0.0186 & 0.0243 & 0.0151 & 0.0147 & 0.0059 & 0.0052 \\
fraud            & 0.0003 & 0.0230 & 0.0013 & 0.0009 & 0.0001 & 0.0001 \\
glass            & 0.0466 & 0.0287 & 0.0263 & 0.1472 & 0.0228 & 0.0248 \\
http             & 0.0000 & 0.0060 & 0.0012 & 0.3882 & 0.0002 & 0.0003 \\
landsat          & 0.0067 & 0.0056 & 0.0121 & 0.0139 & 0.0061 & 0.0049 \\
letter           & 0.0062 & 0.0082 & 0.0148 & 0.0167 & 0.0106 & 0.0096 \\
magicgamma       & 0.0019 & 0.0026 & 0.0054 & 0.0015 & 0.0025 & 0.0025 \\
mammography      & 0.0046 & 0.0143 & 0.0025 & 0.0134 & 0.0012 & 0.0012 \\
mnist            & 0.0027 & 0.0160 & 0.0051 & 0.0109 & 0.0016 & 0.0019 \\
musk             & 0.0000 & 0.0000 & 0.0126 & 0.0000 & 0.0026 & 0.0021 \\
optdigits        & 0.0089 & 0.0077 & 0.0188 & 0.0124 & 0.0028 & 0.0028 \\
pendigits        & 0.0025 & 0.0015 & 0.0037 & 0.0137 & 0.0029 & 0.0023 \\
satellite        & 0.0022 & 0.0051 & 0.0082 & 0.0296 & 0.0033 & 0.0026 \\
satimage-2       & 0.0004 & 0.0008 & 0.0007 & 0.1826 & 0.0006 & 0.0009 \\
shuttle          & 0.0006 & 0.0005 & 0.0002 & 0.0013 & 0.0001 & 0.0001 \\
skin             & 0.0006 & 0.0048 & 0.0029 & 0.2456 & 0.0007 & 0.0008 \\
smtp             & 0.0421 & 0.0454 & 0.0019 & 0.0136 & 0.0005 & 0.0003 \\
speech           & 0.0042 & 0.0036 & 0.0075 & 0.0368 & 0.0052 & 0.0044 \\
thyroid          & 0.0012 & 0.0130 & 0.0015 & 0.2706 & 0.0023 & 0.0012 \\
vertebral        & 0.0315 & 0.0250 & 0.0453 & 0.0832 & 0.0189 & 0.0209 \\
vowels           & 0.0268 & 0.0039 & 0.0174 & 0.1791 & 0.0068 & 0.0073 \\
wine             & 0.0196 & 0.0184 & 0.0336 & 0.1725 & 0.0100 & 0.0095 \\
yeast            & 0.0177 & 0.0130 & 0.0125 & 0.0116 & 0.0098 & 0.0102 \\
\bottomrule
\end{tabular}
}
\end{table}

\begin{table}[!h]
\caption{Standard deviation of the AUROC for deep Models}
\centering
\label{tab:auroc_std_whole_deep}
\small
\setlength{\tabcolsep}{3.5pt}
\renewcommand{\arraystretch}{0.92}
\resizebox{\linewidth}{!}{
\begin{tabular}{lcccccccc}
\toprule
Dataset          & DAGMM  & DeepSVDD & GOAD   & NeuTraLAD & ICL    & MCM    & DRL    & NF-SLT \\
\midrule
ALOI             & 0.0182 & 0.0085   & 0.0030 & 0.0096    & 0.0158 & 0.0063 & 0.0131 & 0.0017 \\
Cardiotocography & 0.0939 & 0.0109   & 0.0144 & 0.0157    & 0.0314 & 0.0141 & 0.0336 & 0.0169 \\
Hepatitis        & 0.0835 & 0.0753   & 0.0582 & 0.0815    & 0.0963 & 0.0200 & 0.0664 & 0.0787 \\
InternetAds      & 0.0430 & 0.0213   & 0.0442 & 0.0179    & 0.0150 & 0.0013 & 0.0181 & 0.0114 \\
Ionosphere       & 0.0788 & 0.0075   & 0.0164 & 0.0054    & 0.0082 & 0.0011 & 0.0072 & 0.0085 \\
Lymphography     & 0.1437 & 0.0111   & 0.0087 & 0.0210    & 0.0150 & 0.0020 & 0.0099 & 0.0120 \\
PageBlocks       & 0.0454 & 0.0105   & 0.0123 & 0.0044    & 0.0049 & 0.0056 & 0.0204 & 0.0027 \\
Pima             & 0.0464 & 0.0191   & 0.0349 & 0.0311    & 0.0203 & 0.0131 & 0.0136 & 0.0163 \\
SpamBase         & 0.0646 & 0.0191   & 0.0228 & 0.0157    & 0.0232 & 0.0055 & 0.0172 & 0.0200 \\
Stamps           & 0.1480 & 0.0776   & 0.1259 & 0.0495    & 0.0875 & 0.0089 & 0.0225 & 0.0243 \\
WBC              & 0.0968 & 0.0047   & 0.1129 & 0.0460    & 0.0142 & 0.0091 & 0.0106 & 0.0118 \\
WDBC             & 0.1513 & 0.0031   & 0.1144 & 0.0027    & 0.0095 & 0.0095 & 0.0123 & 0.0107 \\
WPBC             & 0.0589 & 0.0314   & 0.0347 & 0.0475    & 0.0217 & 0.0144 & 0.0508 & 0.0270 \\
Waveform         & 0.0560 & 0.0173   & 0.0395 & 0.0061    & 0.0186 & 0.0213 & 0.0814 & 0.0224 \\
Wilt             & 0.1216 & 0.0185   & 0.0418 & 0.1231    & 0.0649 & 0.0185 & 0.0265 & 0.0140 \\
annthyroid       & 0.1246 & 0.0202   & 0.0940 & 0.0251    & 0.0332 & 0.0090 & 0.0251 & 0.0097 \\
backdoor         & 0.1457 & 0.0035   & 0.0149 & 0.0050    & 0.0044 & 0.0070 & 0.0142 & 0.0070 \\
breastw          & 0.0886 & 0.0049   & 0.1070 & 0.0260    & 0.0041 & 0.0019 & 0.0079 & 0.0033 \\
campaign         & 0.0401 & 0.0283   & 0.0628 & 0.0134    & 0.0157 & 0.0028 & 0.0288 & 0.0024 \\
cardio           & 0.1023 & 0.0114   & 0.0385 & 0.0518    & 0.0326 & 0.0123 & 0.0221 & 0.0102 \\
celeba           & 0.0658 & 0.0786   & 0.0742 & 0.0140    & 0.0869 & 0.0349 & 0.0344 & 0.0284 \\
census           & 0.0499 & 0.0463   & 0.0455 & 0.0064    & 0.0146 & 0.0034 & 0.0882 & 0.0094 \\
cover            & 0.1598 & 0.0197   & 0.0697 & 0.1051    & 0.0051 & 0.0187 & 0.0286 & 0.0067 \\
donors           & 0.1132 & 0.0025   & 0.0824 & 0.0023    & 0.0398 & 0.0004 & 0.0473 & 0.0010 \\
fault            & 0.0606 & 0.0136   & 0.0299 & 0.0160    & 0.0104 & 0.0084 & 0.0086 & 0.0167 \\
fraud            & 0.0656 & 0.0045   & 0.1054 & 0.0069    & 0.0024 & 0.0110 & 0.0088 & 0.0011 \\
glass            & 0.1455 & 0.0382   & 0.0865 & 0.0594    & 0.0378 & 0.0161 & 0.0194 & 0.0532 \\
http             & 0.0029 & 0.0002   & 0.3625 & 0.1784    & 0.0001 & 0.0030 & 0.0001 & 0.0000 \\
landsat          & 0.0855 & 0.0105   & 0.0193 & 0.0085    & 0.0054 & 0.0058 & 0.0259 & 0.0127 \\
letter           & 0.0548 & 0.0087   & 0.0121 & 0.0143    & 0.0102 & 0.0038 & 0.0248 & 0.0083 \\
magicgamma       & 0.0627 & 0.0049   & 0.0233 & 0.0148    & 0.0177 & 0.0035 & 0.0157 & 0.0027 \\
mammography      & 0.1788 & 0.0139   & 0.0181 & 0.0321    & 0.0435 & 0.0059 & 0.0094 & 0.0067 \\
mnist            & 0.0637 & 0.0095   & 0.0318 & 0.0022    & 0.0103 & 0.0277 & 0.0592 & 0.0042 \\
musk             & 0.2378 & 0.0046   & 0.0002 & 0.0000    & 0.0000 & 0.0005 & 0.0000 & 0.0000 \\
optdigits        & 0.1469 & 0.0255   & 0.0434 & 0.0225    & 0.0681 & 0.0091 & 0.0701 & 0.0147 \\
pendigits        & 0.1647 & 0.0104   & 0.0807 & 0.0124    & 0.0541 & 0.0029 & 0.0235 & 0.0034 \\
satellite        & 0.0745 & 0.0044   & 0.0185 & 0.0030    & 0.0032 & 0.0022 & 0.0074 & 0.0034 \\
satimage-2       & 0.0296 & 0.0004   & 0.0012 & 0.0002    & 0.0008 & 0.0026 & 0.0020 & 0.0009 \\
shuttle          & 0.2304 & 0.0003   & 0.0570 & 0.0004    & 0.0001 & 0.0005 & 0.0003 & 0.0009 \\
skin             & 0.1689 & 0.0095   & 0.0143 & 0.0220    & 0.0509 & 0.0215 & 0.0124 & 0.0113 \\
smtp             & 0.0649 & 0.0039   & 0.0525 & 0.0584    & 0.0185 & 0.0195 & 0.0070 & 0.0108 \\
speech           & 0.0241 & 0.0221   & 0.0375 & 0.0467    & 0.0387 & 0.0057 & 0.0399 & 0.0228 \\
thyroid          & 0.0434 & 0.0132   & 0.1317 & 0.0121    & 0.0075 & 0.0038 & 0.0101 & 0.0028 \\
vertebral        & 0.0745 & 0.0316   & 0.0817 & 0.1008    & 0.0590 & 0.0471 & 0.0576 & 0.0484 \\
vowels           & 0.0895 & 0.0066   & 0.0149 & 0.0502    & 0.0055 & 0.0510 & 0.0119 & 0.0049 \\
wine             & 0.2031 & 0.0288   & 0.0493 & 0.0223    & 0.0264 & 0.0693 & 0.0482 & 0.0544 \\
yeast            & 0.0241 & 0.0219   & 0.0217 & 0.0190    & 0.0179 & 0.0124 & 0.0283 & 0.0150 \\
\bottomrule
\end{tabular}
}
\end{table}

\end{document}